\documentclass[british]{article}
 \usepackage[utf8]{inputenc}
 \usepackage{natbib}
\setcitestyle{authoryear,open={(},close={)}}
\usepackage[nottoc]{tocbibind}
\usepackage[parfill]{parskip}    %
\usepackage{graphicx}
\usepackage{amssymb}
\usepackage{epstopdf}
\usepackage{amsmath}

\usepackage{accents}
\usepackage{amsthm}

\usepackage[normalem]{ulem}

\usepackage{mathrsfs}
\usepackage[toc,page]{appendix}
\usepackage{pdfpages}
\usepackage{attachfile2}
\usepackage[abs]{overpic}
\usepackage{pict2e}
\usepackage{algorithm}
\usepackage{algpseudocode}
\usepackage{cancel}
\usepackage{tikz}
\usepackage{wrapfig}

\usetikzlibrary{bayesnet}
\usetikzlibrary{arrows,decorations.markings}
\usepackage[english]{isodate}
\usepackage{hyperref}
\usepackage{booktabs} %
\usepackage{xcolor}
\usepackage[colorinlistoftodos,prependcaption,textsize=tiny]{todonotes}
\usepackage{mathtools}
\usepackage{mismath}
\usepackage{stackengine}

\usepackage{wrapfig}
\usepackage{subfigure}

\usepackage{caption}

\usepackage[final]{neurips_2020}

\fontsize{11}{13}                    %

\newtheorem{prop}{Proposition}
\newtheorem{theorem}{Theorem}

\usetikzlibrary{patterns}
\usetikzlibrary{bayesnet}
\usetikzlibrary{arrows,decorations.markings}
\tikzstyle{disent_latent} = [circle,pattern=north east lines, pattern color=black!20,draw=black,inner sep=1pt,
minimum size=20pt, font=\fontsize{10}{10}\selectfont, node distance=1]

\DeclareMathOperator{\expect}{\mathbb{E}}

\DeclarePairedDelimiterX\MeijerM[3]{\lparen}{\rparen}%
{\begin{smallmatrix}#1 \\ #2\end{smallmatrix}\delimsize\vert\,#3}

\DeclareMathOperator*{\argmin}{arg\,min}

\newtheorem{definition}{Definition}[section]

\newcommand{\vecto}[1]{\boldsymbol{\mathbf{#1}}}
\renewcommand{\v}{\vecto}

\title{Explicit Regularisation in Gaussian Noise Injections}

\author{Alexander Camuto \\ University of Oxford \\ Alan Turing Institute \\acamuto@turing.ac.uk
 \And
 Matthew Willetts \\ University of Oxford \\ Alan Turing Institute \\ mwilletts@turing.ac.uk
\And
Umut Şimşekli \\ University of Oxford\\ Institut Polytechnique de Paris  \\ umut.simsekli@telecom-paris.fr
\And
Stephen Roberts \\ University of Oxford \\ Alan Turing Institute\\ sjrob@robots.ox.ac.uk
\And
Chris Holmes \\ University of Oxford\\ Alan Turing Institute\\ cholmes@stats.ox.ac.uk}

\begin{document}
\maketitle

\begin{abstract}
We study the regularisation induced in neural networks by Gaussian noise injections (GNIs). 
Though such injections have been extensively studied when applied to data, there have been few studies on understanding the regularising effect they induce when applied to network activations. 
Here we derive the explicit regulariser of GNIs, obtained by marginalising out the injected noise, and show that it penalises functions with high-frequency components in the Fourier domain; particularly in layers closer to a neural network's output. 
We show analytically and empirically that such regularisation produces calibrated classifiers with large classification margins.
\end{abstract}

\section{Introduction}
\label{intro}
Noise injections are a family of methods that involve adding or multiplying samples from a noise distribution, typically an isotropic Gaussian, to the weights or activations of a neural network during training.
The benefits of such methods are well documented. 
Models trained with noise often generalise better to unseen data and are less prone to overfitting \citep{Srivastava2014, Kingma, Poole2014}.

Even though the regularisation conferred by Gaussian noise injections (GNIs) can be observed empirically, and the benefits of noising data are well understood theoretically \citep{Bishop1995, Cohen2019, Webb1994}, there have been few studies on understanding the benefits of methods that inject noise \textit{throughout} a network.
Here we study the \textit{explicit} regularisation of such injections, which is a positive term added to the loss function obtained when we marginalise out the noise we have injected.

Concretely our contributions are: 

\begin{itemize}
    \item We derive an analytic form for an explicit regulariser that explains most of GNIs' regularising effect.
    \item We show that this regulariser penalises networks that learn functions with high-frequency content in the Fourier domain and most heavily regularises neural network layers that are closer to the output. See Figure \ref{fig:illustration} for an illustration.
    \item Finally, we show analytically and empirically that this regularisation induces larger classification margins and better calibration of models. 
\end{itemize}

   \begin{figure}[t!]

    \centering
    \includegraphics[width=0.9\textwidth]{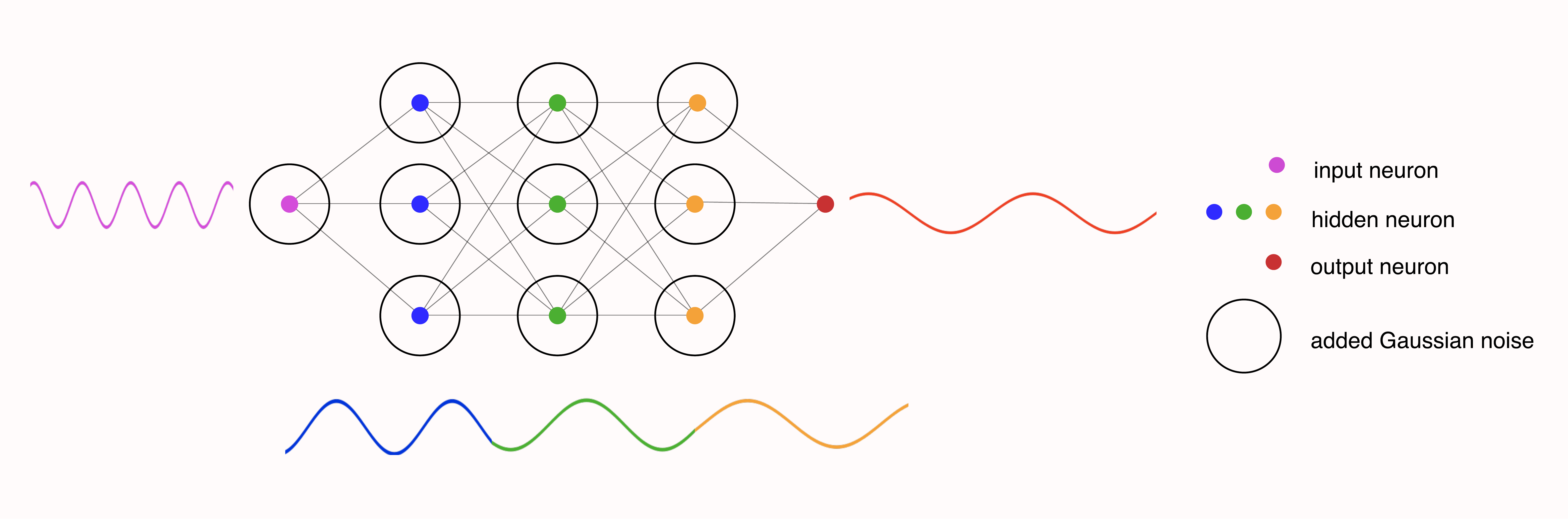}
    \caption{
    Here we illustrate the effect of GNIs injected throughout a network's activations. Each coloured dot represents a neuron's activations.
    We add GNIs, represented as circles, to each layer's activations bar the output layer. 
    GNIs induce a network for which each layer learns a progressively lower frequency function, represented as a sinusoid matching in colour to its corresponding layer. 
    }
    \vspace{-1em}
    \label{fig:illustration}
\end{figure}

\section{Background}

\subsection{Gaussian Noise Injections}
\label{sec:sgd_under_injections}

Training a neural network involves optimising network parameters to maximise the marginal likelihood of a set of labels given features via gradient descent.
With a training dataset $\mathcal{D}$ composed of $N$ data-label pairs of the form $(\v{x}, \v{y}) \ \v{x} \in \mathbb{R}^d, \v{y} \in \mathbb{R}^m$ and a feed-forward neural network with $M$ parameters divided into $L$ layers: $\v{\theta} = \{\v{W}_1,...,\v{W}_L\}$, $\v{\theta} \in \mathbb{R}^M$,
our objective is to minimise the expected negative log likelihood of labels $\v{y}$  given data $\v{x}$,  $- \log p_{\v{\theta}}(\v{y}|\v{x})$ , and find the optimal set of parameters  $\v{\theta}^*$ satisfying: 
\begin{align}
    \v{\theta}_* = \argmin_{\v{\theta}} \mathcal{L}( \mathcal{D}; \v{\theta}),
    \qquad
    \mathcal{L}( \mathcal{D}; \v{\theta}):=- \mathbb{E}_{\v{x},\v{y} \sim \mathcal{D}}\left[ \log p_{\v{\theta}}(\v{y}|\v{x})\right]\,.
    \label{eq:marginal_likelihood}
\end{align}
Under stochastic optimisation algorithms, such as Stochastic Gradient Descent (SGD), we  estimate $\mathcal{L}$ by sampling a mini-batch of data-label pairs $\mathcal{B} \subset \mathcal{D}$. 
\begin{equation}
\mathcal{L}(\mathcal{B}; \v{\theta}) = -\mathbb{E}_{\v{x},\v{y} \sim \mathcal{B}} \log p_{\v{\theta}}(\v{y}|\v{x}) \approx \mathcal{L}( \mathcal{D}; \v{\theta}).
\end{equation}
Consider an $L$ layer network with no noise injections and a non-linearity $\phi$ at each layer.
We obtain the activations ${\v{h}} = \{{\v{h}}_0, ... , {\v{h}}_{L} \}$, where $\v{h}_{0}=\v{x}$ is the input data \textit{before} any noise is injected. For a network consisting of dense layers (a.k.a.\ a multi-layer perceptron: MLP) we have that: 
\begin{align}
\v{h}_{k}(\v{x})=
\phi(\v{W}_k \v{h}_{k-1}(\v{x})).
\label{eq:nonoise_acts}
\end{align}
What happens to these activations when we inject noise? 
First, let $\v{\epsilon}$ be the set of noise injections at each layer:  $\v{\epsilon} = \{{\v{\epsilon}}_0, ... , {\v{\epsilon}}_{L-1} \}$.
When performing a noise injection procedure, the value of the next layer's activations depends on the noised value of the previous layer.
We denote the intermediate, soon-to-be-noised value of an activation as $\widehat{\v{h}}_{k}$ and the subsequently noised value as $\widetilde{\v{h}}_{k}$: 
\begin{align}
\widehat{\v{h}}_{k}(\v{x})= \phi\left(\v{W}_{k}\widetilde{\v{h}}_{k-1}(\v{x})\right)\,, \qquad  
\widetilde{\v{h}}_{k}(\v{x}) = \widehat{\v{h}}_{k}(\v{x}) \circ \v{\epsilon}_{k}\,, 
\label{eq:noise_recursion}
\end{align}
where $\circ$ is some element-wise operation. 
We can, for example, add or multiply Gaussian noise to each hidden layer unit. 
In the additive case, we obtain: 
\begin{align}
\widetilde{\v{h}}_k(\v{x}) &= \widehat{\v{h}}_k(\v{x}) + \v{\epsilon}_k, 
\qquad \v{\epsilon}_k \sim \mathcal{N}(0,\sigma_k^2\v{I}). \label{eq:noise_add}
\end{align}
The multiplicative case can be rewritten as an activation-scaled addition:
\begin{align}
\widetilde{\v{h}}_k(\v{x}) &= \widehat{\v{h}}_k(\v{x}) + \v{\epsilon}_k, 
\qquad \v{\epsilon}_k \sim \mathcal{N}\left(0,\widehat{\v{h}}_k^2(\v{x})\sigma_k^2\v{I}\right).
\label{eq:noise_mult}
\end{align}
Here we focus our analysis on noise \textit{additions}, but through equation \eqref{eq:noise_mult} we can translate our results to the multiplicative case.

\subsection{Sobolev Spaces}
\label{sec:sobolev}
To define a Sobolev Space we use the generalisation of the derivative for multivariate functions of the form $g: \mathbb{R}^d \to \mathbb{R}$. 
We use a multi-index notation $\alpha \in \mathbb{R}^d$ which defines mixed partial derivatives. 
We denote the $\alpha^{\mathrm{th}}$ derivative of $g$ with respect to its input $\v{x}$ as $D^\alpha g(\v{x})$.
\[
D^{\alpha}g = \frac{\partial^{|\alpha|}g}{\partial x^{\alpha_1}_1\dots\partial x^{\alpha_d}_d}
\]
where $|\alpha| = \sum_{i=1}^d|\alpha_i|$. Note that $\v{x}^\alpha = [x_1^{\alpha_1}, \dots, x_d^{\alpha_d}]$ and $\alpha! = \alpha_1!\dots\cdot\dots\alpha_d!$. 
\vspace{1em}

\begin{definition}[\textnormal{\cite{Cucker2002}}]
Sobolev spaces are denoted $W^{l,p}(\Omega), \Omega \subset \mathbb{R}^d$,  where $l$, the order of the space, is a non-negative integer and $p\geq1$.  The Sobolev space of index $(l,p)$ is the space of locally integrable functions $f: \Omega \to \mathbb{R}$ such that for every multi-index $\alpha$ where $|\alpha| < l$ the derivative $D^{\alpha}f$ exists and $D^{\alpha}f \in L^p(\Omega)$.
The norm in such a space is given by 
$\|f\|_{W^{l,p}(\Omega)} = \left(\sum_{|\alpha| \leq l} \int_{\Omega} |D^{\alpha}f(\v{x})|^p d\v{x}\right)^{\frac{1}{p}}$.
\label{def:sobolev}
\end{definition}
For $p=2$ these spaces are Hilbert spaces, with a dot product that defines the $L_2$ norm of a function's derivatives. Further these Sobolev spaces can be defined in a measure space with \textit{finite} measure $\mu$. 
We call such spaces finite measure spaces of the form $W^{l,p}_{\mu}(\mathbb{R}^d)$ and these are the spaces of locally integrable functions such that for every $\alpha$ where $|\alpha| < l, \ D^{\alpha}f \in L^p_{\mu}(\mathbb{R}^d)$, the $L^p$ space equipped with the measure $\mu$. 
The norm in such a space is given by \citep{Hornik1991}: 
\begin{equation}
\|f\|_{W^{l,p}_{\mu}(\mathbb{R}^d)} = \left(\sum_{|\alpha| \leq l} \int_{\mathbb{R}^d} |D^{\alpha}f(\v{x})|^p d\mu(\v{x})\right)^{\frac{1}{p}}, f \in W^{l,p}_{\mu}(\mathbb{R}^d), |\mu(\v{x})| < \infty \ \forall \v{x} \in \mathbb{R}^d
\label{eq:weighted_sobolev_norm}
\end{equation}

Generally a Sobolev space over a compact subset $\Omega$ of $\mathbb{R}^d$ can be expressed as a weighted Sobolev space with a measure $\mu$ which has compact support on $\Omega$ \citep{Hornik1991}. 

\cite{Hornik1991} have shown that neural networks with continuous activations, which have continuous and bounded derivatives up to order $l$, such as the sigmoid function, are universal approximators in the \textit{weighted} Sobolev spaces of order $l$, meaning that they form a dense subset of Sobolev spaces.
Further, \cite{Czarnecki2017} have shown that networks that use piecewise linear activation functions (such as $
\mathrm{ReLU}$ and its extensions) are \textit{also} universal approximators in the Sobolev spaces of order 1 where the domain $\Omega$ is some compact subset of $\mathbb{R}^d$. 
As mentioned above, this is equivalent to being dense in a weighted Sobolev space on $\mathbb{R}^d$  where the measure $\mu$ has compact support. 
Hence, we can view a neural network, with sigmoid or piecewise linear activations to be a parameter that indexes a function in a weighted Sobolev space with index $(1,2)$, i.e. $f_{\v{\theta}} \in W^{1,2}_{\mu}(\mathbb{R}^d)$.

\section{The Explicit Effect of Gaussian Noise Injections}

Here we consider the case where we noise all layers with \textit{isotropic} noise, except the final predictive layer which we also consider to have no activation function. 
We can express the effect of the Gaussian noise injection on the cost function as an added term $\Delta\mathcal{L}$, which is dependent on $\v{\mathcal{E}}_{L}$, the noise accumulated on the final layer $L$ from the noise additions $\v{\epsilon}$ on the previous hidden layer activations. 
\begin{equation}
\widetilde{\mathcal{L}}(\mathcal{B};\v{\theta}, \v{\epsilon}) =  \mathcal{L}(\mathcal{B}; \v{\theta}) + \Delta\mathcal{L}(\mathcal{B};\v{\theta},\v{\mathcal{E}}_{L})
\end{equation}

To understand the regularisation induced by GNIs, we want to study the regularisation that these injections induce \textit{consistently} from batch to batch. 
To do so, we want to remove the stochastic component of the GNI regularisation and extract a regulariser that is of consistent sign. 
Regularisers that change sign from batch-to-batch do not give a constant objective to optimise, making them unfit as regularisers \citep{Botev2017, Sagun2018, Wei2020}.

As such, we study the explicit regularisation these injections induce by way of the expected regulariser, $\expect_{\v{\epsilon} \sim p(\v{\epsilon})} \left[ \Delta\mathcal{L}(\cdot) \right]$ that marginalises out the injected noise $\v{\epsilon}$. 
To lighten notation, we denote this as  $\expect_{\v{\epsilon}} \left[\Delta\mathcal{L}(\cdot)\right]$.
We extract $R$, a constituent term of the expected regulariser that dominates the remainder terms in norm, and is \textit{consistently positive}.

Because of these properties, $R$ provides a lens through which to study the effect of GNIs. 
As we show, this term has a connection to the Sobolev norm and the Fourier transform of the function parameterised by the neural network. Using these connections we make inroads into better understanding the regularising effect of noise injections on neural networks.

To begin deriving this term, we first need to define the accumulated noise $\v{\mathcal{E}}_{L}$. 
We do so by applying a Taylor expansion to each noised layer.
As in Section \ref{sec:sobolev} we use the generalisation of the derivative for multivariate functions using a multi-index $\alpha$.
For example $D^\alpha h_{k,i}(\v{h}_{k-1}(\v{x}))$ denotes the $\alpha^{\mathrm{th}}$ derivative of the $i^{\mathrm{th}}$ activation of the $k^{\mathrm{th}}$ layer ($h_{k,i}$) with respect to the preceding layer's activations  $\v{h}_{k-1}(\v{x})$ and $D^\alpha \mathcal{L}(\v{h}_{k}(\v{x}), \v{y})$ denotes the $\alpha^{\mathrm{th}}$ derivative of the loss with respect to the non-noised activations $\v{h}_{k}(\v{x})$.

\begin{prop}
Consider an $L$ layer neural network experiencing isotropic GNIs at each layer $k \in [0, \dots, L - 1]$ of dimensionality $d_k$. 
We denote this added noise as $\v{\epsilon} = \{{\v{\epsilon}}_0, ... , {\v{\epsilon}}_{L-1} \}$. 
We assume $\v{h}_L$ is in $C^\infty$ the class of infinitely differentiable functions. 
We can define the accumulated noise at layer each layer $k$ using a multi-index $\alpha_k \in \mathbb{N}^{d_{k-1}}$:
\begin{align*}
    \mathcal{E}_{L,i} &=
    \sum_{|\alpha_L|=1}^{\infty}
    \frac{1}{\alpha_L!}
     \left(D^{\alpha_L} h_{L,i}(\v{h}_{L-1}(\v{x}))\right)
     \v{\mathcal{E}}_{L-1}^{\alpha_L}, \ i = 1, \dots, d_L \\
     \mathcal{E}_{k,i} &= \epsilon_{k,i} +  \sum_{|\alpha_k|=1}^{\infty}
    \frac{1}{\alpha_k!}
    \left(D^{\alpha_k} h_{k,i}(\v{h}_{k-1}(\v{x}))\right) 
    \v{\mathcal{E}}_{k-1}^{\alpha_k} , \ i = 1, \dots, d_k,\ k = 1 \dots L-1  \\  \v{\mathcal{E}}_0 &= \v{\epsilon}_0 \ 
\end{align*}
\label{prop:accum_noise}
where $\v{x}$ is drawn from the dataset $\mathcal{D}$, $\v{h}_{k}$ are the activations before any noise is added, as defined in Equation \eqref{eq:nonoise_acts}. 
\end{prop}

See Appendix \ref{app:accum_noise} for the proof. 
Given this form for the accumulated noise, we can now define the expected regulariser. 
For compactness of notation, we denote each layer's Jacobian as $\v{J}_k \in \mathbb{R}^{d_L \times d_k}$ and the Hessian of the loss with respect to the final layer as $\v{H}_L \in \mathbb{R}^{d_L \times d_L}$. 
Each entry of $\v{J}_{k}$ is a partial derivative  of $f^{\theta}_{k,i}$, the function from layer $k$ to the $i^{\mathrm{th}}$ network output, $i = 1 ... d_L$. 
\begin{align*}
    \v{J}_{k}(\v{x}) = \begin{bmatrix} 
    \frac{f^{\theta}_{k,1}}{\partial h_{k,1}} & \frac{f^{\theta}_{k,1}}{\partial h_{k,2}} & \dots \\
    \vdots & \ddots & \\
    \frac{f^{\theta}_{k,d_L}}{\partial h_{k,1}} &        & \frac{f^{\theta}_{k,d_L}}{\partial h_{k,d_k}} 
    \end{bmatrix} , \ 
    \v{H}_{L}(\v{x}, \v{y}) = \begin{bmatrix} 
    \frac{\partial^2 \mathcal{L}}{\partial h^2_{L,1}} & \frac{\partial^2 \mathcal{L}}{\partial h_{L,1}\partial h_{L,2}} & \dots \\
    \vdots & \ddots & \\
    \frac{\partial^2 \mathcal{L}}{\partial h_{L,d_L}\partial h_{L,1}} &        & \frac{\partial^2 \mathcal{L}}{\partial h^2_{L,d_L}} 
    \end{bmatrix} 
\end{align*}

Using these notations we can now define the explicit regularisation induced by GNIs. 

\begin{theorem}
Consider an $L$ layer neural network experiencing isotropic GNIs at each layer $k \in [0, \dots, L - 1]$ of dimensionality $d_k$. 
We denote this added noise as $\v{\epsilon} = \{{\v{\epsilon}}_0, ... , {\v{\epsilon}}_{L-1} \}$. 
We assume  $\mathcal{L}$ is in $C^\infty$ the class of infinitely differentiable functions. 
We can marginalise out the injected noise $\v{\epsilon}$ to obtain an added regulariser: 
\begin{align*}
\expect_{\v{\epsilon}} \left[ \Delta\mathcal{L}(\mathcal{B};\v{\theta},\v{\mathcal{E}}_{L}) \right] 
    &= \mathbb{E}_{(\v{x},\v{y}) \sim \mathcal{B}} \left[\frac{1}{2}\sum_{k=0}^{L-1}\left[\sigma_k^2\mathrm{Tr}\left(\v{J}^\intercal_{k}(\v{x})
    \v{H}_{L}(\v{x}, \v{y})\v{J}_{k}(\v{x})\right)\right] \right]  + \expect_{\v{\epsilon} } \left[\mathcal{C}(\mathcal{B}, \v{\epsilon} )\right]
    \nonumber 
\end{align*}
where $\v{h}_{k}$ are the activations before any noise is added, as in equation \eqref{eq:nonoise_acts}. $\expect_{\v{\epsilon} } \left[\mathcal{C}(\cdot)\right]$ is a remainder term in higher order derivatives.
\label{prop:explicit_reg}
\end{theorem}
See Appendix \ref{app:exp_reg} for the proof and for the exact form of the remainder. 
We denote the first term in Theorem~\ref{prop:explicit_reg} as:
\begin{equation}
     R(\mathcal{B}; \v{\theta}) = \mathbb{E}_{(\v{x},\v{y}) \sim \mathcal{B}} \left[\frac{1}{2}\sum_{k=0}^{L-1}\left[\sigma_k^2\mathrm{Tr}\left(\v{J}^\intercal_{k}(\v{x})
    \v{H}_{L}(\v{x}, \v{y})\v{J}_{k}(\v{x})\right)\right] \right]
\end{equation}

To understand the main contributors behind the regularising effect of GNIs, we first want to establish the relative importance of the two terms that constitute the explicit effect. 
We know that $R$ is the added regulariser for the linearised version of a neural network, defined by its Jacobian.
This linearisation well approximates neural network behaviour for \textit{sufficiently wide} networks \citep{ntk, chizat:hal-01945578, NEURIPS2019_dbc4d84b} in early stages of training \citep{chen2020generalized}, and we can expect $R$ to dominate the remainder term in norm, which consists of higher order derivatives. 
In Figure \ref{fig:kappadom} we show that this is the case for a range of GNI variances, datasets, and activation functions for networks with 256 neurons per layer; where the remainder is estimated as: 
\[\expect_{\v{\epsilon} } \left[\mathcal{C}(\mathcal{B}, \v{\epsilon})\right] \approx \frac{1}{1000}\sum_{i=0}^{1000}\widetilde{\mathcal{L}}(\mathcal{B};\v{\theta}, \v{\epsilon}) - R(\mathcal{B}; \v{\theta}) - \mathcal{L}(\mathcal{B};\v{\theta}).\] 
These results show that $R$ is a significant component of the regularising effect of GNIs. 
It dominates the remainder $\expect_{\v{\epsilon} } \left[\mathcal{C}(\cdot)\right]$ in norm and is always positive, as we will show, thus offering a consistent objective for SGD to minimise.
Given that $R$ is a likely candidate for understanding the effect of GNIs; we further study this term separately in regression and classification settings.

\begin{figure}[t!]
    \centering
     \subfigure[][BHP Sigmoid]{\includegraphics[width=0.25\textwidth]{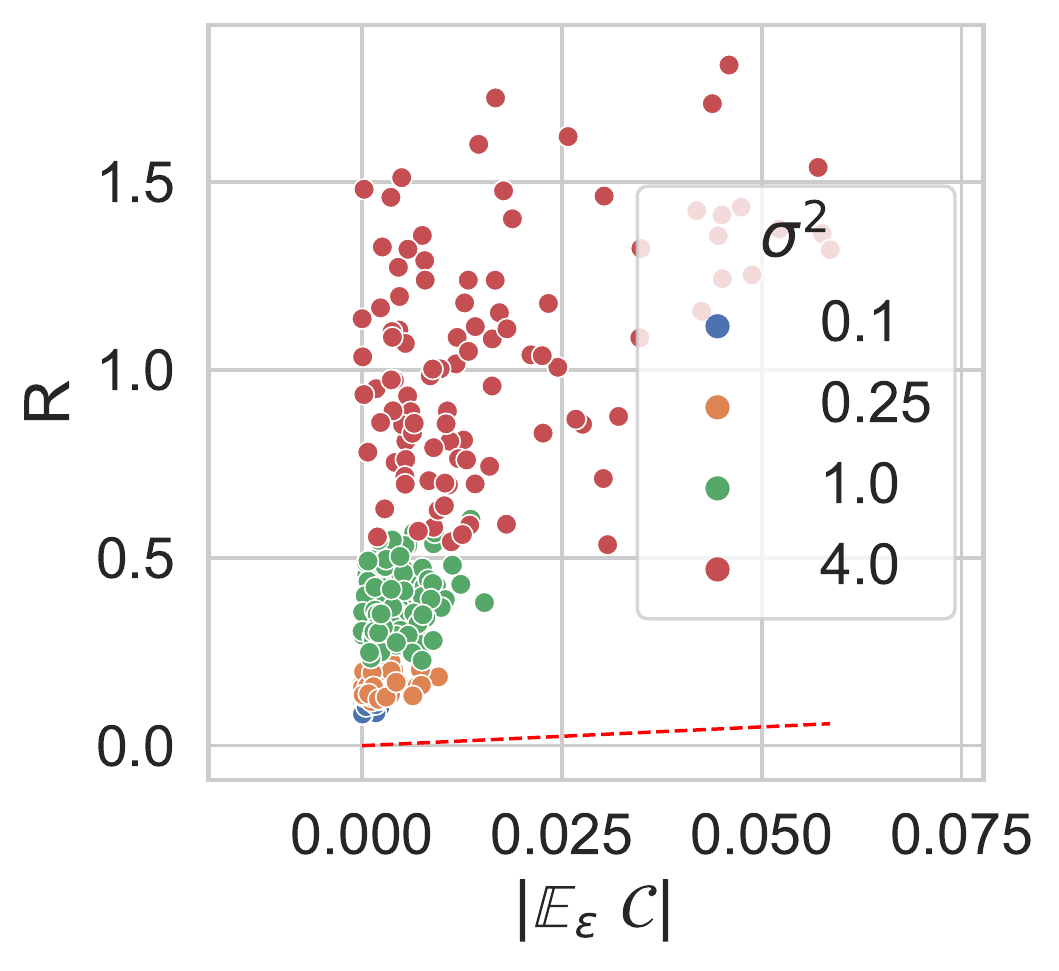}}
     \subfigure[][CIFAR10 ELU]{\includegraphics[width=0.24\textwidth]{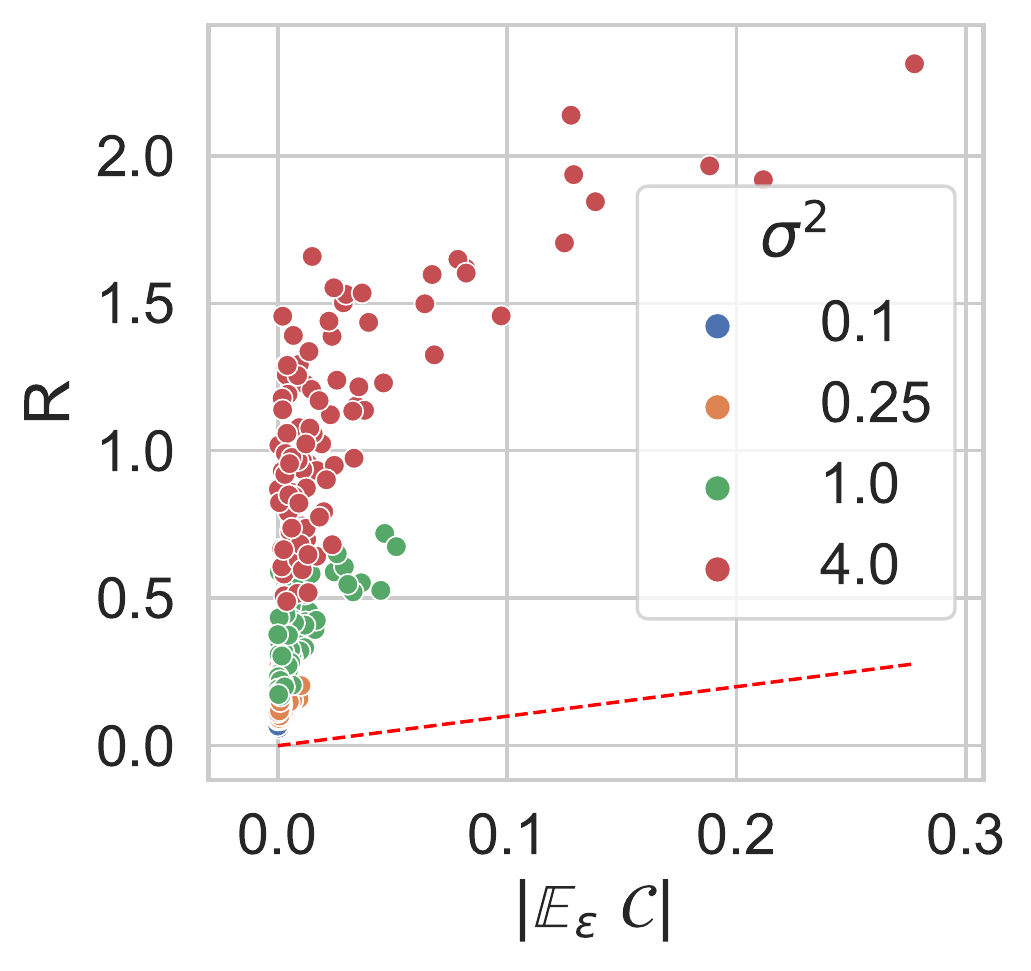}}
     \subfigure[][BHP Sigmoid]{\includegraphics[width=0.25\textwidth]{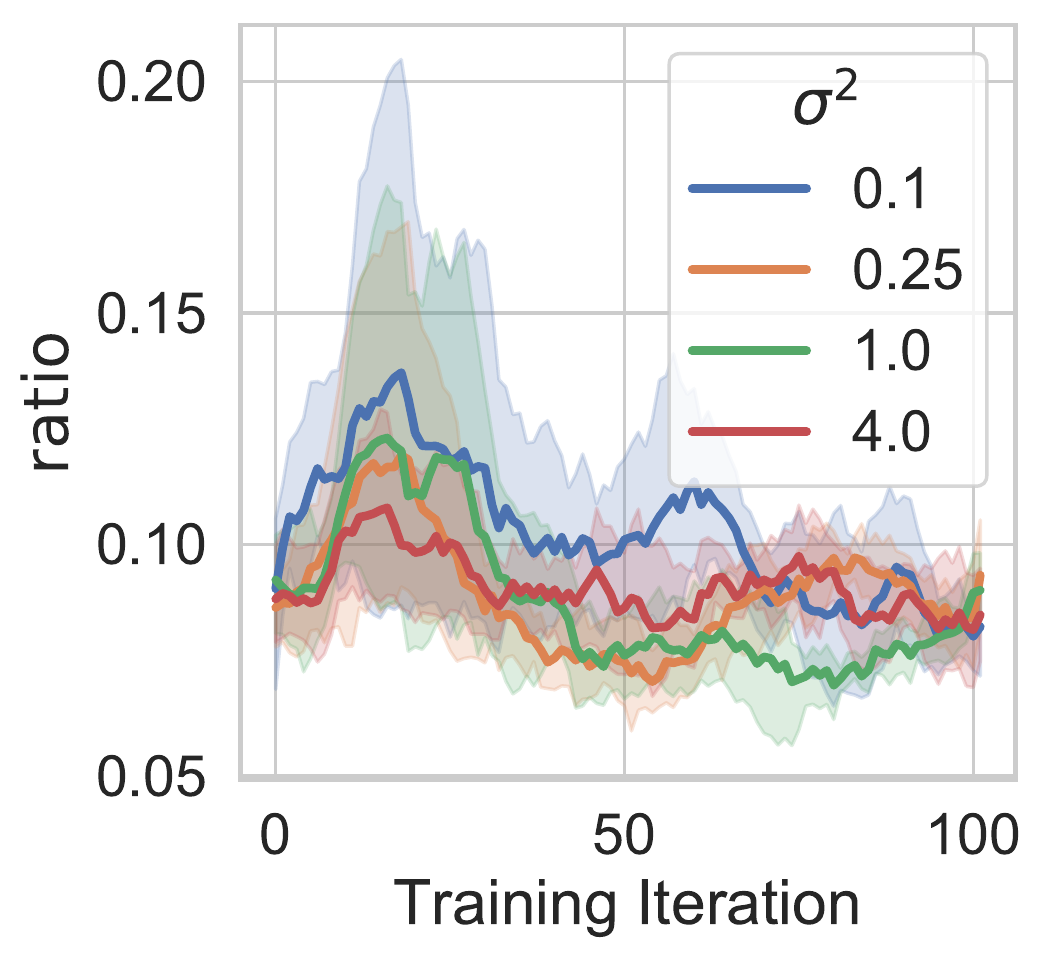}}
    \subfigure[][CIFAR10 ELU]{\includegraphics[width=0.24\textwidth]{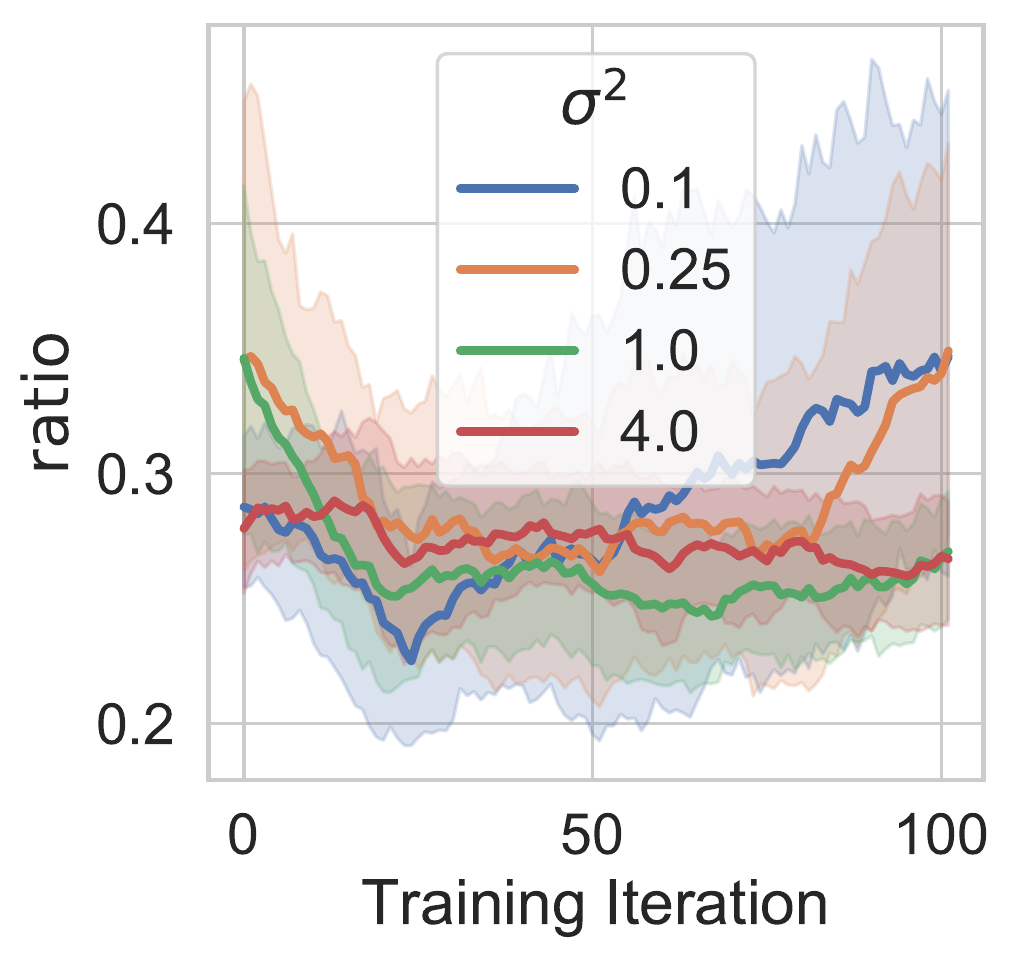}}
   
    \caption{
    In (a,b) we plot $R(\cdot)$ vs $\expect_{\v{\epsilon} } \left[\mathcal{C}(\cdot)\right]$ at initialisation for 6-layer-MLPs with GNIs at each 256-neuron layer with the same variance $\sigma^2 \in [0.1,0.25,1.0,4.0]$ at each layer. Each point corresponds to one of 250 different network initialisation acting on a batch of size 32 for the classification dataset CIFAR10 and regression dataset Boston House Prices (BHP) datasets. The dotted red line corresponds to $y=x$ and demonstrates that for all batches and GNI variances $R$ is greater than $\expect_{\v{\epsilon} } \left[\mathcal{C}(\cdot)\right]$. In (c,d) we plot $\mathrm{ratio}=|\expect_{\v{\epsilon} } \left[\mathcal{C}(\cdot)\right]|/R(\cdot)$ in the first 100 training iteration for 10 randomly initialised networks. Shading corresponds to the standard deviation of values over the 10 networks. $R(\cdot)$ remains dominant in early stages of training as the ratio is less than 1 for all steps. 
    }
    \label{fig:kappadom}
\end{figure}

\paragraph{Regularisation in Regression}

In the case of regression one of the most commonly used loss functions is the mean-squared error (MSE), which is defined for a data label pair $(\v{x}, \v{y})$ as:
\begin{equation}
\mathcal{L}(\v{x}, \v{y}) = \frac{1}{2}(\v{y} - \v{h}_{L}(\v{x}))^2.
\label{eq:mse}
\end{equation}
For this loss, the Hessians in Theorem \ref{prop:explicit_reg} are simply the identity matrix. 
The explicit regularisation term, guaranteed to be positive is: 
\begin{equation}
    R(\mathcal{B}; \v{\theta}) = \frac{1}{2} \mathbb{E}_{\v{x} \sim \mathcal{B}} \left[ \sum_{k=0}^{L-1} \sigma^2_k (\|\v{J}_{k}(\v{x})\|^2_F) \right].
    \label{eq:mse_reg_term}
\end{equation}

where $\sigma^2_k$ is the variance of the noise $\v{\epsilon}_{k}$ injected at layer $k$ and $\|\cdot\|_F$ is the Frobenius norm. 
See Appendix \ref{app:exp_reg_regression} for a proof.

\paragraph{Regularisation in Classification }
In the case of classification, we consider the case of a cross-entropy (CE) loss.
Recall that we consider our network outputs $\v{h}_L$ to be the pre-$\mathrm{softmax}$ of the logits of the final layer.
We denote $\v{p}(\v{x})=\mathrm{softmax}(\v{h}_L(\v{x}))$. 
For a pair $(\v{x}, \v{y})$ we have:
\begin{equation}
\mathcal{L}(\v{x}, \v{y}) = - \sum_{c=0}^C \v{y}_{c} \log (\v{p}(\v{x}))_c),
\label{eq:cross_entropy}
\end{equation}
where $c$ indexes over $C$ possible classes.
The hessian $\v{H}_{L}(\cdot)$ no longer depends on $\v{y}$: 
\begin{equation}
\v{H}_{L}(\v{x})_{i,j} = 
\begin{cases}
\v{p}(\v{x})_i(1 - \v{p}(\v{x})_j) & i = j \\
-\v{p}(\v{x})_i\v{p}(\v{x})_j & i \neq j \\
\end{cases}
\label{eq:hessian_classification}
\end{equation} 
This Hessian is positive-semi-definite  and $R(\cdot)$, guaranteed to be positive, can be written as:
\begin{equation}
    R(\mathcal{B}; \v{\theta}) = \frac{1}{2} \mathbb{E}_{\v{x} \sim \mathcal{B}} \left[ \sum_{k=0}^{L-1} \sigma_k^2 \sum_{i,j}(\mathrm{diag}( \v{H}_{L}(\v{x}))^\intercal{\v{J}}^2_{k}(\v{x}) )_{i,j} \right],
\label{eq:ce_reg_term}
\end{equation}
where $\sigma^2_k$ is the variance of the noise $\v{\epsilon}_{k}$ injected at layer $k$.
See Appendix \ref{app:exp_reg_classification} for the proof.

\begin{figure}[t!]
    \centering
    \subfigure[][SVHN ELU]{\includegraphics[width=0.54\textwidth]{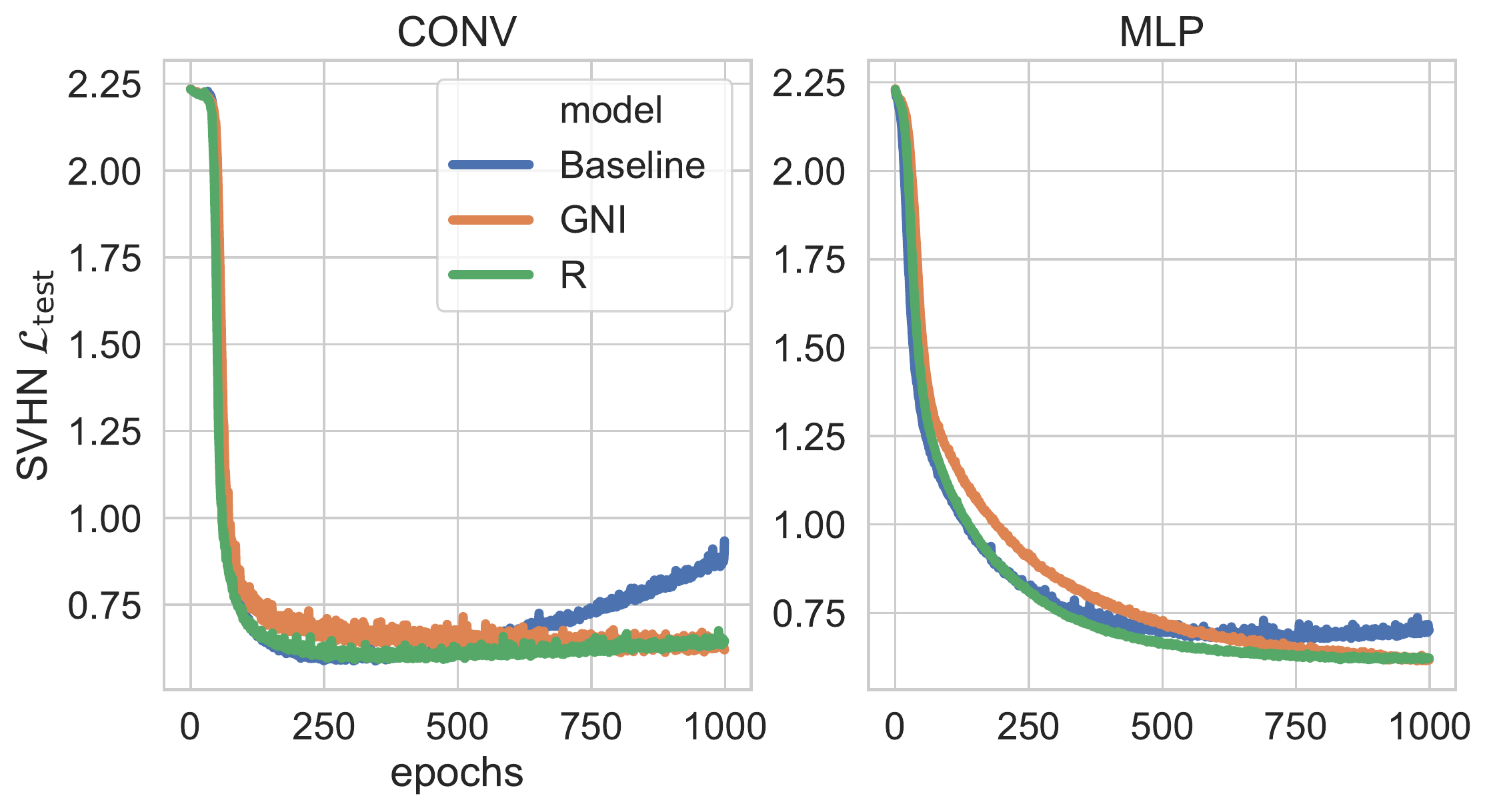}}
    \hspace{1cm}
    \subfigure[][SVHN MLP $\mathrm{Tr}(\v{H})$ ]{\includegraphics[width=0.3\textwidth]{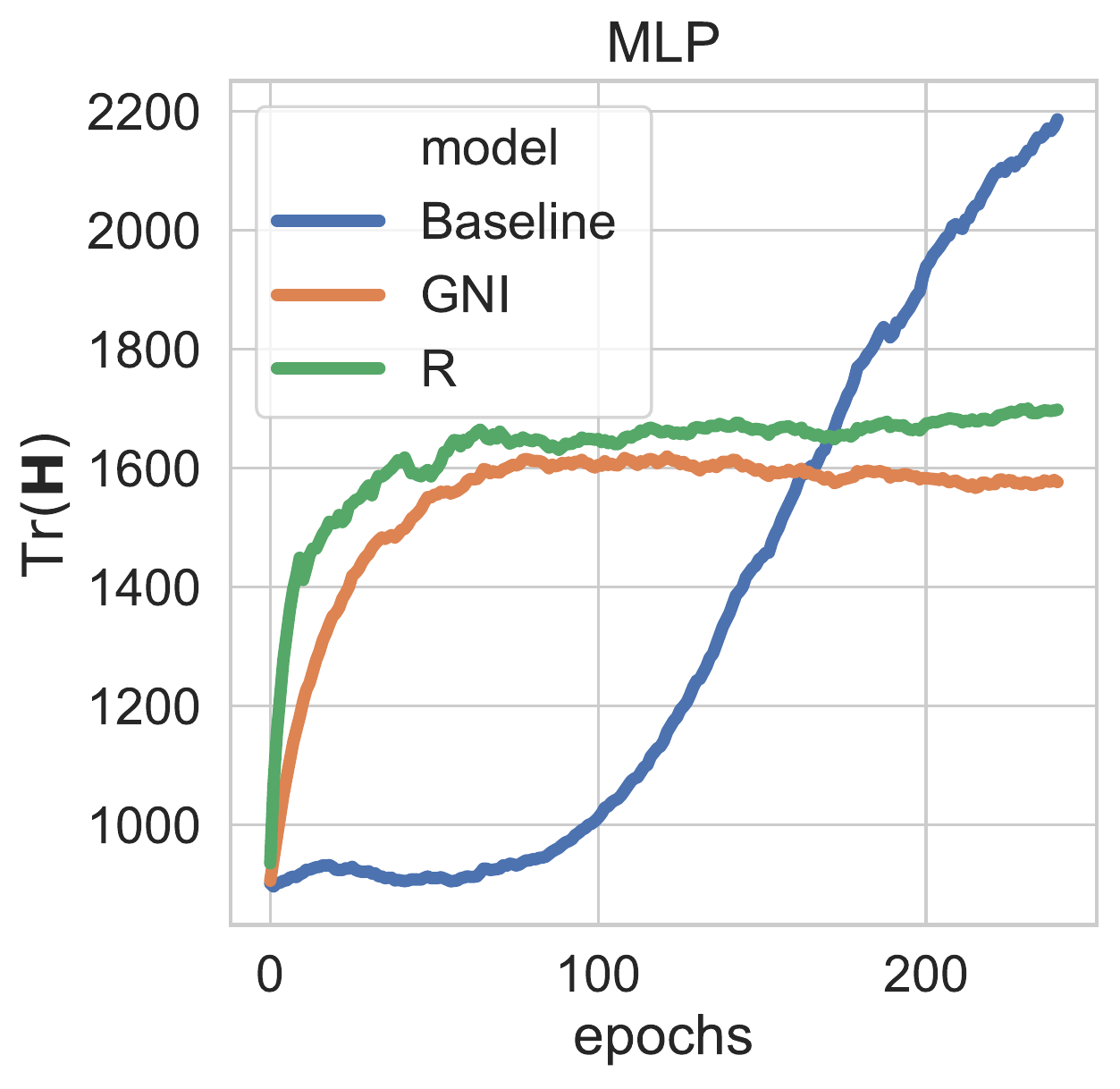}} \\
    \caption{
    Figure (a) shows the test set loss for convolutional models (CONV) and 4 layer MLPs trained on SVHN with $R(\cdot)$ and GNIs for $\sigma^2=0.1$, and no noise (Baseline).
    Figure (b) shows the trace of the network parameter Hessian for a 2-layer, 32-unit-per-layer MLP where $\v{H}_{i,j} = \frac{\partial \mathcal{L}}{\partial w_i\partial w_j}$, which is a proxy for the parameters' location in the loss landscape.  
    All networks use ELU activations. 
    See Appendix \ref{app:exp_reg_results} for more such results on other datasets and network architectures.
    }
    \label{fig:marg_approx}
\end{figure}

To test our derived regularisers, in Figure~\ref{fig:marg_approx} we show that models trained with $R$ and GNIs have similar training profiles, whereby they have similar test-set loss and parameter Hessians throughout training, meaning that they have almost identical trajectories through the loss landscape. This implies that $R$ is a good descriptor of the effect of GNIs and that we can use this term to understand the mechanism underpinning the regularising effect of GNIs. 
As we now show, it penalises neural networks that parameterize functions with higher frequencies in the Fourier domain; offering a novel lens under which to study GNIs. 

\section{Fourier Domain Regularisation}
\label{sec:Fourier}

To link our derived regularisers to the Fourier domain, we use the connection between neural networks and Sobolev Spaces mentioned above. 
Recall that by \cite{Hornik1991}, we can only assume a sigmoid or piecewise linear neural network parameterises a function in a weighted Sobolev space with measure $\mu$, if we assume that the measure $\mu$ has compact support on a subset $\Omega \in \mathbb{R}^d$.
As such, we equip our space with the \textit{probability} measure $\mu(\v{x})$, which we assume has compact support on some subset $\Omega \subset \mathbb{R}^d$ where $\mu(\Omega)=1$. 
We define it such that $d\mu(\v{x}) = p(\v{x})d\v{x}$ where $d\v{x}$ is the Lebesgue measure and $p(\v{x})$ is the data density function. 
Given this measure, we can connect the derivative of functions that are in the Hilbert-Sobolev space $W^{1,2}_{\mu}(\mathbb{R}^d)$ to the Fourier domain. 

\begin{theorem}
Consider a function, $f: \mathbb{R}^d \to \mathbb{R}$, with a $d$-dimensional input and a single output with $f \in W^{1,2}_{\mu}(\mathbb{R}^d)$ where $\mu$ is a probability measure which we assume has compact support on some subset $\Omega \subset \mathbb{R}^d$ such that $\mu(\Omega)=1$.
Let us assume the derivative of $f$, $D^{\alpha} f$, is in  $L^2(\mathbb{R}^d)$ for some multi-index $\alpha$ where $|\alpha|=1$.  
Then we can write that: 
\begin{align}
 \sum_{|\alpha|=1}\|D^\alpha f\|^2_{L^2_{\mu}(\mathbb{R}^d)} &= 
   \int_{\mathbb{R}^d} \sum^d_{j=1} \Bigr|\mathcal{G}(\v{\omega},j)\overline{ \mathcal{G}(\v{\omega},j) * \mathcal{P}(\v{\omega})} \Bigr|  d\v{\omega}  \\ 
 \mathcal{G}(\v{\omega},j) &=  \v{\omega}_j\mathcal{F}(\v{\omega})  \nonumber
\end{align}
where $\mathcal{F}$ is the Fourier transform of $f$, $\mathcal{P}$ is the Fourier transform or the `characteristic function' of the probability measure $\mu$, $j$ indexes over $\v{\omega} = [\omega_1, \dots, \omega_d]$, $*$ is the convolution operator, and $\overline{(\cdot)}$ is the complex conjugate. 
\label{th:plancherel-measure-space}
\end{theorem}

See Appendix \ref{app:plancherel_th} for the proof. 
Note that in the case where the dataset contains finitely many points, the integrals for the norms of the form $ \|D^\alpha f_{\theta}\|^2_{L^2_{\mu}(\mathbb{R}^d)}$ are approximated by sampling a batch from the dataset which is distributed according to the presumed probability measure $\mu(\v{x})$. 
Expectations over a batch thus approximate integration over $\mathbb{R}^d$ with the measure $\mu(\v{x})$ and this approximation improves as the batch size grows. 
We can now use Theorem \ref{th:plancherel-measure-space} to link $R$ to the Fourier domain.

\begin{figure}[t!]
    \centering
     \hspace{0.4cm}
    \includegraphics[width=0.3\textwidth]{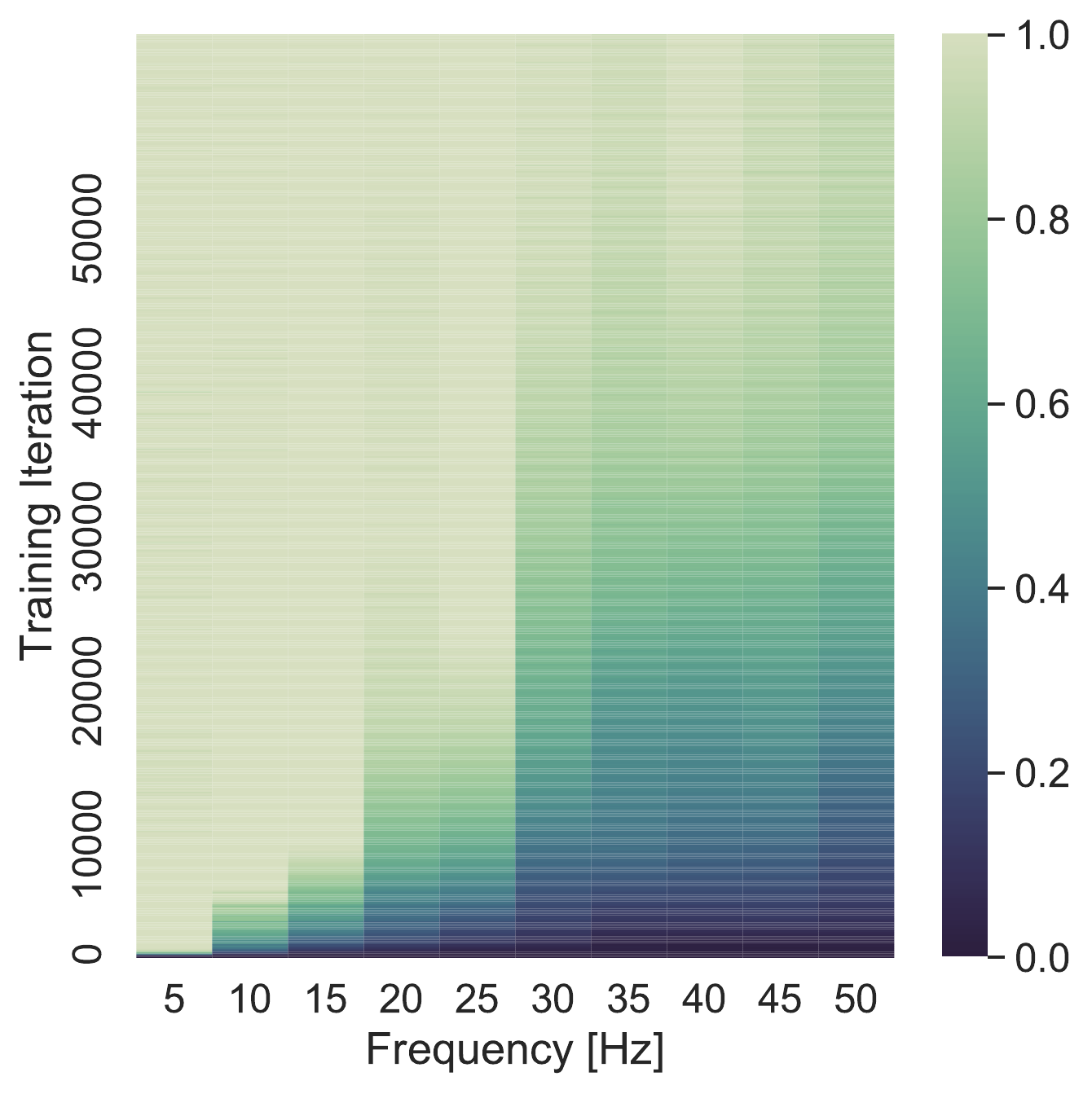}
    \includegraphics[width=0.3\textwidth]{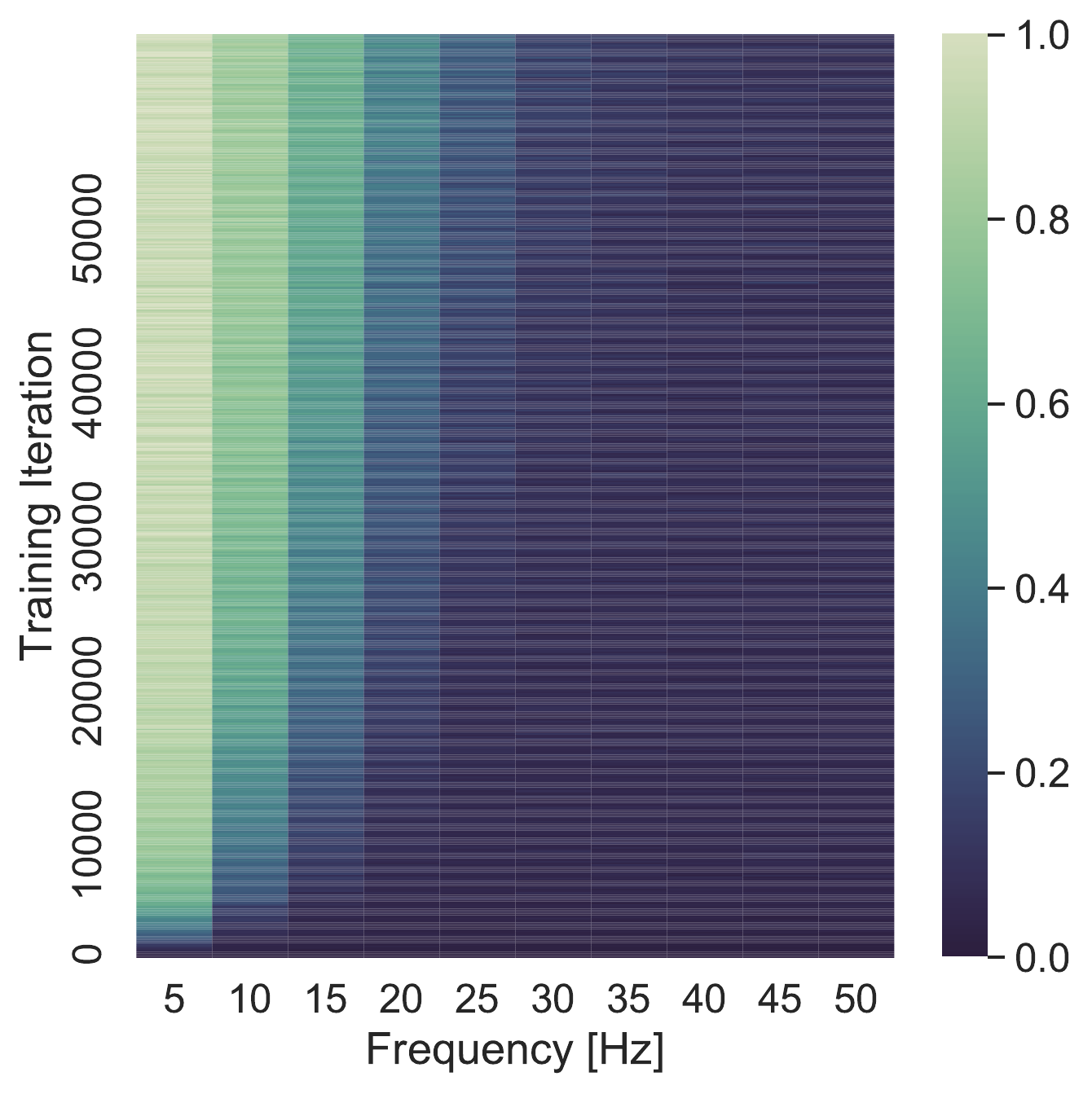}
    \includegraphics[width=0.3\textwidth]{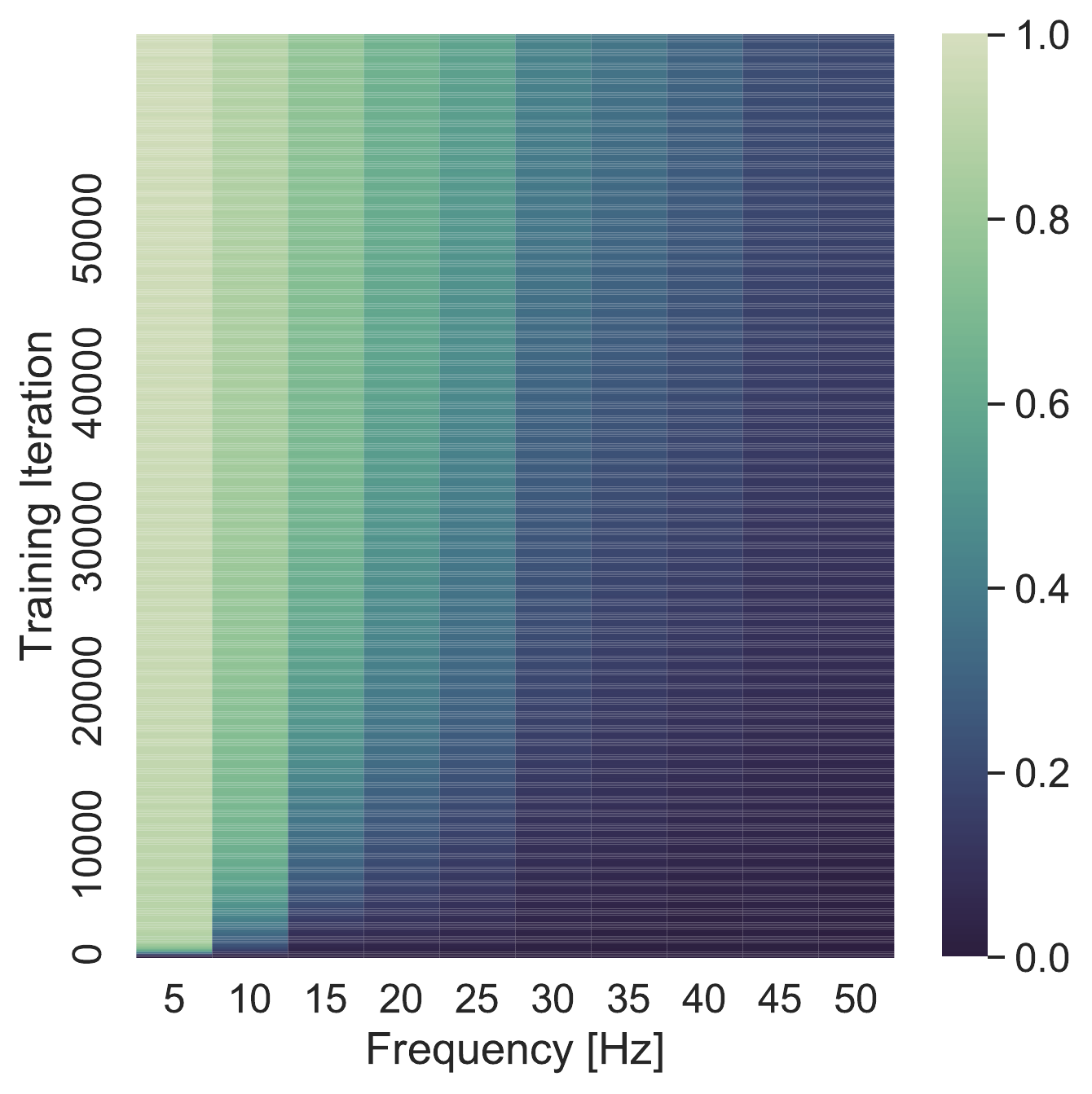} \\
    \vspace{-0.1cm}
    \hspace{-0.4cm}
    \subfigure[Baseline]{\includegraphics[width=0.3\textwidth]{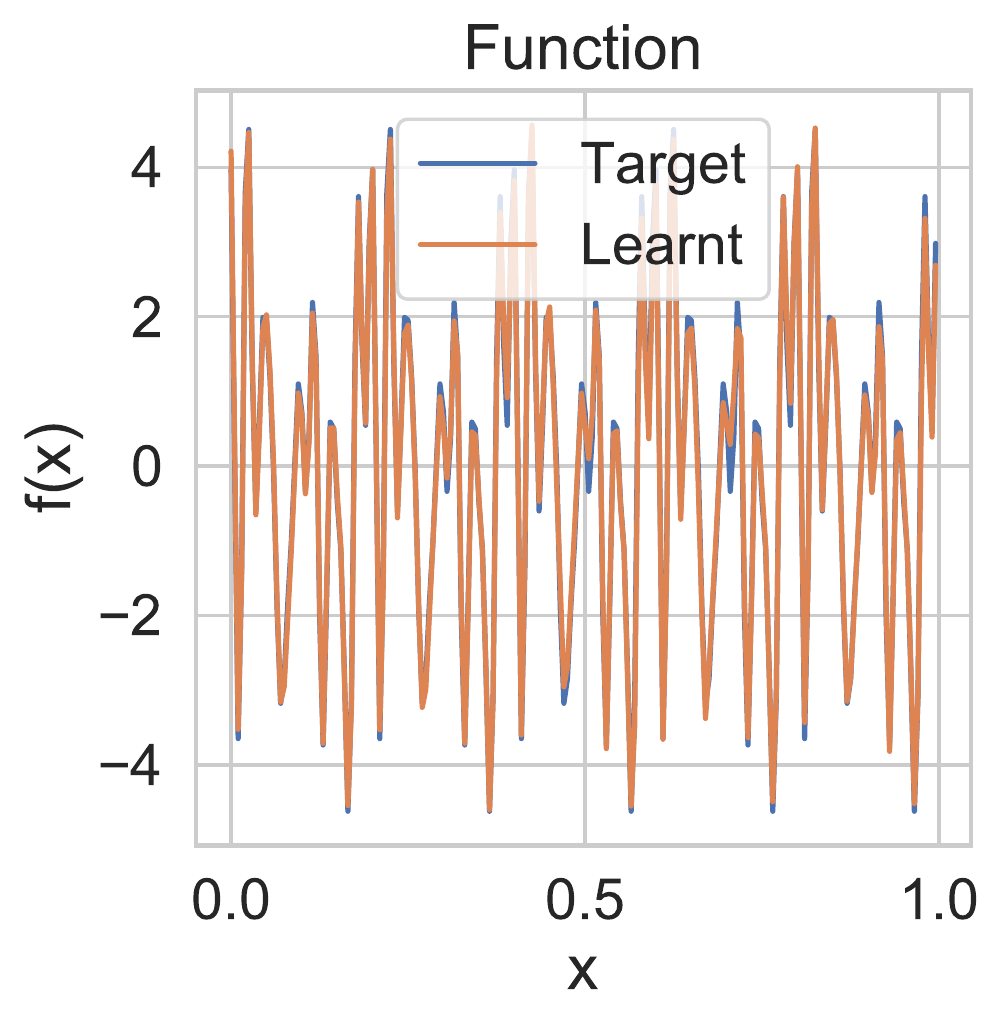}}
    \subfigure[Noise]{\includegraphics[width=0.3\textwidth]{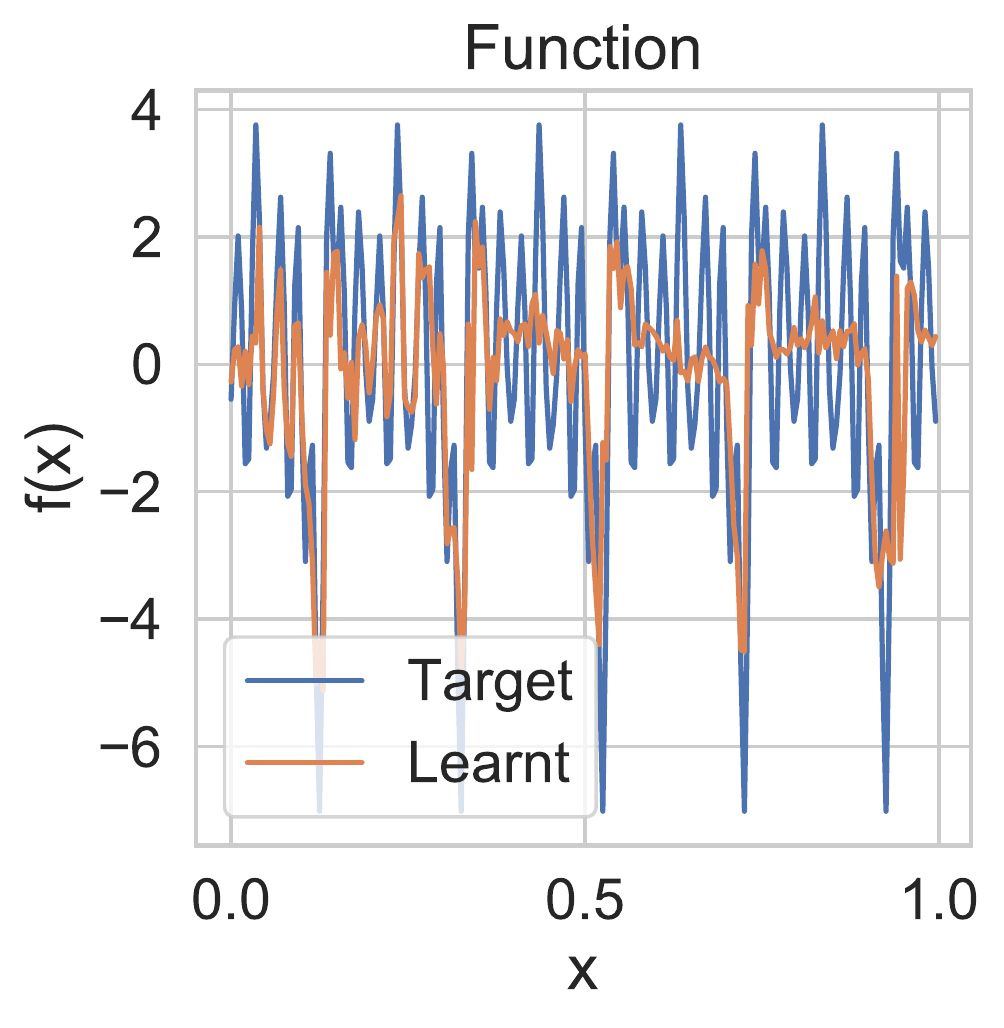}}
    \subfigure[$R$]{\includegraphics[width=0.3\textwidth]{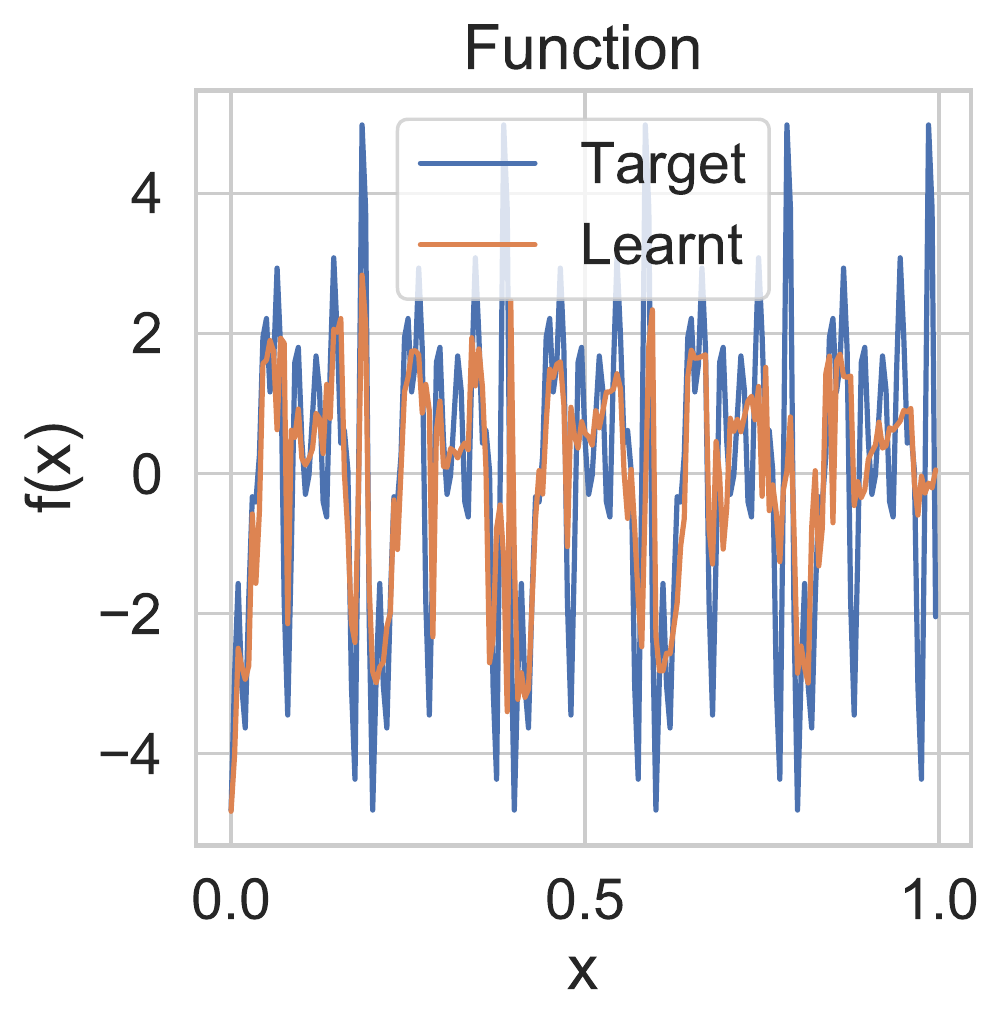}}
    
    \caption{As in \citet{Rahaman2019}, we train 6-layer deep 256-unit wide ReLU networks trained to regress the function $\lambda(z) = \sum_i \sin(2\pi r_iz+\phi(i))$ with $r_i  \in(5,10,\dots,45,50)$. 
    We train these networks with no noise (Baseline), with GNIs of variance $0.1$ injected into each layer except the final layer (Noise), and with the $R(\cdot)$ for regression in \eqref{eq:mse_reg_term}.
    The first row shows the Fourier spectrum (x-axis) of the networks as training progresses (y-axis) averaged over 10 training runs. Colours show each frequency's amplitude clipped between 0 and 1. 
    The second row shows samples of randomly generated target functions and the function learnt by the networks. 
    }
    \label{fig:network_fourier}
\end{figure}

\paragraph{Regression} 
Let us begin with the case of regression.
Assuming differentiable and continuous activation functions, then the Jacobians within $R$ are equivalent to the derivatives in Definition \ref{def:sobolev}.
Theorem \ref{th:plancherel-measure-space} only holds for functions that have 1-D outputs, but we can decompose the Jacobians $\v{J}_{k}$ as the derivatives of multiple 1-D output functions.
Recall, that the $i^\mathrm{th}$ row of the matrix $\v{J}_{k}$ is the set of partial derivatives of $f^{\theta}_{k,i}$, the function from layer $k$ to the $i^{\mathrm{th}}$ network output, $i = 1 ... d_L$, with respect to the $k^\mathrm{th}$ layer activations.
Using this perspective, and the fact that each $f^{\theta}_{k,i} \in W^{1,2}_{\mu}(\mathbb{R}^{d_k})$ ($d_k$ is the dimensionality of the $k^{\mathrm{th}}$ layer), if we assume that the probability measure of our space $\mu(\v{x})$ has compact support, we use Theorem \ref{th:plancherel-measure-space} to write:
\begin{align}
 R(\mathcal{B}; \v{\theta}) &= \frac{1}{2}\mathbb{E}_{\v{x} \sim \mathcal{B}} \left[ \sum_{k=0}^{L-1} \sigma_k^2 \sum_{i}\|\v{J}_{k,i}(\v{x})\|^2_{F} \right] \nonumber \\ &= \frac{1}{2}  \sum_{k=0}^{L-1} \sigma_k^2 \sum_{i} \mathbb{E}_{\v{x} \sim \mathcal{B}} \left[\|\v{J}_{k,i}(\v{x})\|^2_{F} \right] \nonumber \\ 
 &\approx \frac{1}{2}  \sum_{k=0}^{L-1} \sigma_k^2 \sum_{i}  \sum_{|\alpha|=1}\|D^\alpha f^{\theta}_{k,i}\|^2_{L^2_{\mu}(\mathbb{R}^{d_k})} \nonumber \\
 &= \frac{1}{2}\sum_{k=0}^{L-1} \sigma_k^2 \sum_{i} \int_{\mathbb{R}^{d_k}} \sum^{d_k}_{j=1} \Bigr|\mathcal{G}^{\theta}_{k,i}(\v{\omega},j)\overline{ \mathcal{G}^{\theta}_{k,i}(\v{\omega},j) * \mathcal{P}(\v{\omega})}\Bigr|   d\v{\omega}
 \label{eq:sobolev_reg_mse}
\end{align}
where $\v{h}_0 =\v{x}$, $i$ indexes over output neurons, and $\mathcal{G}^{\theta}_{k,i}(\v{\omega}, j) = \v{\omega}_j \mathcal{F}^{\theta}_{k,i}(\v{\omega})$, where $\mathcal{F}^{\theta}_{k,i}$ is the Fourier transform of the function $f^{\theta}_{k,i}$.
The approximation comes from the fact that in SGD, as mentioned above, integration over the dataset is approximated by sampling mini-batches $\mathcal{B}$.

If we take the data density function to be the empirical data density, meaning that it is supported on the $N$ points of the dataset $\mathcal{D}$ (i.e it is a set of $\delta$-functions centered on each point), then as the size $B$ of a batch $\mathcal{B}$ tends to $N$ we can write that:
\begin{align}
 \lim_{B\to N} R(\mathcal{B}; \v{\theta}) &=  \frac{1}{2}\sum_{k=0}^{L-1} \sigma_k^2 \sum_{i} \int_{\mathbb{R}^{d_k}} \sum^{d_k}_{j=1} \Bigr|\mathcal{G}^{\theta}_{k,i}(\v{\omega},j)\overline{ \mathcal{G}^{\theta}_{k,i}(\v{\omega},j) * \mathcal{P}(\v{\omega})}\Bigr|   d\v{\omega}.
\end{align}

\paragraph{Classification}
The classification setting requires a bit more work. 
Recall that our Jacobians are weighted by $\mathrm{diag}(\v{H}_{L}(\v{x}))^\intercal$, which has positive entries that are less than 1 by Equation \eqref{eq:hessian_classification}. We can define a new set of measures such that  $d\mu_i(\v{x}) = \mathrm{diag}(\v{H}_{L}(\v{x}))^\intercal_i p(\v{x}) d\v{x}, \ i = 1 \dots d_L$.
Because this new measure is positive, finite and still has compact support, Theorem \ref{th:plancherel-measure-space} still holds for the spaces indexed by $i$: $W^{1,2}_{\mu_i}(\mathbb{R}^d)$. 

Using these new measures, and the fact that each $f^{\theta}_{k,i} \in W^{1,2}_{\mu_i}(\mathbb{R}^{d_k})$, we can use Theorem \ref{th:plancherel-measure-space} to write that for classification models: 
\begin{align}
 R(\mathcal{B}; \v{\theta}) &=  \frac{1}{2}  \sum_{k=0}^{L-1} \sigma_k^2 \sum_{i} \mathbb{E}_{\v{x} \sim \mathcal{B}} \left[\mathrm{diag}(\v{H}_{L}(\v{x}))^\intercal_i \|\v{J}_{k,i}(\v{x})\|^2_{2} \right] 
 \nonumber\\ &\approx \frac{1}{2}  \sum_{k=0}^{L-1} \sigma_k^2 \sum_{i} \sum_{|\alpha|=1}\|D^\alpha f^{\theta}_{k,i}\|^2_{L^2_{\mu_i}(\mathbb{R}^{d_k})} \nonumber \\
 &=  \frac{1}{2}\sum_{k=0}^{L-1} \sigma_k^2 \sum_{i} \int_{\mathbb{R}^{d_k}}  \sum^{d_k}_{j=1} \Bigr|\mathcal{G}^{\theta}_{k,i}(\v{\omega}, j)\overline{ \mathcal{G}^{\theta}_{k,i}(\v{\omega}, j) * \mathcal{P}_i(\v{\omega})}\Bigr|   d\v{\omega}
 \label{eq:sobolev_reg_ce}
\end{align}

Here $\mathcal{P}_i$ is the Fourier transform of the $i^\mathrm{th}$ measure $\mu_i$ and as before $\mathcal{G}^{\theta}_{k,i}(\v{\omega}, j) = \v{\omega}_j \mathcal{F}^{\theta}_{k,i}(\v{\omega})$, where $\mathcal{F}^{\theta}_{k,i}$ is the Fourier transform of the function $f^{\theta}_{k,i}$.
Again as the batch size increases to the size of the dataset, this approximation becomes exact.

For both regression and classification, GNIs, by way of $R$, induce a prior which favours smooth functions with low-frequency components.
This prior is enforced by the terms $\mathcal{G}^{\theta}_{k,i}(\v{\omega}, j)$ which become large in magnitude when functions have high-frequency components, penalising neural networks that learn such functions. In Appendix \ref{app:tikhonov} we also show that this regularisation in the Fourier domain corresponds to a form of \textit{Tikhonov regularisation}.  

In Figure \ref{fig:network_fourier}, we demonstrate empirically that networks trained with GNIs learn functions that don't overfit; with lower-frequency components relative to their non-noised counterparts.

\paragraph{A layer-wise regularisation} Note that there is a recursive structure to the penalisation induced by $R$. 
Consider the layer-to-layer functions which map from a layer $k-1$ to $k$, $\v{h}_{k}(\v{h}_{k-1}(\v{x}))$. 
With a slight abuse of notation, $\nabla_{\v{h}_{k-1}}  \v{h}_{k}(\v{x})$ is the Jacobian defined element-wise as: 
\[\left(\nabla_{\v{h}_{k-1}}  \v{h}_{k}(\v{x})\right)_{i,j} = \frac{\partial h_{k,i}}{\partial h_{k-1,j}(\v{x})},\]
where as before $h_{k,i}$ is the $i^\mathrm{th}$ activation of layer $k$. 

$\|\nabla_{\v{h}_{k-1}}  \v{h}_{k}(\v{x})\|_2^2$ is penalised $k$ times in $R$ as this derivative appears in $\v{J}_0, \v{J}_1 \dots \v{J}_{k-1}$ due to the chain rule.
As such, when training with GNIs, we can expect the norm of $\|\nabla_{\v{h}_{k-1}}  \v{h}_{k}(\v{x})\|_2^2$ to decrease as the layer index $k$ increases (i.e the closer we are to the network output).
By Theorem \ref{th:plancherel-measure-space}, and Equations \eqref{eq:sobolev_reg_mse}, and \eqref{eq:sobolev_reg_ce}, larger  $\|\nabla_{\v{h}_{k-1}}  \v{h}_{k}(\v{x})\|_2^2$ correspond to functions with higher frequency components.
Consequently, when training with GNIs the layer to layer function $\v{h}_{k}(\v{h}_{k-1}(\v{x}))$ will have higher frequency components than the next layer's  function $\v{h}_{k+1}(\v{h}_{k}(\v{x}))$.

\begin{figure}[t!]

    \centering
    \subfigure[][Baseline  ]{\includegraphics[width=0.245\textwidth]{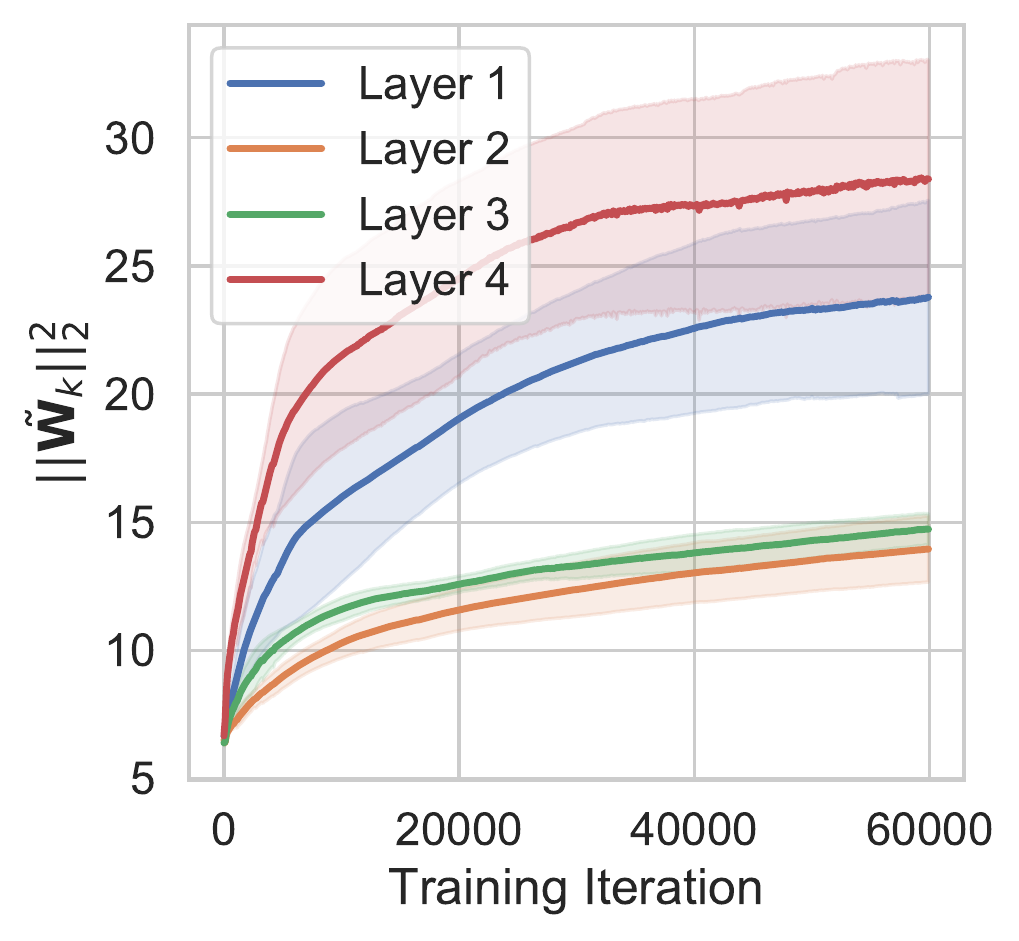}} 
    \subfigure[][GNI  ]{\includegraphics[width=0.245\textwidth]{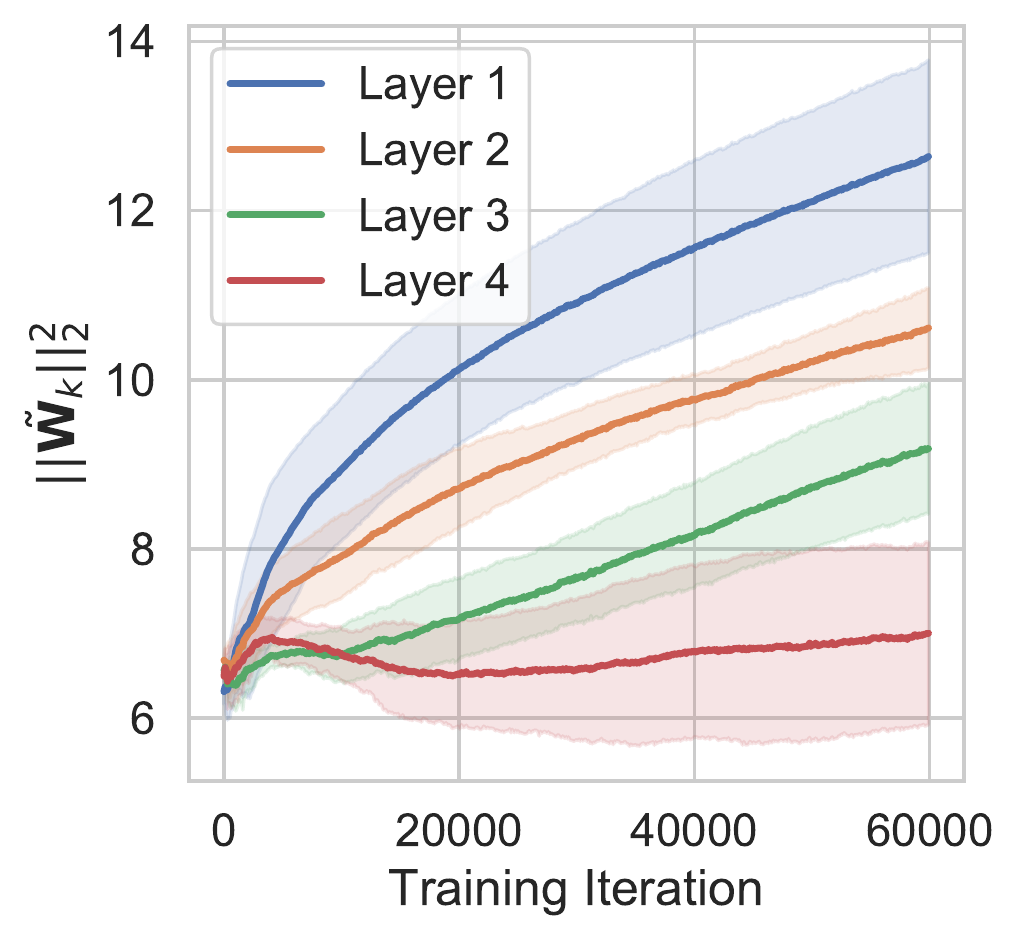}}
    \subfigure[][Baseline  ]{\includegraphics[width=0.245\textwidth]{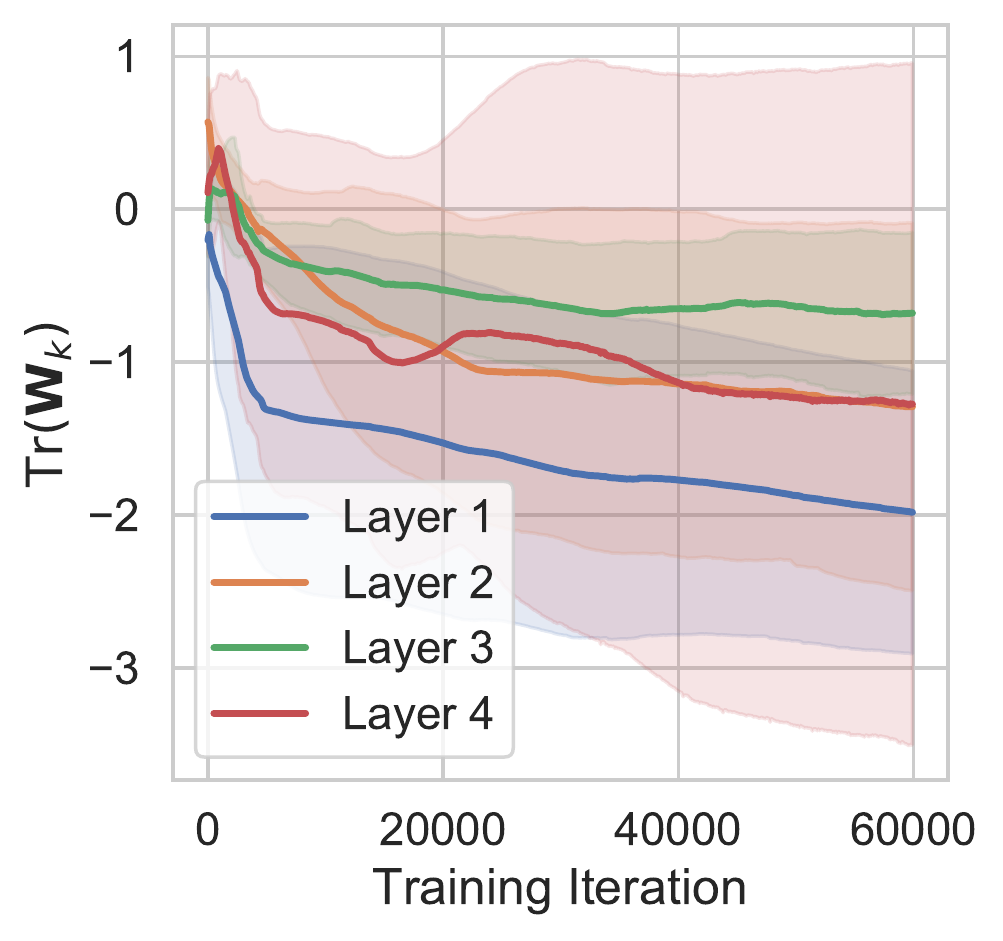}} 
    \subfigure[][GNI  ]{\includegraphics[width=0.245\textwidth]{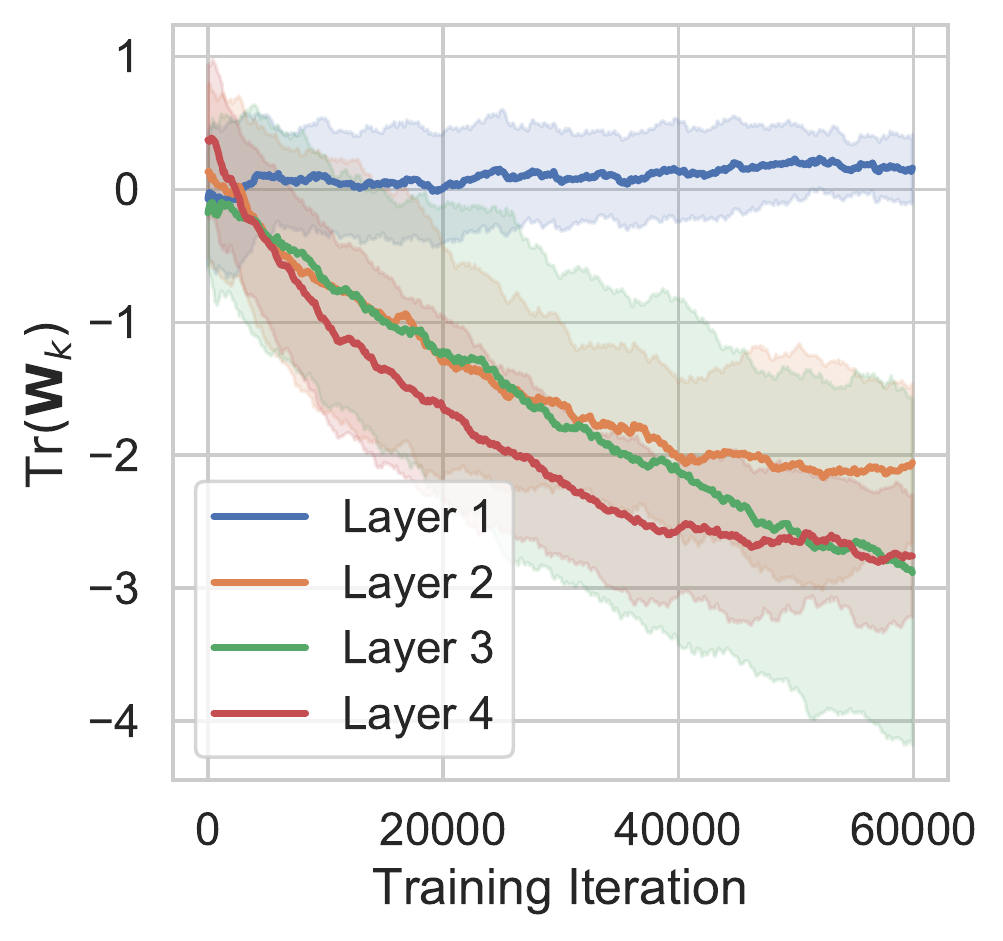}} 

    \caption{We use 6-layer deep 256-unit wide ReLU networks on the same dataset as in Figure \ref{fig:network_fourier} trained with (GNI) and without GNI (Baseline). 
    In (a,b), for layers with square weight matrices, we plot 
    $\|\widetilde{\v{W}}_k\|_2^2$.
    In (c,d) we plot the trace of these layers' weight matrices $\mathrm{Tr}(\v{W}_k)$. 
    For GNI models, as the layer index $k$ increases, $\mathrm{Tr}(\v{W}_k)$ and  $\|\widetilde{\v{W}}_k\|_2^2$ decrease, indicating that each successive layer in these networks learns a function with lower frequency components than the past layer.  
    }
    \label{fig:per_layer_reg}
\end{figure}

We measure this layer-wise regularisation in $\mathrm{ReLU}$ networks, using $\nabla_{\v{h}_{k-1}}  \v{h}_{k}(\v{x}) = \widetilde{\v{W}}_k$.
$\widetilde{\v{W}}_k$ is obtained from the original weight matrix $\v{W}_k$ by setting its $i^{\mathrm{th}}$ column to zero whenever the neuron $i$ of the $k^{\mathrm{th}}$ layer is inactive.
Also note that the inputs of $\mathrm{ReLU}$ network hidden layers, which are the outputs of another $\mathrm{ReLU}$-layer, will be positive.
Negative weights are likely to `deactivate' a $\mathrm{ReLU}$-neuron, inducing smaller $\|\widetilde{\v{W}}_k\|_2^2$, and thus parameterising a lower frequency function.
We use the trace of a weight matrix as an indicator for the `number' of negative components.

In Figure \ref{fig:per_layer_reg} we demonstrate that $\|\widetilde{\v{W}}_k\|_2^2$ \textit{and} $\mathrm{Tr}(\v{W}_k)$  decrease as $k$ increases for $\mathrm{ReLU}$-networks trained with GNIs, indicating that each successive layer in these networks learns a function with lower frequency components than the past layer. 
This striation and ordering is clearly absent in the baselines trained without GNIs.

\paragraph{The Benefits of GNIs}

What does regularisation in the Fourier domain accomplish? 
The terms in $R$ are the traces of the Gauss-Newton decompositions of the second order derivatives of the loss. 
By penalising this we are more likely to land in wider (smoother) minima (see Figure \ref{fig:marg_approx}), which has been shown,  contentiously \citep{Dinh2017}, to induce networks with better generalisation properties \citep{Keskar2019OnMinima, Jastrzebski2017}. 
GNIs however, confer other benefits too.

\textit{Sensitivity.}  A model's weakness to input perturbations is termed the \textit{sensitivity}. 
\citet{Rahaman2019} have shown empirically that classifiers biased towards lower frequencies in the Fourier domain are less sensitive,
and there is ample evidence demonstrating that models trained with noised data are less sensitive \citep{Liu2018, Li2018}.
The Fourier domain - sensitivity connection can be established by studying the \textit{classification margins} of a model (see Appendix \ref{app:classification_margins}).

\textit{Calibration.}
Given a network's prediction $\hat{y}(\v{x})$ with confidence $\hat{p}(\v{x})$ for a point $\v{x}$, perfect calibration consists of being as likely to be correct as you are confident: $p(\hat{y}=y|\hat{p}=r)=r, \,\, \forall r\in[0,1]$ \citep{Dawid1982, Degroot1982TheForecasters}. 
In Appendix \ref{app:capacity} we show that models that are biased toward lower frequency spectra have lower `capacity measures', which measure model complexity and lower values of which have been shown empirically to induce better calibrated models \citep{Guo2017}. 
In Figure \ref{fig:calibration_app}
we show that this holds true for models trained with GNIs.

\section{Related Work}

Many variants of GNIs have been proposed to regularise neural networks.
\citet{Poole2014} extend this process and apply noise to all computational steps in a neural network layer. Not only is noise applied to the layer input it is applied to the layer output and to the pre-activation function logits.
The authors allude to explicit regularisation but only derive a result for a single layer auto-encoder with a single noise injection.
Similarly, \citet{Bishop1995} derive an analytic form for the explicit regulariser induced by noise injections on \textit{data} and show that such injections are equivalent to Tikhonov regularisation in an unspecified function space. 

Recently \citet{Wei2020} conducted similar analysis to ours, dividing the effects of Bernoulli dropout into \textit{explicit} and \textit{implicit} effects. 
Their work is built on that of \citet{Mele1993}, \citet{Helmbold2015}, and \citet{wager2013dropout} who perform this analysis for linear neural networks. 
\citet{Arora2020} derive an explicit regulariser for Bernoulli dropout on the final layer of a neural network.
Further, recent work by \citet{dieng2018noisin} shows that noise additions on recurrent network hidden states outperform Bernoulli dropout in terms of performance and bias.

\section{Conclusion}

In this work, we derived analytic forms for the explicit regularisation induced by Gaussian noise injections, demonstrating
that the explicit regulariser penalises networks with high-frequency content in Fourier space.
Further we show that this regularisation is not distributed evenly within a network, as it disproportionately penalises high-frequency content in layers closer to the network output.
Finally we demonstrate that this regularisation in the Fourier domain has a number of beneficial effects. 
It induces training dynamics that preferentially land in wider minima, it reduces model sensitivity to noise, and induces better calibration.

\newpage

\section*{Acknowledgments}

This research was directly funded by the Alan Turing Institute under Engineering and Physical Sciences Research Council (EPSRC) grant EP/N510129/1.
AC was supported by an EPSRC Studentship.
MW was supported by EPSRC grant EP/G03706X/1.
U\c{S} was supported by the French National Research Agency (ANR) as a part of the FBIMATRIX (ANR-16-CE23-0014) project.
SR gratefully acknowledges support from the UK Royal Academy of Engineering and the Oxford-Man Institute.
CH was supported by the Medical Research Council, the Engineering and Physical Sciences Research Council, Health Data Research UK, and the Li Ka Shing Foundation

\section*{Impact Statement}

This paper uncovers a new mechanism by which a widely used regularisation method operates and paves the way for designing new regularisation methods which take advantage of our findings. 
Regularisation methods produce models that are not only less likely to overfit, but also have better calibrated predictions that are more robust to distribution shifts. 
As such improving our understanding of such methods is critical as machine learning models become increasingly ubiquitous and embedded in decision making.

\bibliographystyle{plainnat}
\bibliography{references.bib}

\begin{thebibliography}{47}
\providecommand{\natexlab}[1]{#1}
\providecommand{\url}[1]{\texttt{#1}}
\expandafter\ifx\csname urlstyle\endcsname\relax
  \providecommand{\doi}[1]{doi: #1}\else
  \providecommand{\doi}{doi: \begingroup \urlstyle{rm}\Url}\fi

\bibitem[Aleksziev(2019)]{Aleksziev}
Rita Aleksziev.
\newblock {Tangent Space Separability in Feedforward Neural Networks}.
\newblock In \emph{NeurIPS}, 2019.

\bibitem[Arora et~al.(2020)Arora, Bartlett, Mianjy, and Srebro]{Arora2020}
Raman Arora, Peter Bartlett, Poorya Mianjy, and Nathan Srebro.
\newblock {Dropout: Explicit Forms and Capacity Control}.
\newblock 2020.

\bibitem[Arora et~al.(2019)Arora, Du, Hu, Li, Salakhutdinov, and
  Wang]{NEURIPS2019_dbc4d84b}
Sanjeev Arora, Simon~S Du, Wei Hu, Zhiyuan Li, Russ~R Salakhutdinov, and
  Ruosong Wang.
\newblock On exact computation with an infinitely wide neural net.
\newblock In \emph{Advances in Neural Information Processing Systems},
  volume~32. Curran Associates, Inc., 2019.

\bibitem[Bishop(1995)]{Bishop1995}
Chris~M. Bishop.
\newblock {Training with Noise is Equivalent to Tikhonov Regularization}.
\newblock \emph{Neural Computation}, 7\penalty0 (1):\penalty0 108--116, 1995.

\bibitem[Botev et~al.(2017)Botev, Ritter, and Barber]{Botev2017}
Aleksandar Botev, Hippolyt Ritter, and David Barber.
\newblock {Practical Gauss-Newton optimisation for deep learning}.
\newblock In \emph{ICML}, 2017.

\bibitem[Burger and Neubauer(2003)]{Burger2003}
Martin Burger and Andreas Neubauer.
\newblock {Analysis of Tikhonov regularization for function approximation by
  neural networks}.
\newblock \emph{Neural Networks}, 16\penalty0 (1):\penalty0 79--90, 2003.

\bibitem[Chen et~al.(2020)Chen, Cao, Gu, and Zhang]{chen2020generalized}
Zixiang Chen, Yuan Cao, Quanquan Gu, and Tong Zhang.
\newblock A generalized neural tangent kernel analysis for two-layer neural
  networks, 2020.

\bibitem[Chizat et~al.(2019)Chizat, Oyallon, and Bach]{chizat:hal-01945578}
Lenaic Chizat, Edouard Oyallon, and Francis Bach.
\newblock {On Lazy Training in Differentiable Programming}.
\newblock In \emph{NeurIPS}, December 2019.

\bibitem[Cohen et~al.(2019)Cohen, Rosenfeld, and Kolter]{Cohen2019}
Jeremy Cohen, Elan Rosenfeld, and J.~Zico Kolter.
\newblock {Certified adversarial robustness via randomized smoothing}.
\newblock In \emph{ICML}, 2019.

\bibitem[Constantine and Savits(1996)]{constantine}
G.~M. Constantine and T.~H. Savits.
\newblock A multivariate faa di bruno formula with applications.
\newblock \emph{Transactions of the American Mathematical Society},
  348\penalty0 (2):\penalty0 503--520, 1996.

\bibitem[Cucker and Smale(2002)]{Cucker2002}
Felipe Cucker and Steve Smale.
\newblock {On the mathematical foundations of learning}.
\newblock \emph{Bulletin of the American Mathematical Society}, 39\penalty0
  (1):\penalty0 1--49, 2002.

\bibitem[Czarnecki et~al.(2017)Czarnecki, Osindero, Jaderberg, Swirszcz, and
  Pascanu]{Czarnecki2017}
Wojciech~Marian Czarnecki, Simon Osindero, Max Jaderberg, Grzegorz Swirszcz,
  and Razvan Pascanu.
\newblock {Sobolev training for neural networks}.
\newblock In \emph{NeurIPS}, 2017.

\bibitem[Dawid(1982)]{Dawid1982}
A~P Dawid.
\newblock {The Well-Calibrated Bayesian}.
\newblock \emph{Journal of the American Statistical Association}, 77\penalty0
  (379), 1982.

\bibitem[DeGroot and Fienberg(1983)]{Degroot1982TheForecasters}
Morris~H. DeGroot and Stephen~E. Fienberg.
\newblock The comparison and evaluation of forecasters.
\newblock \emph{Journal of the Royal Statistical Society. Series D (The
  Statistician)}, 32:\penalty0 12--22, 1983.

\bibitem[Dieng et~al.(2018)Dieng, Ranganath, Altosaar, and
  Blei]{dieng2018noisin}
Adji~B Dieng, Rajesh Ranganath, Jaan Altosaar, and David~M Blei.
\newblock Noisin: Unbiased regularization for recurrent neural networks.
\newblock \emph{arXiv preprint arXiv:1805.01500}, 2018.

\bibitem[Dinh et~al.(2017)Dinh, Pascanu, Bengio, and Bengio]{Dinh2017}
Laurent Dinh, Razvan Pascanu, Samy Bengio, and Yoshua Bengio.
\newblock {Sharp minima can generalize for deep nets}.
\newblock In \emph{ICML}, 2017.

\bibitem[Farquhar et~al.(2020)Farquhar, Smith, and Gal]{Farquhar2020}
Sebastian Farquhar, Lewis Smith, and Yarin Gal.
\newblock {Try Depth Instead of Weight Correlations: Mean-field is a Less
  Restrictive Assumption for Deeper Networks}.
\newblock In \emph{NeurIPS}, 2020.

\bibitem[Girosi and Poggio(1990)]{Girosi1990}
F~Girosi and T~Poggio.
\newblock {Biological Cybernetics Networks and the Best Approximation
  Property}.
\newblock \emph{Artificial Intelligence}, 176:\penalty0 169--176, 1990.

\bibitem[Guo et~al.(2017)Guo, Pleiss, Sun, and Weinberger]{Guo2017}
Chuan Guo, Geoff Pleiss, Yu~Sun, and Kilian~Q. Weinberger.
\newblock {On calibration of modern neural networks}.
\newblock In \emph{ICML}, 2017.

\bibitem[Hauser and Ray(2017)]{Hauser2017}
Michael Hauser and Asok Ray.
\newblock {Principles of Riemannian geometry in neural networks}.
\newblock In \emph{NeurIPS}, 2017.

\bibitem[Helmbold and Long(2015)]{Helmbold2015}
David~P. Helmbold and Philip~M. Long.
\newblock {On the inductive bias of dropout}.
\newblock \emph{Journal of Machine Learning Research}, 16:\penalty0 3403--3454,
  2015.

\bibitem[Hornik(1991)]{Hornik1991}
Kurt Hornik.
\newblock {Approximation capabilities of multilayer feedforward networks}.
\newblock \emph{Neural Networks}, 4\penalty0 (2):\penalty0 251--257, 1991.

\bibitem[Jacot et~al.(2018)Jacot, Gabriel, and Hongler]{ntk}
Arthur Jacot, Franck Gabriel, and Clement Hongler.
\newblock Neural tangent kernel: Convergence and generalization in neural
  networks.
\newblock In \emph{NeurIPS}, 2018.

\bibitem[Jakubovitz and Giryes(2018)]{Jakubovitz2018}
Daniel Jakubovitz and Raja Giryes.
\newblock {Improving DNN robustness to adversarial attacks using jacobian
  regularization}.
\newblock \emph{Lecture Notes in Computer Science}, pages 525--541, 2018.

\bibitem[Jastrz{\c{e}}bski et~al.(2017)Jastrz{\c{e}}bski, Kenton, Arpit,
  Ballas, Fischer, Bengio, and Storkey]{Jastrzebski2017}
Stanisław Jastrz{\c{e}}bski, Zachary Kenton, Devansh Arpit, Nicolas Ballas,
  Asja Fischer, Yoshua Bengio, and Amos Storkey.
\newblock {Three Factors Influencing Minima in SGD}.
\newblock In \emph{NeurIPS}, 2017.

\bibitem[Keskar et~al.(2019)Keskar, Nocedal, Tang, Mudigere, and
  Smelyanskiy]{Keskar2019OnMinima}
Nitish~Shirish Keskar, Jorge Nocedal, Ping Tak~Peter Tang, Dheevatsa Mudigere,
  and Mikhail Smelyanskiy.
\newblock {On large-batch training for deep learning: Generalization gap and
  sharp minima}.
\newblock In \emph{ICLR}, 2019.

\bibitem[Kingma et~al.(2015)Kingma, Salimans, and Welling]{Kingma}
Diederik~P. Kingma, Tim Salimans, and Max Welling.
\newblock {Variational dropout and the local reparameterization trick}.
\newblock In \emph{NeurIPS}, 2015.

\bibitem[Kunin et~al.(2019)Kunin, Bloom, Goeva, and Seed]{Kunin1999}
Daniel Kunin, Jonathan~M. Bloom, Aleksandrina Goeva, and Cotton Seed.
\newblock {Loss landscapes of regularized linear autoencoders}.
\newblock In \emph{ICML}, 2019.

\bibitem[LeCun et~al.(1998)LeCun, Bottou, Orr, and M{\"u}ller]{LeCun1998}
Yann~A. LeCun, L{\'e}on Bottou, Genevieve~B. Orr, and Klaus-Robert M{\"u}ller.
\newblock \emph{Efficient BackProp}, pages 9--48.
\newblock 1998.

\bibitem[Li et~al.(2018)Li, Chen, Wang, and Carin]{Li2018}
Bai Li, Changyou Chen, Wenlin Wang, and Lawrence Carin.
\newblock Second-order adversarial attack and certifiable robustness.
\newblock \emph{CoRR}, 2018.

\bibitem[Liu et~al.(2019)Liu, Dong, Zhang, Gong, and Shi]{Liu2018}
Yuhang Liu, Wenyong Dong, Lei Zhang, Dong Gong, and Qinfeng Shi.
\newblock {Variational bayesian dropout with a hierarchical prior}.
\newblock In \emph{IEEE CVPR}, 2019.

\bibitem[Mele and Altarelli(1993)]{Mele1993}
Barbara Mele and Guido Altarelli.
\newblock {Lepton spectra as a measure of b quark polarization at LEP}.
\newblock \emph{Physics Letters B}, 299\penalty0 (3-4):\penalty0 345--350,
  1993.

\bibitem[Naeini et~al.(2015)Naeini, Cooper, and Hauskrecht]{Naeini2015}
Mahdi~Pakdaman Naeini, Gregory~F. Cooper, and Milos Hauskrecht.
\newblock {Obtaining well calibrated probabilities using Bayesian Binning}.
\newblock \emph{Proceedings of the National Conference on Artificial
  Intelligence}, 4:\penalty0 2901--2907, 2015.

\bibitem[Neyshabur et~al.(2015)Neyshabur, Tomioka, and Srebro]{Neyshabur}
Behnam Neyshabur, Ryota Tomioka, and Nathan Srebro.
\newblock Norm-based capacity control in neural networks.
\newblock In \emph{PMLR}, 2015.

\bibitem[Neyshabur et~al.(2017)Neyshabur, Bhojanapalli, McAllester, and
  Srebro]{Neyshabur2017}
Behnam Neyshabur, Srinadh Bhojanapalli, David McAllester, and Nathan Srebro.
\newblock {Exploring generalization in deep learning}.
\newblock In \emph{NeurIPS}, 2017.

\bibitem[Niculescu-Mizil and Caruana(2005)]{Niculescu-Mizil2005}
Alexandru Niculescu-Mizil and Rich Caruana.
\newblock {Predicting good probabilities with supervised learning}.
\newblock In \emph{ICML}, 2005.

\bibitem[Poole et~al.(2014)Poole, Sohl-Dickstein, and Ganguli]{Poole2014}
Ben Poole, Jascha Sohl-Dickstein, and Surya Ganguli.
\newblock {Analyzing noise in autoencoders and deep networks}.
\newblock 2014.

\bibitem[Poole et~al.(2016)Poole, Lahiri, Raghu, Sohl-Dickstein, and
  Ganguli]{Poole2016}
Ben Poole, Subhaneil Lahiri, Maithra Raghu, Jascha Sohl-Dickstein, and Surya
  Ganguli.
\newblock {Exponential expressivity in deep neural networks through transient
  chaos}.
\newblock In \emph{NeurIPS}, 2016.

\bibitem[Rahaman et~al.(2019)Rahaman, Baratin, Arpit, Draxler, Lin, Hamprecht,
  Bengio, and Courville]{Rahaman2019}
Nasim Rahaman, Aristide Baratin, Devansh Arpit, Felix Draxler, Min Lin, Fred
  Hamprecht, Yoshua Bengio, and Aaron Courville.
\newblock On the spectral bias of neural networks.
\newblock In \emph{ICML}, 2019.

\bibitem[Sagun et~al.(2018)Sagun, Evci, G{\"{u}}ney, Dauphin, and
  Bottou]{Sagun2018}
Levent Sagun, Utku Evci, V.~Ugur G{\"{u}}ney, Yann Dauphin, and L{\'{e}}on
  Bottou.
\newblock {Empirical analysis of the hessian of over-parametrized neural
  networks}.
\newblock 2018.

\bibitem[Sokoli{\'{c}} et~al.(2017)Sokoli{\'{c}}, Giryes, Sapiro, and
  Rodrigues]{Sokolic2017}
Jure Sokoli{\'{c}}, Raja Giryes, Guillermo Sapiro, and Miguel~R.D. Rodrigues.
\newblock {Robust Large Margin Deep Neural Networks}.
\newblock \emph{IEEE Transactions on Signal Processing}, 65\penalty0
  (16):\penalty0 4265--4280, 2017.

\bibitem[Srivastava et~al.(2014)Srivastava, Hinton, Krizhevsky, Sutskever, and
  Salakhutdinov]{Srivastava2014}
Nitish Srivastava, Geoffrey Hinton, Alex Krizhevsky, Ilya Sutskever, and Ruslan
  Salakhutdinov.
\newblock {Dropout: A simple way to prevent neural networks from overfitting}.
\newblock \emph{Journal of Machine Learning Research}, 15:\penalty0 1929--1958,
  2014.

\bibitem[Tikhonov(1977)]{Tikhonov1977}
A~N (Andrei~Nikolaevich) Tikhonov.
\newblock \emph{{Solutions of ill-posed problems / Andrey N. Tikhonov and
  Vasiliy Y. Arsenin ; translation editor, Fritz John}}.
\newblock 1977.

\bibitem[Wager et~al.(2013)Wager, Wang, and Liang]{wager2013dropout}
Stefan Wager, Sida Wang, and Percy~S Liang.
\newblock Dropout training as adaptive regularization.
\newblock In \emph{Advances in neural information processing systems}, pages
  351--359, 2013.

\bibitem[Webb(1994)]{Webb1994}
Andrew~R. Webb.
\newblock {Functional Approximation by FeedForward Networks: A Least-Squares
  Approach to Generalization}.
\newblock \emph{IEEE Transactions on Neural Networks}, 5\penalty0 (3):\penalty0
  363--371, 1994.

\bibitem[Wei et~al.(2020)Wei, Kakade, and Ma]{Wei2020}
Colin Wei, Sham Kakade, and Tengyu Ma.
\newblock {The Implicit and Explicit Regularization Effects of Dropout}.
\newblock 2020.

\bibitem[Zhang et~al.(2017)Zhang, Recht, Bengio, Hardt, and Vinyals]{Zhang2019}
Chiyuan Zhang, Benjamin Recht, Samy Bengio, Moritz Hardt, and Oriol Vinyals.
\newblock {Understanding deep learning requires rethinking generalization}.
\newblock In \emph{ICLR}, 2017.

\end{thebibliography}

\newpage
\appendix
\clearpage
\setcounter{equation}{0}
\setcounter{theorem}{0}
\setcounter{figure}{0}
\setcounter{prop}{0}
\renewcommand\thefigure{\thesection.\arabic{figure}}
\setcounter{table}{0}
\renewcommand{\thetable}{\thesection.\arabic{table}}
\setcounter{definition}{0}
\renewcommand{\thedefinition}{\thesection.\arabic{definition}}

\clearpage

\section{Technical Proofs}
\subsection{Proof of Proposition~\ref{prop:accum_noise}}
\label{app:accum_noise}

\begin{proof}[Proof of Proposition~\ref{prop:accum_noise}]
Recall that $\v{h}$ denotes the vanilla activations of the network, those we obtain with no noise injection.
Let us \textit{not} inject noise in the final, predictive, layer of our network such that the noise on this layer is accumulated from the noising of previous layers.

We denote $\v{\mathcal{E}}_{k}$ the noise accumulated at layer $k$ from GNIs in previous layers, and potential GNIs at layer $k$ itself. 
We denote $\mathcal{E}_{L,i}$ the $i^{\mathrm{th}}$ element of the noise at layer $L$, the layer to which we do not add noise. 
This can be defined as a Taylor expansion around the accumulated noise at the previous layer $L-1$:
 \begin{equation}
    \mathcal{E}_{L,i} =
    \sum_{|\alpha_L|=1}^{\infty}
    \frac{1}{\alpha_L!}
     \left(D^{\alpha_L} h_{L,i}(\v{h}_{L-1}(\v{x}))\right) \v{\mathcal{E}}_{L-1}^{\alpha_L}
 \end{equation}

 where we use $\alpha_L$ as a multi-index over derivatives.
 
Generally if we noise all layers up to the penultimate layer of index $L-1$ we can define the accumulated noise at layer $k$, $\v{\mathcal{E}}_k$ recursively because Gaussian have finite moments: 
 \begin{equation}
     \mathcal{E}_{k,i} = \epsilon_{k,i} +  \sum_{|\alpha_k|=1}^{\infty}
     \frac{1}{\alpha_k!}
    \left(D^{\alpha_k} h_{k,i}(\v{h}_{k-1}(\v{x}))\right) 
    \v{\mathcal{E}}_{k-1}^{\alpha_k} , \ i = 1, \dots, d_k,\ k = 0 \dots L-1 
 \end{equation}
 
 where $\v{\mathcal{E}}_0 = \v{\epsilon}_0$ is the base case.

\end{proof} 

\subsection{Proof of Theorem 1}
\label{app:exp_reg}

\begin{proof}[Proof of Theorem\ref{prop:explicit_reg}]

Let us first consider the Taylor series expansion of the loss function with the accumulated noise defined in Proposition~\ref{prop:accum_noise}. 
Denoting $\v{\epsilon} = [\v{\epsilon}_{L-1}, \dots, \v{\epsilon}_{0}]$ we have: 
 \begin{align}
    &\expect_{\v{\epsilon} } \left[\mathcal{L}(\v{h}_{L}(\v{x}) + \v{\mathcal{E}}_L,  \v{y})\right] = \mathcal{L}(\v{x}, \v{y})  +  \expect_{\v{\epsilon} } \left[\sum_{|\alpha|=1}^{\infty}
    \frac{1}{\alpha!}
    \left(D^{\alpha}\mathcal{L}(\v{h}_{L}(\v{x}), \v{y}) \right)
    \v{\mathcal{E}}_{L}^{\alpha}\right] 
\end{align}
Note that the dot product with the $i^{\mathrm{th}}$ element of the final layer noise $\mathcal{E}_{L,i}$ can be written as 
 \begin{align}
    &\expect_{\v{\epsilon} } \left[\sum_{|\alpha|=1}^{\infty}
    \frac{1}{\alpha!}
    \left(D^{\alpha}\mathcal{L}(\v{h}_{L}(\v{x}), \v{y}) \right)
    \mathcal{E}_{L,i}\right] \nonumber\\
    &= \expect_{\v{\epsilon} } \left[\sum_{|\alpha|=1}^{\infty}
    \frac{1}{\alpha!}
    \left(D^{\alpha}\mathcal{L}(\v{h}_{L}(\v{x}), \v{y}) \right)
   \left(\sum_{|\alpha_L|=1}^{\infty}
    \frac{1}{\alpha_L!}
     \left(D^{\alpha_L} h_{L,i}(\v{h}_{L-1}(\v{x}))\right) \v{\mathcal{E}}_{L-1}^{\alpha_L}\right)\right] \\
     &= \expect_{\v{\epsilon} } \left[\sum_{|\alpha|=1}^{\infty}
    \frac{1}{\alpha!}
    \left(D^{\alpha}\mathcal{L}(\v{h}_{L}(\v{x}), \v{y}) \right)
   \left(\sum_{|\alpha_L|=1}^{\infty}
    \frac{1}{\alpha_L!}
     \left(D^{\alpha_L} h_{L,i}(\v{h}_{L-1}(\v{x}))\right) (\v{\epsilon}_{L-1} + \dots)^{\alpha_L}\right)\right] 
\end{align}
where the dots here denote the accumulated noise term on layer $L-1$ \textit{before} we add the Gaussian noise $\v{\epsilon}_{L-1}$. 
When looking at all elements of $\v{\mathcal{E}}_{L}^{\alpha}$, not just the $i^\mathrm{th}$ element, note that this is essentially the Taylor series expansion of $\mathcal{L}$ around the series expansion of $\v{h}_L$ around $\v{\epsilon}_{L-1}$. 
We know that the product of the Taylor series of a composed function $f \circ g$ with the Taylor series of $g$ is simply the Taylor series of $f$ around $\v{x}$ \citep{constantine}.
This can be deduced from the slightly opaque Faà di Bruno's formula, which states that for multivariate derivatives of a composition of functions $f:\mathbb{R}^m \to \mathbb{R}$ and $g: \mathbb{R}^d \to \mathbb{R}^m$ and a multi-index $\alpha$ \citep{constantine}
\[
D^{\alpha} (f \circ g)(\v{x}) = \sum_{1 \leq |\lambda| \leq |\alpha|} D^{\lambda}f(g(\v{x})) \sum_{s=1}^{|\alpha|} \sum_{p_s(\lambda,\alpha)} (\alpha!) \prod_{j=1}^{s} \frac{(D^{l_j}g(\v{x}))^{k_j}}{(k_j!)[l_j!]^{|k_j|}}, 
\]
where $p_s(\lambda,\alpha) = \{(k_1,\dots,k_s);(l_1,\dots,l_s): |k_i| > 0, \ 0 \prec l_1 \dots \prec l_s, \ \sum^s_{i=1}k_i=\lambda \sum^s_{i=1}|k_i|l_i=\alpha$ , where $\prec$ denotes a partial order.

Applying this recursively to each layer $k$, we obtain that,
\begin{align}
    &\expect_{\v{\epsilon} } \left[\sum_{|\alpha|=1}^{\infty}
    \frac{1}{\alpha!}
    \left(D^{\alpha}\mathcal{L}(\v{h}_{L}(\v{x}), \v{y}) \right)
    \v{\mathcal{E}}_{L}^{\alpha}\right] \nonumber \\
     &=  \expect_{\v{\epsilon} } \left[
    \sum_{k=0}^{L-1}\left[
    \sum_{|\alpha|=1}^{\infty}
    \frac{1}{\alpha_k!}
   \left(D^{\alpha_k} \mathcal{L}(\v{h}_{k}(\v{x}), \v{y}) \right)  \v{\epsilon}^{\alpha_k}_k\right] + \mathcal{C}((\v{x}, \v{y});\v{\epsilon})\right]
\end{align}

Here $\mathcal{C}(\v{\epsilon}, \v{x}, \v{y})$ represents cross-interactions between the noise at each layer $k$ $\v{\epsilon}_k$ and the noise injections at preceding layers with index less than $k$. 
We can further simplify the added term to the loss, 
 \begin{align}
    &\expect_{\v{\epsilon} } \left[
    \sum_{k=0}^{L-1}\left[
    \sum_{|\alpha|=1}^{\infty}
    \frac{1}{\alpha_k!}
   \left(D^{\alpha_k} \mathcal{L}(\v{h}_{k}(\v{x}), \v{y}) \right)  \v{\epsilon}^{\alpha_k}_k\right] + \mathcal{C}((\v{x}, \v{y});\v{\epsilon})\right]  \nonumber \\
   &=
    \sum_{k=0}^{L-1}\left[
    \sum_{|\alpha_k|=1}^{\infty}
    \frac{1}{2\alpha_k!}
   \left(D^{2\alpha_k} \mathcal{L}(\v{h}_{k}(\v{x}), \v{y}) \right)  \expect_{\v{\epsilon} } \left[\v{\epsilon}^{2\alpha_k}_k\right]\right] +  \expect_{\v{\epsilon} } \left[\mathcal{C}((\v{x}, \v{y});\v{\epsilon})\right] 
\end{align}

The second equality comes from the fact that odd-numbered moments of $\v{\epsilon}_k$, will be 0 and that $\mathcal{L}(\v{x}, \v{y}) = \mathcal{L}(\v{h}_{L}(\v{x}), \v{y})$. 
The final equality comes from the moments of a mean 0 Gaussian, where $j$ takes the values of the multi-index. 
Note that $\left[\v{\epsilon}^{2\alpha_k}_k\right]$ are the even numbered moments of a zero mean Gaussian, \[\expect\left[\v{\epsilon}^{2\alpha_k}_k\right] =  [\sigma^{2\alpha_{k,1}}_k(2\alpha_{k,1} -1)!,\dots ,\sigma^{2\alpha_{k,d_k}}_k(2\alpha_{k,d_k} -1)!]^\intercal
\]

Though these equalities can already offer insight into the regularising mechanisms of GNIs, they are not easy to work with and will often be computationally intractable. 
We focus on the first set of terms here where each $|\alpha_k|=1$, which we denote $R(\v{x}, \v{\theta})$
\begin{align}
    R(\v{x}, \theta) = &\sum_{k=0}^{L-1}\left[
    \sum_{|\alpha_k|=1}
    \frac{1}{2\alpha_k!}
   \left(D^{2\alpha_k} \mathcal{L}(\v{h}_{k}(\v{x}), \v{y}) \right) \expect_{\v{\epsilon} } \left[\v{\epsilon}^{2\alpha_k}_k\right] \right] \nonumber \\ 
   &\approx \sum_{k=0}^{L-1}\left[
    \sum_{|\alpha_k|=1}
  \frac{\sigma_k^{2}}{2}
   \left(D^{2\alpha_k} \mathcal{L}(\v{h}_{L}(\v{x}), \v{y})\right)\v{J}^{2\alpha_k}_{k}(\v{x})  \right]
\end{align}

The last approximation corresponds to the Gauss-Newton approximation of second-order derivatives of composed functions where we've discarded the second set of terms of the form $D \mathcal{L}(\v{h}_{L})(\v{x})\left(D^2 \v{h}_{L}(\v{h}_{k}(\v{x}))\right)$.
We will include these terms in our remainder term $\mathcal{C}$. 
For compactness of notation, we denote each layer's Jacobian as $\v{J}_k \in \mathbb{R}^{d_L \times d_k}$.
Each entry of $\v{J}_{k}$ is a partial derivative  of $f^{\theta}_{k,i}$, the function from layer $k$ to the $i^{\mathrm{th}}$ network output, $i = 1 ... d_L$. 
\begin{align*}
    \v{J}_{k}(\v{x}) = \begin{bmatrix} 
    \frac{f^{\theta}_{k,1}}{\partial h_{k,1}} & \frac{f^{\theta}_{k,1}}{\partial h_{k,2}} & \dots \\
    \vdots & \ddots & \\
    \frac{f^{\theta}_{k,d_L}}{\partial h_{k,1}} &        & \frac{f^{\theta}_{k,d_L}}{\partial h_{k,d_k}} 
    \end{bmatrix} , 
\end{align*}
Again, for simplicity of notation $\v{J}^{\alpha_k}_k$ selects the column indexed by $|\alpha_k|=1$. 
Also note that the sum over $|\alpha_k|=1$ effectively indexes over the diagonal of the Hessian of the Loss with respect to the $L^{\mathrm{th}}$ layer activations. 
We denote this Hessian as $\v{H}_{L}(\v{x}, \v{y})\in \mathbb{R}^{d_L
\times d_L}$. 
\begin{align*}
    \v{H}_{L}(\v{x}, \v{y}) = \begin{bmatrix} 
    \frac{\partial^2 \mathcal{L}}{\partial h^2_{L,1}} & \frac{\partial^2 \mathcal{L}}{\partial h_{L,1}\partial h_{L,2}} & \dots \\
    \vdots & \ddots & \\
    \frac{\partial^2 \mathcal{L}}{\partial h_{L,d_L}\partial h_{L,1}} &        & \frac{\partial^2 \mathcal{L}}{\partial h^2_{L,d_L}} 
    \end{bmatrix} 
\end{align*}

This gives us that 
\[
R(\v{x}; \v{\theta}) = \frac{1}{2}\sum_{k=0}^{L-1}\left[\sigma_k^2\mathrm{Tr}\left(\v{J}^\intercal_{k}(\v{x})
    \v{H}_{L}(\v{x}, \v{y})\v{J}_{k}(\v{x}) \right)\right] 
\]

For notational simplicity we include the terms that $R$ does not capture into the remainder $\expect_{\v{\epsilon} } \left[\mathcal{C}((\v{x}, \v{y});\v{\epsilon})\right]$.  We take expectations over the batch and have: 
 \begin{align}
    &\expect_{(\v{x},\v{y}) \sim \mathcal{B} } \left[\expect_{\v{\epsilon} } \left[\mathcal{L}(\v{h}_{L}(\v{x}) + \v{\mathcal{E}}_L, \v{y})\right] \right]
   =  \mathcal{L}(\mathcal{B}; \v{\theta}) + R(\mathcal{B}; \v{\theta}) +  \expect_{\v{\epsilon} } \left[\mathcal{C}(\mathcal{B}; \v{\epsilon})\right] \\
    R(\mathcal{B}; \v{\theta}) &= \mathbb{E}_{(\v{x},\v{y}) \sim \mathcal{B}} \left[\frac{1}{2}\sum_{k=0}^{L-1}\left[\sigma_k^2\mathrm{Tr}\left(\v{J}^\intercal_{k}(\v{x})
    \v{H}_{L}(\v{x}, \v{y})\v{J}_{k}(\v{x})\right)\right] \right]   \\
    &\expect_{\v{\epsilon} } \left[\mathcal{C}(\mathcal{B}; \v{\epsilon})\right]  = 
    \mathbb{E}_{(\v{x},\v{y}) \sim \mathcal{B}} \left[
   \expect_{\v{\epsilon} } \left[\sum_{|\alpha|=1}^{\infty}
    \frac{1}{\alpha!}
    \left(D^{\alpha}\mathcal{L}(\v{h}_{L}(\v{x}), \v{y}) \right)
    \v{\mathcal{E}}_{L}^{\alpha}\right]\right] - R(\mathcal{B}; \v{\theta})
\end{align}

This concludes the proof.

\end{proof}

\subsection{Proof of Theorem 2}
\label{app:plancherel_th}

\begin{proof}[Proof of Theorem~\ref{th:plancherel-measure-space}]

Because $f \in W^{1,2}_{\mu}(\mathbb{R}^d)$ we know that by definition, for $|\alpha|=1$: \[\|D^{\alpha} f\|^2_{L^2_{\mu}(\mathbb{R}^d)} = \int_{\mathbb{R}^d} |D^{\alpha} f(\v{x}) \cdot D^{\alpha} f(\v{x}) \cdot \mu(\v{x})| d\v{x} < \infty\] where $d\v{x}$ is the Lebesgue measure.
By Minkowski's inequality we know that: 
\[\int_{\mathbb{R}^d} |D^{\alpha} f(\v{x}) \cdot D^{\alpha} f(\v{x}) \cdot \mu(\v{x}) \cdot \mu(\v{x})| d\v{x} < \int_{\mathbb{R}^d} |\mu(\v{x})| d\v{x} \int_{\mathbb{R}^d} |D^{\alpha} f(\v{x}) \cdot D^{\alpha} f(\v{x}) \cdot \mu(\v{x})| d\v{x}\]

By definition $\mu$, a probability measure, is $L^1$ integrable.
As such: 
\[\int_{\mathbb{R}^d} |D^{\alpha} f(\v{x}) \cdot D^{\alpha} f(\v{x}) \cdot \mu(\v{x}) \cdot \mu(\v{x})| d\v{x} < \infty\]
Let $m(\v{x}) = D^{\alpha}f(\v{x}) \cdot \mu(\v{x})$, by the equation above, $m(\v{x}) \in L^2(\mathbb{R}^d)$.
As both $D^{\alpha} f(\v{x})$ (by assumption) and $m(\v{x})$ are $L^2$ integrable in $\mathbb{R}^d$, we can apply Fubini's Theorem and Plancherel's Theorem straighforwardly such that: 

\begin{equation}
    \sum_{|\alpha|=1}\|D^{\alpha} f\|^2_{L^2_{\mu}(\mathbb{R}^d)} = 
     \int_{\mathbb{R}^d} \sum^d_{j=1}\Bigr|i \v{\omega}_j\mathcal{F}(\v{\omega}) \cdot \overline{\mathcal{M}(\v{\omega}, j)}\Bigr| d\v{\v{\omega}} \nonumber  
\end{equation}

where $\mathcal{F}$ is the Fourier transform of $f$, $i^2=-1$, and $\v{\omega}_j \mathcal{F}(\v{\omega})$ is simply the Fourier transform of the derivative indexed by $\alpha$. 
$\mathcal{M}(\v{\omega}, j)$ is given by
\begin{equation}
    \mathcal{M}(\v{\omega}, j) = i \left(\v{\omega}_j\mathcal{F}(\v{\omega})\right)*\mathcal{P}(\v{\omega})
\end{equation}

where $\mathcal{P}$ is the Fourier transform of the probability measure $\mu$, $ \v{\omega}_j \mathcal{F}(\v{\omega})$ is as before, and * denotes the convolution operator. 
Substituting $\mathcal{G}(\v{\omega},j) = \v{\omega}_j \mathcal{F}(\v{\omega})$
we obtain: 
\begin{align*}
     \sum_{|\alpha|=1}\|D^{\alpha} f\|^2_{L^2_{\mu}(\mathbb{R}^d)} &=  \int_{\mathbb{R}^d} \sum^d_{j=1} \Bigr|(i\overline{i}) \mathcal{G}(\v{\omega},j)\overline{ \mathcal{G}(\v{\omega},j) * \mathcal{P}(\v{\omega})}\Bigr|  d\v{\omega} \nonumber \\
    &= \int_{\mathbb{R}^d} \sum^d_{j=1} \Bigr|\mathcal{G}(\v{\omega},j)\overline{ \mathcal{G}(\v{\omega},j) * \mathcal{P}(\v{\omega})} \Bigr|  d\v{\omega} 
\end{align*}

This concludes the proof.

\end{proof}

\subsection{Regularisation in Regression Models and Autoencoders}
\label{app:exp_reg_regression}

In the case of regression the most commonly used loss is the mean-square error.
\begin{equation}
\mathcal{L}(\v{x}, \v{y}) = \frac{1}{2}(\v{y} - \v{h}_{L}(\v{x}))^2 \nonumber
\end{equation}

In this case, $\v{H}_{L, n}$ is $\v{I}$. 
As such:
\begin{align}
R(\mathcal{B}; \v{\theta}) &= \frac{1}{2}\mathbb{E}_{\v{x} \sim \mathcal{B}} \left[\sum_{k=0}^L \sigma_k^2(\mathrm{Tr}(\v{J}_{k}(\v{x})^\intercal\v{J}_{k}(\v{x})))]\right] =\frac{1}{2} \mathbb{E}_{\v{x} \sim \mathcal{B}} \left[ \sum_{k=0}^{L-1} \sigma^2_k (\|\v{J}_{k}(\v{x})\|^2_F) \right] \nonumber
\end{align}

This added term corresponds to the trace of the covariance matrix of the outputs $\v{h}_L$ given an input $\v{h}_{k}$. 
As such we are penalising the sum of output variances of the approximator; we are penalising the sensitivity of outputs to perturbations in layer $k$ \citep{Webb1994, Bishop1995}.

For $\mathrm{ReLU}$-like activations ($\mathrm{ELU}$, $\mathrm{Softplus}$ ...) , because our functions are at \textit{most} linear, we can bound our regularisers using the Jacobian of an equivalent linear network: 
\begin{equation}
\sum_{k=0}^L \sigma^2_k(\|\v{J}_{k}(\v{x})\|^2) < \sum_{k=0}^L \sigma^2_k(\|\v{J}^{\mathrm{linear}}_{k}(\v{x})\|^2) = \sum_{k=0}^L\sigma^2_k(\|\v{W}_L \dots \v{W}_k\|^2)
\nonumber
\end{equation}

Where $\v{J}^{\mathrm{linear}}_{k}(\v{x})$ is the gradient evaluated with no non-linearities in our network. 
This upper bound is reminiscent of $rank-k$ ridge regression, but here we penalise each sub-network in our network \citep{Kunin1999}. 
Also note that the regression setting is directly translatable to Auto-Encoders, where the labels are the input data.

\subsection{Regularisation in Classifiers}
\label{app:exp_reg_classification}

In the case of classification, we consider the cross-entropy loss.
Recall that we consider our network outputs $\v{h}_L$ to be the pre-$\mathrm{softmax}$ of logits of the final layer $\v{L}$.
We denote $\v{p}(\v{x})=\mathrm{softmax}(\v{h}_L(\v{x}))$. 
The loss is thus: 
\begin{equation}
\mathcal{L}(\v{x}, \v{y}) = - \sum_{c=0}^M \v{y}_{n,c} \log (\mathrm{softmax}(\v{h}_{L}(\v{x}))_c)
\label{eqapp:ce}
\end{equation}

where $c$ indexes over the $M$ possible classes of the classification problem. 
The hessian $\v{H}_L$ in this case is easy to compute and has the form: 

\begin{equation}
\v{H}_{L}(\v{x})_{i,j} = 
\begin{cases}
\v{p}(\v{x})_i(1 - \v{p}(\v{x})_j) & i = j \\
-\v{p}(\v{x})_i\v{p}(\v{x})_j & i \neq j \\
\end{cases}
\end{equation}

As \citet{Wei2020}, \citet{Sagun2018}, and \cite{LeCun1998} show, this Hessian is PSD, meaning that $\mathrm{Tr}(\v{J}_{k}\v{H}_{L} \v{J}^\intercal_{k})$ will be positive, fulfilling the criteria for a valid regulariser. 
\begin{align}
&R(\mathcal{B}; \v{\theta}) = \mathbb{E}_{\v{x} \sim \mathcal{B}} \left[\frac{1}{2}\sum_{k=0}^L \sigma_k^2 \sum_{i,j}(  \v{H}_{L}(\v{x}) \circ \v{J}_{k}(\v{x})\v{\v{J}}^\intercal_{k}(\v{x}))_{i,j}\right] \nonumber\\
&= \mathbb{E}_{\v{x} \sim \mathcal{B}} \left[\frac{1}{2}\sum_{k=0}^L \sigma_k^2 \sum_{i,j}(\mathrm{diag}(\v{H}_{L}(\v{x}))^\intercal\v{J}^2_{k}(\v{x}) )_{i,j} + \frac{1}{2}\sum_{k=0}^L \sigma_k^2 \sum_{\forall i,j \ i \neq j}(\v{H}_{L}(\v{x}) \circ \v{J}_{k}(\v{x})\v{\v{J}}^\intercal_{k}(\v{x}))_{i,j}\right] \nonumber
\label{eqapp:ce_reg}
\end{align}

$\mathrm{diag}( \v{H}_{L}(\v{x}))^\intercal$ is the row vector of the diagonal of $\v{H}_{L}(\v{x})$.
The first  equality is due to the fact that $\v{H}_L$ is symmetric and is due to the commutative properties of the trace operator.
The final equality is simply the decomposition of the sum of the matrix product into diagonal and off-diagonal elements. 
For shallow networks, the off-diagonal elements of $\v{J}_k\v{J}^\intercal_k$ are likely to be small and it can be approximated by $\v{J}^2_k$ \citep{Poole2016, Hauser2017, Farquhar2020, Aleksziev}.
See Figure \ref{fig:covariance_examples} for a demonstration that the off-diagonal elements of $\v{J}_k^\intercal\v{J}_k$, are negligible for smaller networks.
Ignoring these off-diagonal terms, we obtain an added positive term: 
\begin{equation}
    R(\mathcal{B}; \v{\theta}) \approx \mathbb{E}_{\v{x} \sim \mathcal{B}} \left[\frac{1}{2}\sum_{k=0}^L \sigma_k^2 \sum_{i,j}(\mathrm{diag}(\v{H}_{L}(\v{x}))^\intercal\v{J}^2_{k}(\v{x}) )_{i,j} \right]
\label{eqapp:ce_reg_term}
\end{equation}

For $\mathrm{ReLU}$-like activations ($\mathrm{ELU}$, $\mathrm{Softplus}$ ...), because our functions are at \textit{most} linear, we can bound our regularisers using the Jacobian of an equivalent linear network: 
\begin{equation}
\sum_{k=0}^L \sigma_k^2 \sum_{i,j}(\mathrm{diag}( \v{H}_{L}(\v{x}))^\intercal\v{J}_{k}(\v{x})^2 )_{i,j} <   \sum_{k=0}^L\sigma^2_k\sum_{i,j}(\mathrm{diag}( \v{H}_{L}(\v{x}))^\intercal(\v{W}_L \dots \v{W}_k)^2 ))_{i,j}
\label{eqapp:ce_reg_bound}
\end{equation}

\begin{figure}[h]
    \centering
    \subfigure[][SVHN MLP,  $k$=0]{\includegraphics[width=0.25\textwidth]{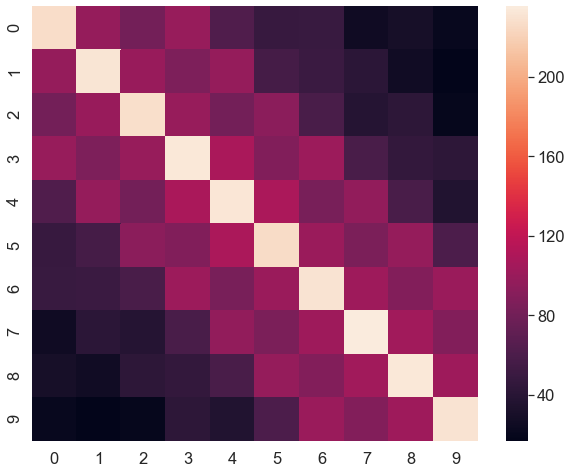}}
    \subfigure[][SVHN MLP,  $k$=1]{\includegraphics[width=0.25\textwidth]{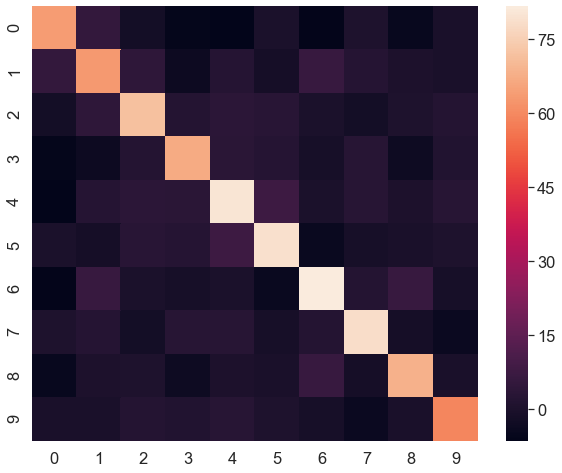}} 
    \subfigure[][SVHN MLP,  $k$=2]{\includegraphics[width=0.25\textwidth]{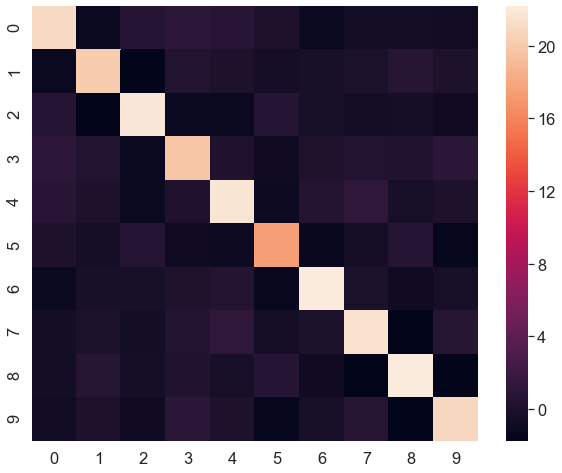}} \\
    \subfigure[][CIFAR10 CONV,  $k$=0]{\includegraphics[width=0.25\textwidth]{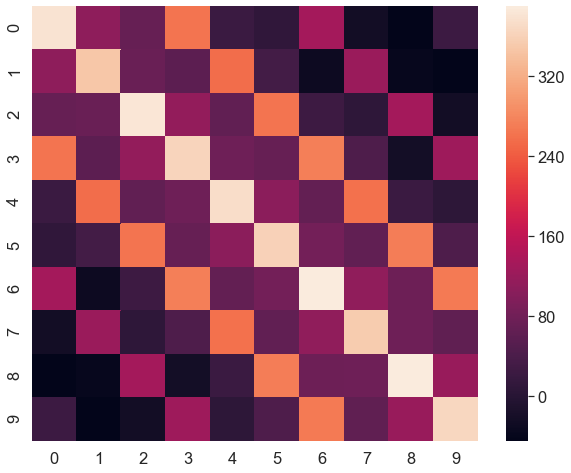}}
    \subfigure[][CIFAR10 CONV,  $k$=1]{\includegraphics[width=0.25\textwidth]{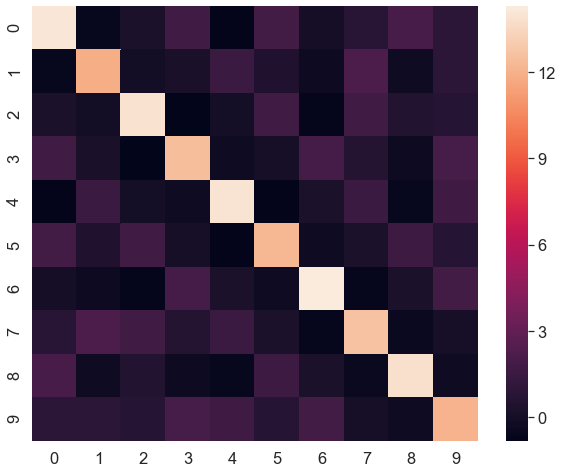}} 
    \subfigure[][CIFAR10 CONV,  $k$=2]{\includegraphics[width=0.25\textwidth]{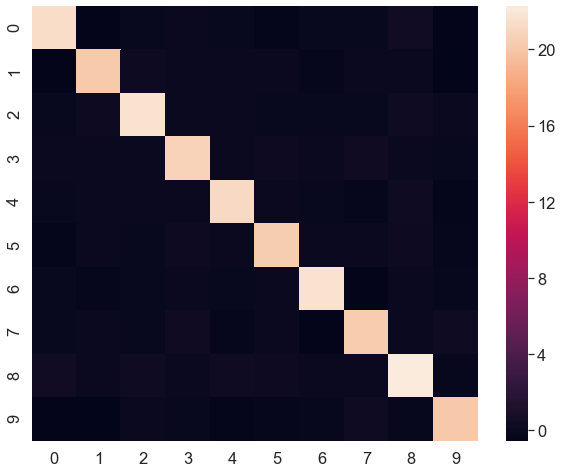}}
    
    \caption{Samples of heatmaps of 10 by 10 matrices $\v{J}_k^\intercal\v{J}_k$ ($k$ indexing over layers) for 2-layer MLPs and convolutional networks (CONV) trained to convergence (with no regularisation) on the SVHN and CIFAR10 classification datasets, each with 10 classes. 
    We can clearly see that the diagonal elements of these matrices dominate in all examples, though less so for the data layer. 
    }
    \label{fig:covariance_examples}
\end{figure}

\clearpage

\section{Tikhonov Regularisation}
\label{app:tikhonov}

Note that because we are penalising the terms of the Sobolev norm associated with the first order derivatives, this constitutes a form of Tikhonov regularisation. 
Tikhonov regularisation involves adding some regulariser to the loss function, which encodes a notion of `smoothness' of a function $f$ \citep{Bishop1995}. 
As such, by design, regularisers of this form have been shown to have beneficial regularisation properties when used in the training objective of neural networks by smoothing the loss landscape \citep{Girosi1990, Burger2003}. 
If we have a loss of the form $\mathcal{L}(\mathcal{B}; \v{\theta})$, the Tikhonov regularised loss becomes:
\begin{equation}
    \mathcal{L}(\mathcal{B}; \v{\theta}) + \lambda\|f^{\theta}\|^2_{\mathcal{H}}
\end{equation}
where $f^{\theta}$ is the function with parameters $\theta$ which we are learning and $\|\cdot\|_{\mathcal{H}}$ is the norm or semi-norm in the Hilbert space $\mathcal{H}$ and $\lambda$ is a (multidimensional) penalty which penalises elements of $\|f^{\theta}\|^2_{\mathcal{H}}$ unequally, or is data-dependent \citep{Tikhonov1977, Bishop1995}. 
In our case $\mathcal{H}$ is the Hilbert-Sobolev space $W^{1,2}_{\mu}(\mathbb{R}^d)$ with norm dictated by Equation \eqref{eq:weighted_sobolev_norm}. 
$R(\cdot)$ penalises the function's semi-norm in this space.  

\section{Measuring Calibration}
\label{app:calibration}

A neural network classifier gives a prediction $\hat{y}(\v{x})$ with confidence $\hat{p}(\v{x})$ (the probability attributed to that prediction) for a datapoint $\v{x}$.
Perfect calibration consists of being as likely to be correct as you are confident:
\begin{equation}
        p(\hat{y}=y|\hat{p}=r)=r, \quad \forall r\in[0,1]
\end{equation}
To see how closely a model approaches perfect calibration, we plot reliability diagrams \citep{Guo2017, Niculescu-Mizil2005}, which show the accuracy of a model as a function of its confidence over $M$ bins $B_m$.

\begin{align}
\mathrm{acc}(B_m) &=\frac{1}{|B_m|}\sum_{i\in B_m}\mathbf{1}(\hat{y}_i = y_i) \\
\mathrm{conf}(B_m) &=\frac{1}{|B_m|}\sum_{i\in B_m}\hat{p}_i
\end{align}

We also calculate the Expected Calibration Error (ECE) \cite{Naeini2015}, the mean difference between the confidence and accuracy over bins:
\begin{equation}
        \mathrm{ECE} = \sum_{m=1}^M \frac{|B_m|}{N}|\mathrm{acc}(B_m) - \mathrm{conf}(B_m)|
        \label{eqapp:ece}
\end{equation}
However, note that ECE only measures calibration, not refinement.
For example, if we have a balanced test set one can trivially obtain ECE $\approx 0$ by sampling predictions from a uniform distribution over classes while having very low accuracy.

\section{Classification Margins} 
\label{app:classification_margins}
Typically, models with larger classification margins are less sensitive to input perturbations \citep{Sokolic2017, Jakubovitz2018, Cohen2019, Liu2018, Li2018}. 
Such margins are the distance in data-space between a point $\v{x}$ and a classifier's decision boundary.
Larger margins mean that a classifier associates a larger region centered on a point $\v{x}$ to the same class. 
Intuitively this means that noise added to $\v{x}$ is still likely to fall within this region, leaving the classifier prediction unchanged.
\citet{Sokolic2017} and \citet{Jakubovitz2018} define a classification margin $M$ that is the radius of the largest metric ball centered on a point $\v{x}$ to which a classifier assigns $\v{y}$, the true label.

\begin{prop}[\citet{Jakubovitz2018}]
Consider a classifier that outputs a correct prediction for the true class $A$ associated with a point $\v{x}$. 
Then the first order  approximation  for  the l2-norm of the classification margin $M$, which is the  minimal  perturbation  necessary to fool a classifier, is lower bounded by:
\begin{align}
M(\v{x}) \geq \frac{(\v{h}^A_{L}(\v{x}) - \v{h}^B_{L}(\v{x}) )}{\sqrt{2} \|\v{J}_0(\v{x}))|_F}. 
\label{eq:radius_robust_frob}
\end{align}
We have $\v{h}^A_{L}(\v{x}) \geq \v{h}^B_{L}(\v{x})$, where
$\v{h}^A_{L}(\v{x})$ is the $L^{th}$ layer activation (pre-softmax) associated with the true class $A$, and $\v{h}^B_{L}(\v{x})$ is the second largest $L^{th}$ layer activation.
\end{prop}

Networks that have lower-frequency spectrums  and consequently have smaller norms of Jacobians (as established in Section \ref{sec:Fourier}
), will have larger classification margins and will be less sensitive to perturbations.
This explains the empirical observations of  \citet{Rahaman2019} which showed that functions biased towards lower frequencies are more robust to input perturbations. 

\begin{figure}[t!]

    \centering
    \subfigure[][$\v{J}_0$ CIFAR  ]{\includegraphics[width=0.3\textwidth]{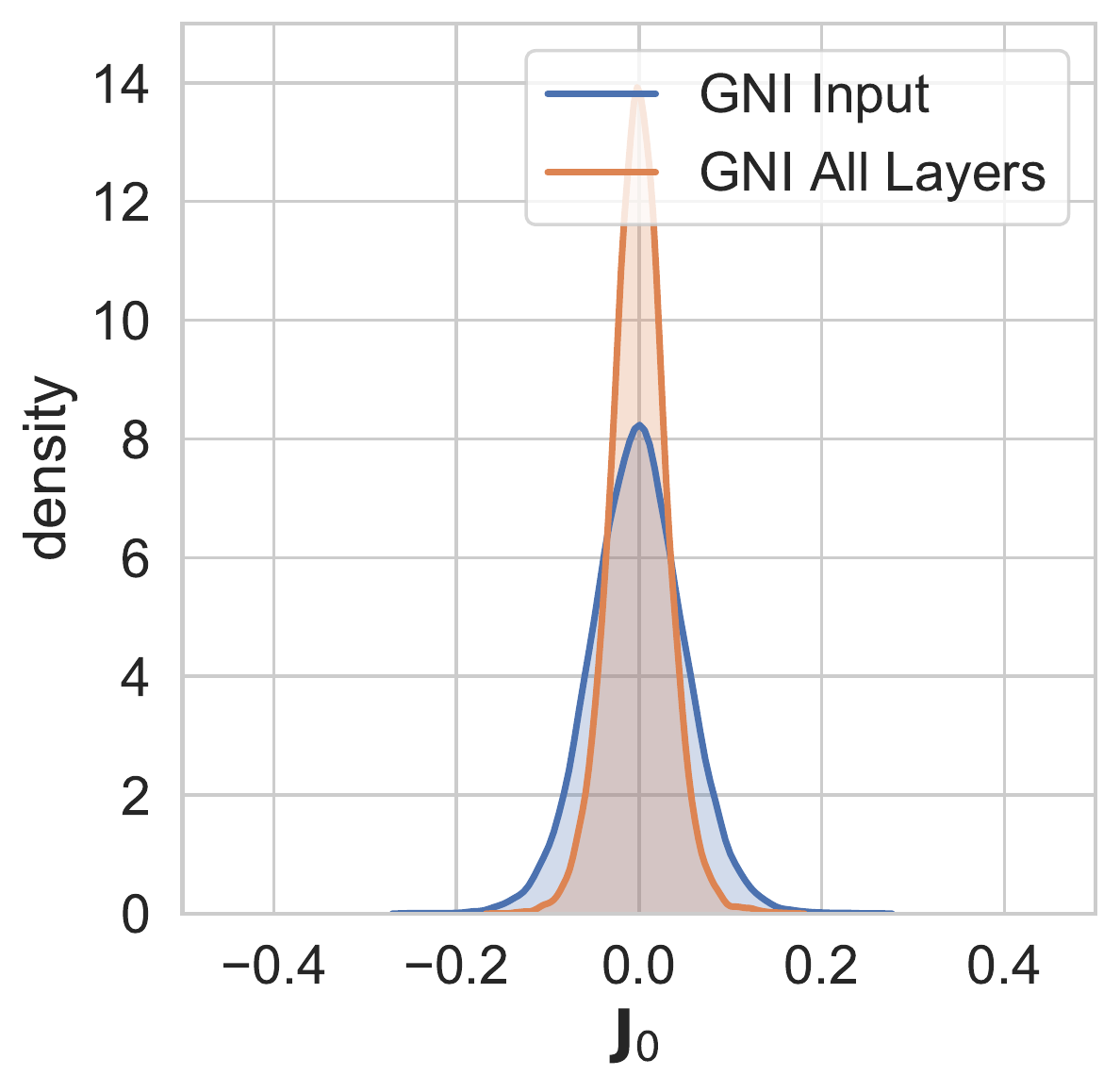}} 
    \subfigure[][$\v{J}_0$ SVHN  ]{\includegraphics[width=0.3\textwidth]{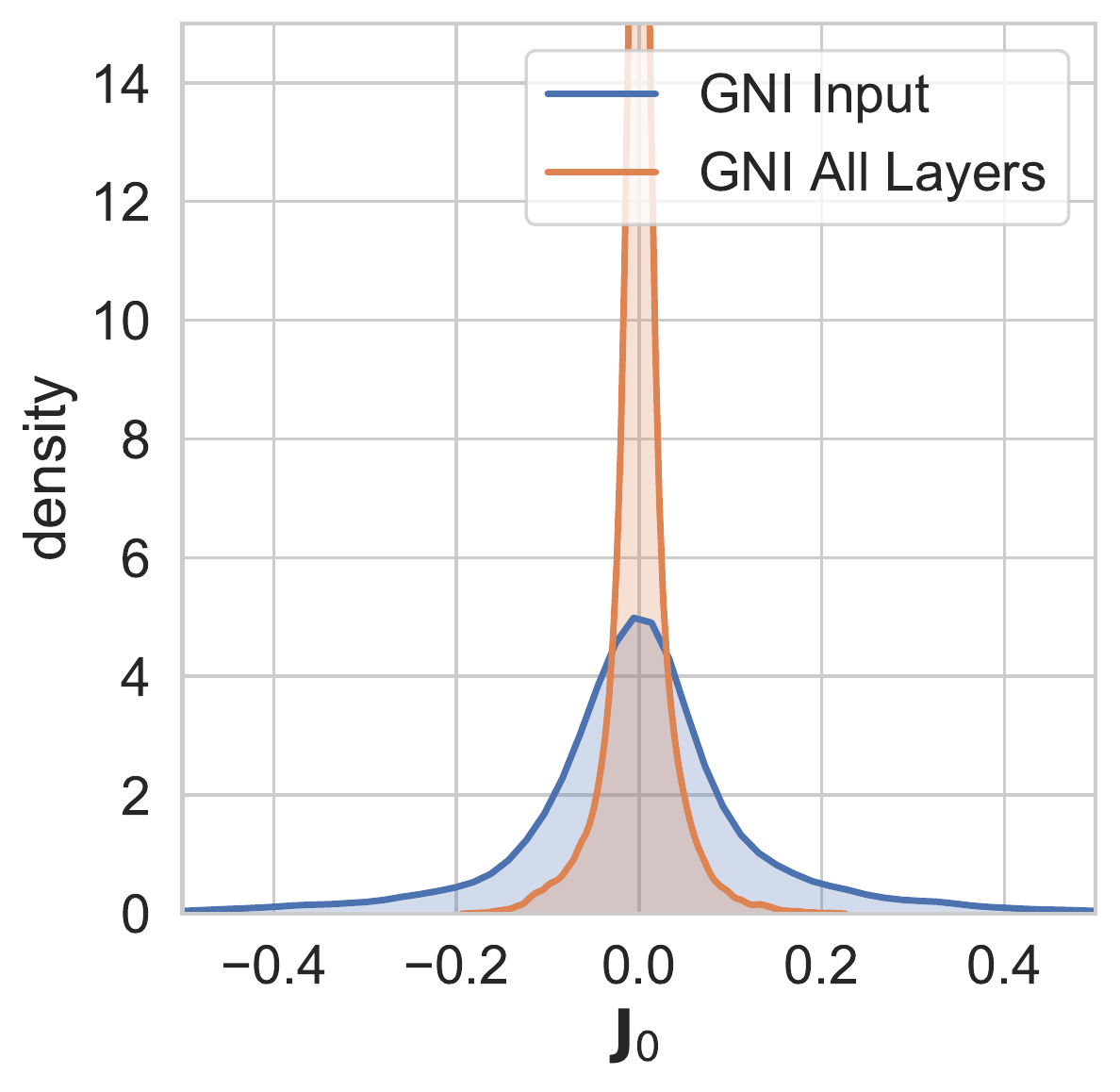}} 

    \caption{
    Here we show distribution plots of $\v{J}_0$ for 2-layer MLPs trained on CIFAR10 (a) and SVHN (b) for models trained with no noise (Baseline), models trained with noise on their inputs (GNI Input), models trained with noise on all their layers (GNI All Layers). Noising all layers induces a larger penalisation on the norm of $\v{J}_0$, seen clearly here by the shrinkage to 0 of $\v{J}_0$ for models trained in this manner. 
    }
    \label{fig:J0}
\end{figure}

What does this entail for GNIs applied to each layer of a network ?
We can view the penalisation of the norms of the Jacobians, induced by GNIs for each layer $k$, as an unpweighted penalisation of $\|\v{J}_0(\v{x})\|_F$. 
By the chain rule $\v{J}_0$ can be expressed in terms of any of the other network Jacobians $\v{J}_0(\v{x}) = \v{J}_k(\v{x})\frac{\partial \v{h}_k}{\v{x}} \forall k \in [0 \dots L]$. 
We can write $\|\v{J}_0(\v{x})\|_F = \|\v{J}_k(\v{x})\frac{\partial \v{h}_k}{\v{x}}\|_F \leq \|\v{J}_k(\v{x})\|_F\|\frac{\partial \v{h}_k}{\v{x}}\|_F$. 
Minimising $\|\v{J}_0(\v{x})\|_F$ is equivalent to minimising $\|\v{J}_k(\v{x})\|_F$ and $\|\frac{\partial \v{h}_k}{\v{x}}\|_F$, and upweighted penalisations of  $\|\v{J}_k(\v{x})\|_F$ should translate into a shrinkage of $\|\v{J}_0(\v{x})\|_F$. 
As such, noising each layer should induce a smaller $\|\v{J}_0(\v{x})\|_F$, and larger classification margins than solely noising data.
We support this empirically in Figure \ref{fig:J0}.

In Figure \ref{fig:adv_attack} we confirm that these larger classification translate into a lessened sensitvity to noise.

\section{Model Capacity} 
\label{app:capacity}
Intuitively one can view lower frequency functions as being `less complex', and less likely to overfit. 
This can be visualised in Figure \ref{fig:network_fourier}.
A measure of model complexity is given by `capacity' measures. 
If we have a model class $\mathcal{H}$, then the capacity assigns a non-negative number to each hypothesis in the model class $\mathcal{M} : \{\mathcal{H},\mathcal{D}_{train}\} \to \mathbb{R}^+$, where $\mathcal{D}_{train}$ is the training set and a lower capacity is an indicator of better model generalisation \citep{Neyshabur2017}. 
Generally, deeper and narrower networks induce large capacity models that are likely to overfit and generalise poorly \citep{Zhang2019}.
The network Jacobian's spectral norm and Frobenius norm  are good approximators of model capacity and are clearly linked to $R$  \citep{Guo2017, Neyshabur2017, Neyshabur}.

The Frobenius norm of the network Jacobian corresponds to a norm in Sobolev space which is a measure of a network's high-frequency components in the Fourier domain.
From this we offer the first theoretical results on why norms of the Jacobian are a good measure of model capacity: as low-frequency functions correspond to smoother functions that are less prone to overfitting, a smaller norm of the Jacobian is thus a measure of a smoother `less complex' model.

\clearpage
\newpage
\section{Additional Results }
\label{app:exp_reg_results}

\begin{figure}[h!]
    \centering
     \subfigure[][BHP MLP Loss]{\includegraphics[width=0.34\textwidth]{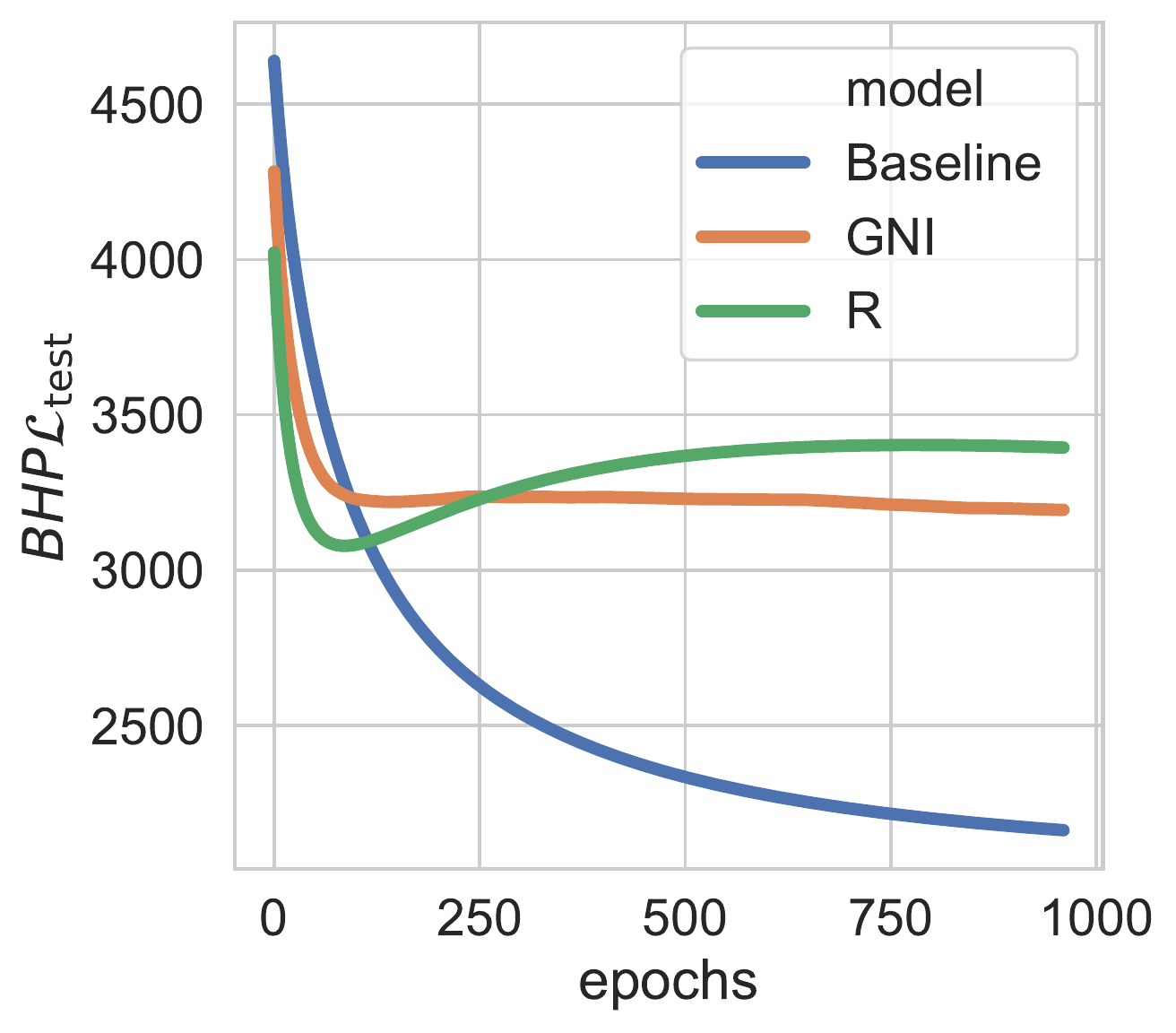}}
    \caption{
    In Figure (a) we show the test set loss for the regression dataset Boston House Prices (BHP) for 4-layer ELU MLPs trained with $R(\cdot)$ and GNIs for $\sigma^2=0.1$. 
    We compare to a non-noised baseline (Baseline).
    Exp Reg captures much of the effect of noise injections. 
    The test set loss is quasi-identical between Exp Reg and Noise runs which clearly differentiate themselves from Baseline runs. 
    }
    \label{fig:marg_approx_app_1}
\end{figure}

\begin{figure}[h!]
    \centering
    \subfigure[][SVHN MLP, $\sigma^2=0.1$]{\includegraphics[height=1.5in]{figures/hessian_approximations_svhn_add_new_runs.pdf}}
    \subfigure[][BHP MLP $\sigma^2=0.1$]{\includegraphics[height=1.5in]{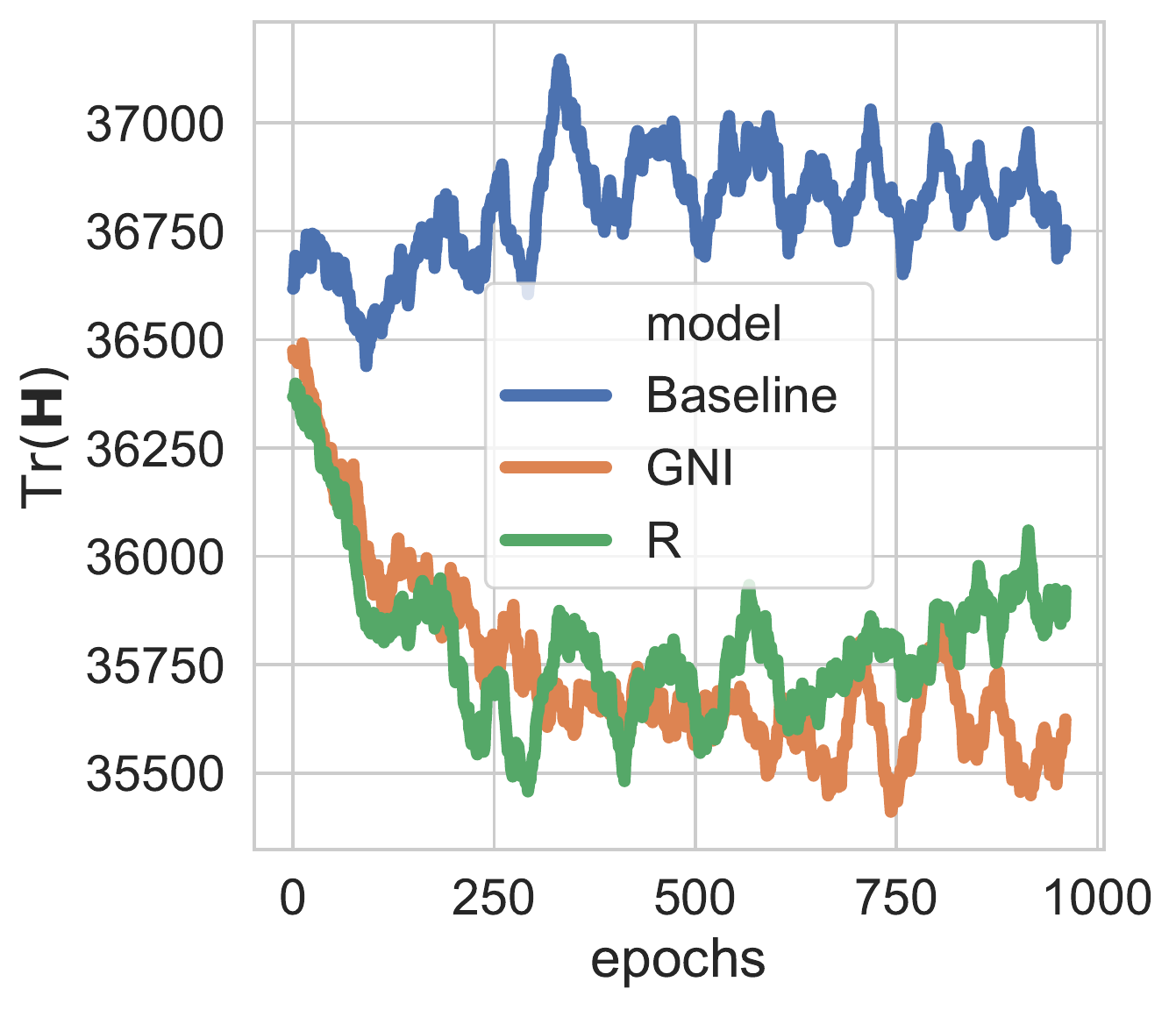}}
    
    \caption{Here we use small variance noise injections and show that the $R(\cdot)$ (Exp Reg) in equation \eqref{eq:mse_reg_term} and \eqref{eq:ce_reg_term}, induces the same trajectory through the loss landscape as GNIs (Noise).
    We show the trace of the Hessian of neural weights ($H_{i,j} = \frac{\partial \mathcal{L}}{\partial w_i\partial w_j}$) for a smaller 2-layer 32 unit MLP trained on the classification datasets CIFAR10 (a), and SVHN (b), and the regression dataset Boston House Prices (BHP) (c). 
    In all experiments we compare to a non-noised baseline (Baseline).
    $\mathrm{Tr}({\v{H}})$, which approximates the trajectory of the model weights through the loss landscape, is quasi identical for Exp Reg and Noise and is clearly distinct from Baseline, supporting the fact that the explicit regularisers we have derived are valid. 
    As expected the explicit regulariser and the noised models have smoother trajectories (lower trace) through the loss landscape, except for CIFAR10. 
    }
    \label{fig:hessian_app}
\end{figure}

\begin{figure}[h!]
    \centering
   \subfigure[][ELU non-linearities, $\sigma^2=0.1$] {\includegraphics[width=0.7\textwidth]{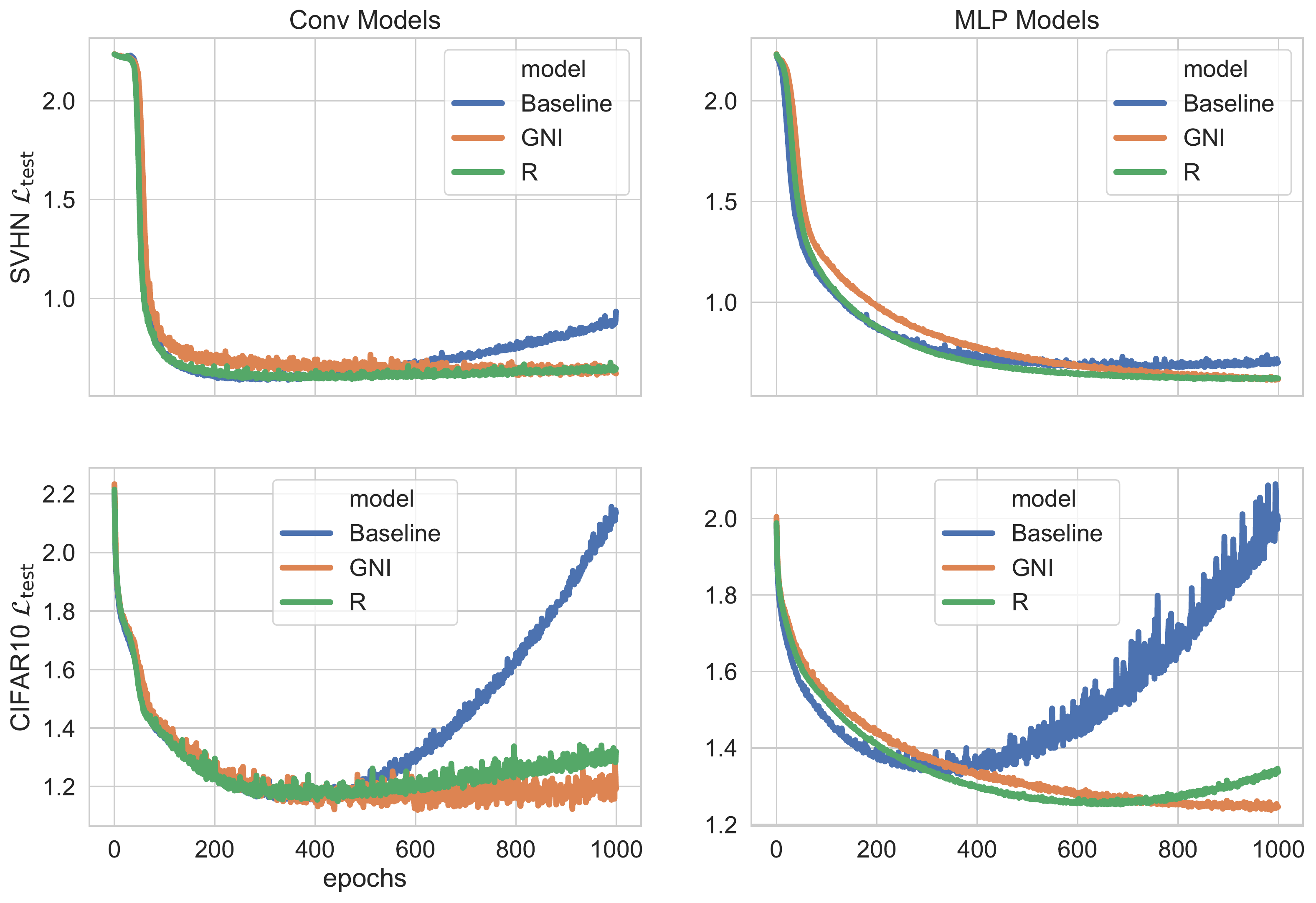}} \\
    \subfigure[][ReLU non-linearities, $\sigma^2=0.1$]{\includegraphics[width=0.7\textwidth]{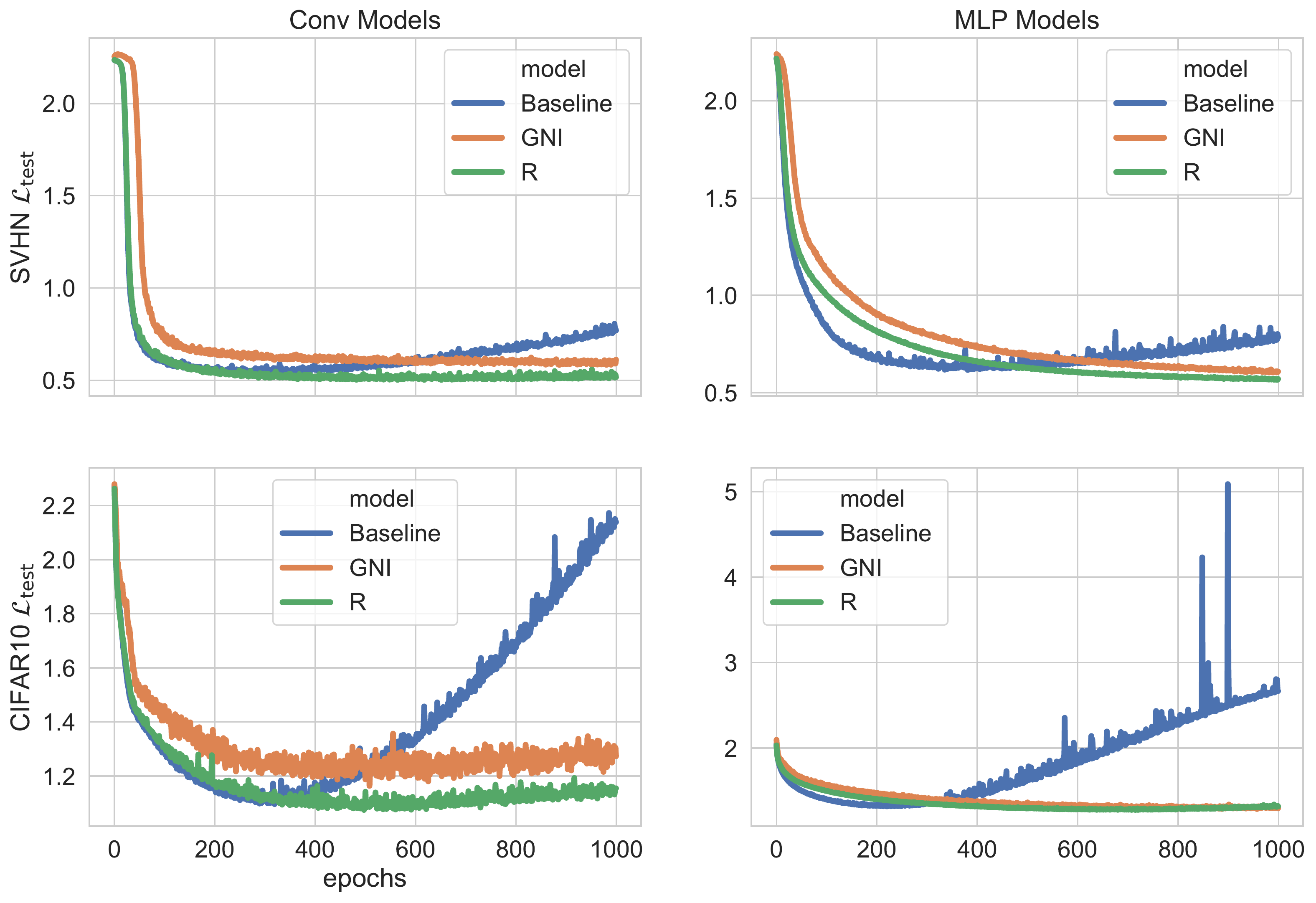}}
    \caption{Illustration of the loss induced by the $R(\cdot)$ for classification detailed in equation \eqref{eq:ce_reg_term} for convolutional and MLP architectures, and for ReLU and ELU non-linearities. 
    The loss trajectory is quasi-identical to models trained with GNIs and the trajectories are clearly distinct from baselines (Baseline), supporting the fact that the explicit regularisers we have derived are valid.}
    \label{fig:marg_approx_app}
\end{figure}

\newpage

 \begin{figure}[t!]

    \centering
    \subfigure[][CIFAR  ]{\includegraphics[width=0.3\textwidth]{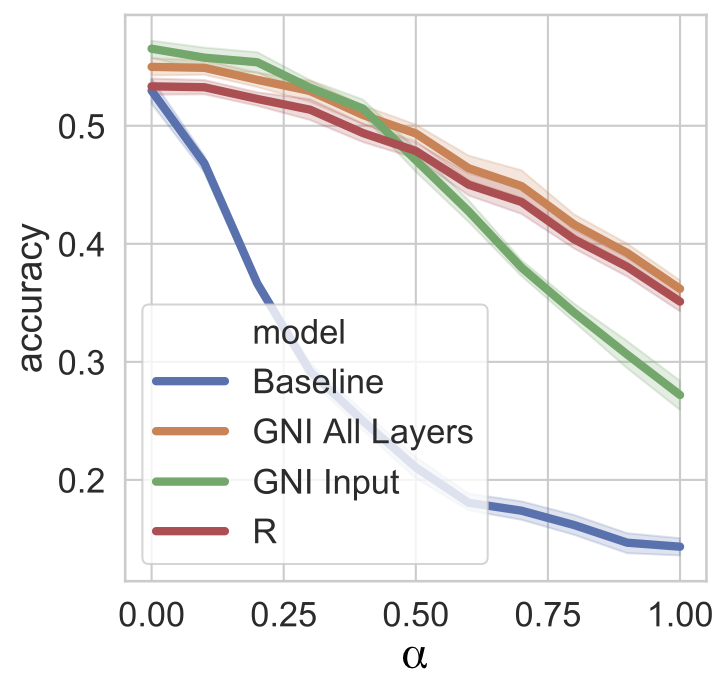}} 
    \hspace{1.0cm}
    \subfigure[][SVHN  ]{\includegraphics[width=0.3\textwidth]{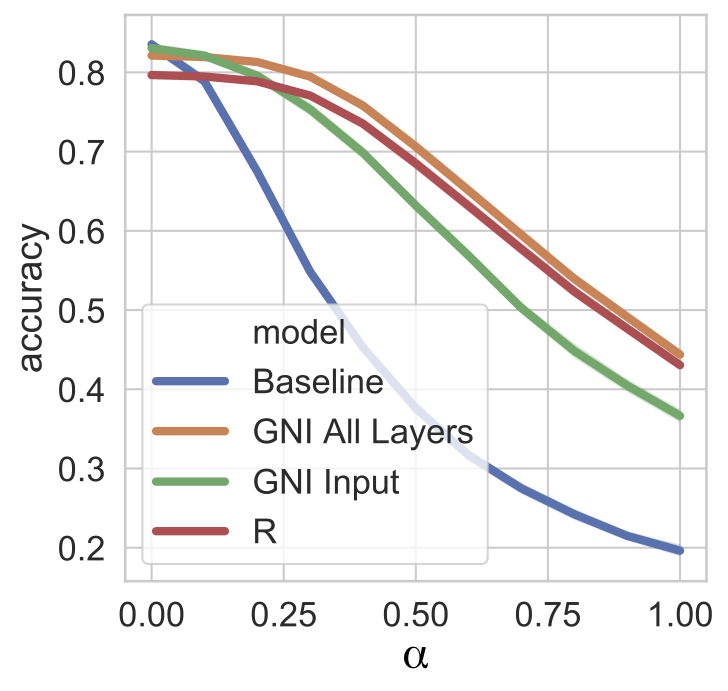}} 
    \caption{In (a) and (b) a model's sensitivity to noise by adding noise of variance $\alpha^2$ to data and measuring the resulting model accuracy given this corrupted test data. We show this for 2-layer MLPs trained on CIFAR10 (a) and SVHN (b) for models trained with no noise (Baseline), models trained with noise on their inputs (GNI Input), models trained with noise on all their layers (GNI All Layers), and models trained with the $R(\cdot)$ for classification.
    Noise added during training has variance $\sigma^2=0.1$ and confidence intervals are the standard deviation over batches of size 1024.
    Models trained with noise on all layers, and those trained with $R(\cdot)$, have the slowest decay of performance as $\alpha$ increases, confirming that such models have larger classification margins.
    }
    \label{fig:adv_attack}
\end{figure}

\begin{figure}[h!]
    \centering
    \subfigure[][CIFAR10 MLP, $\sigma=0.1$]{\includegraphics[height=1.2in]{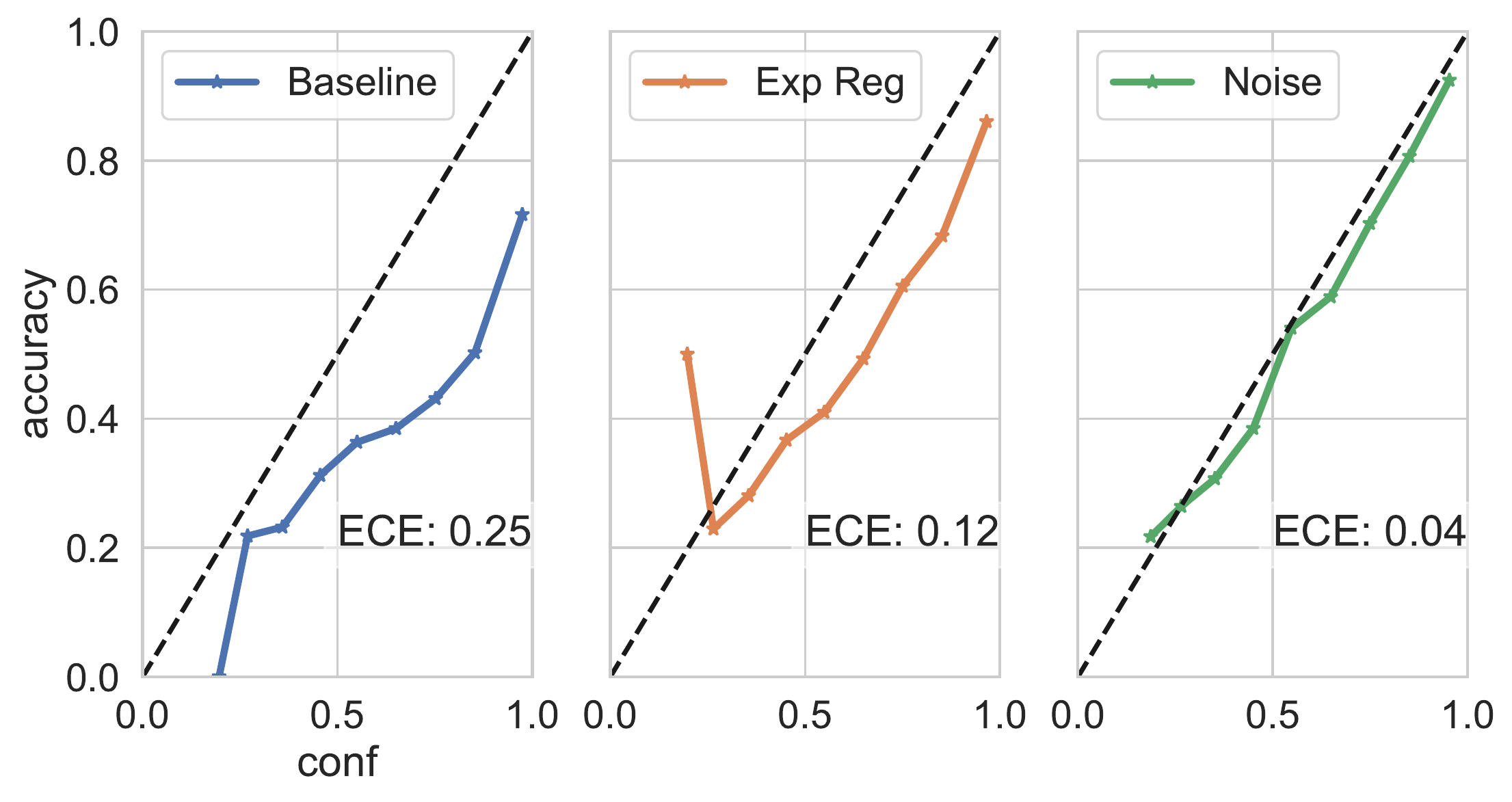}}
    \subfigure[][CIFAR10 MLP, $\sigma=0.1$]{\includegraphics[height=1.2in]{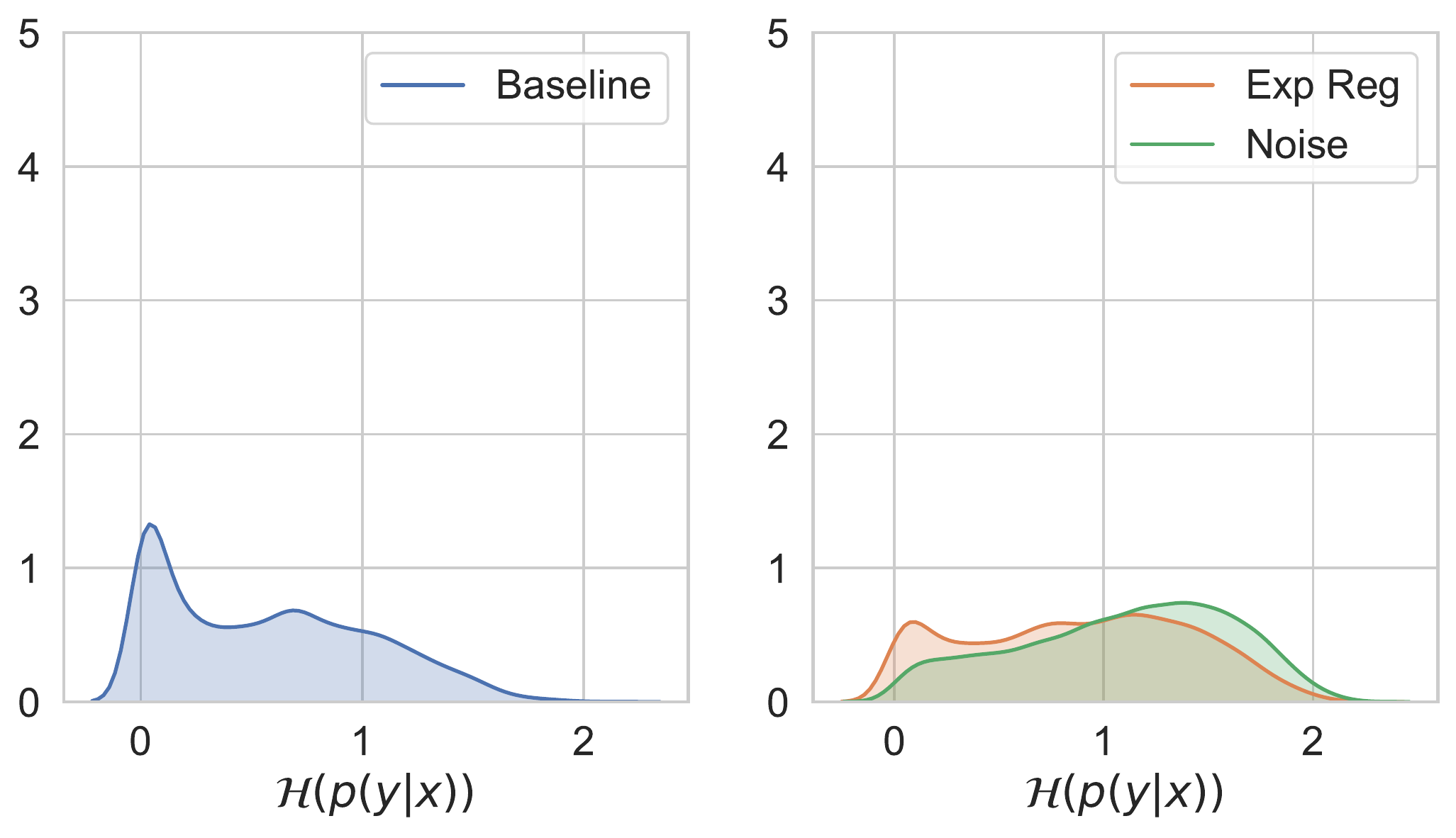}} \\
    \subfigure[][SVHN MLP, $\sigma=0.1$]{\includegraphics[height=1.2in]{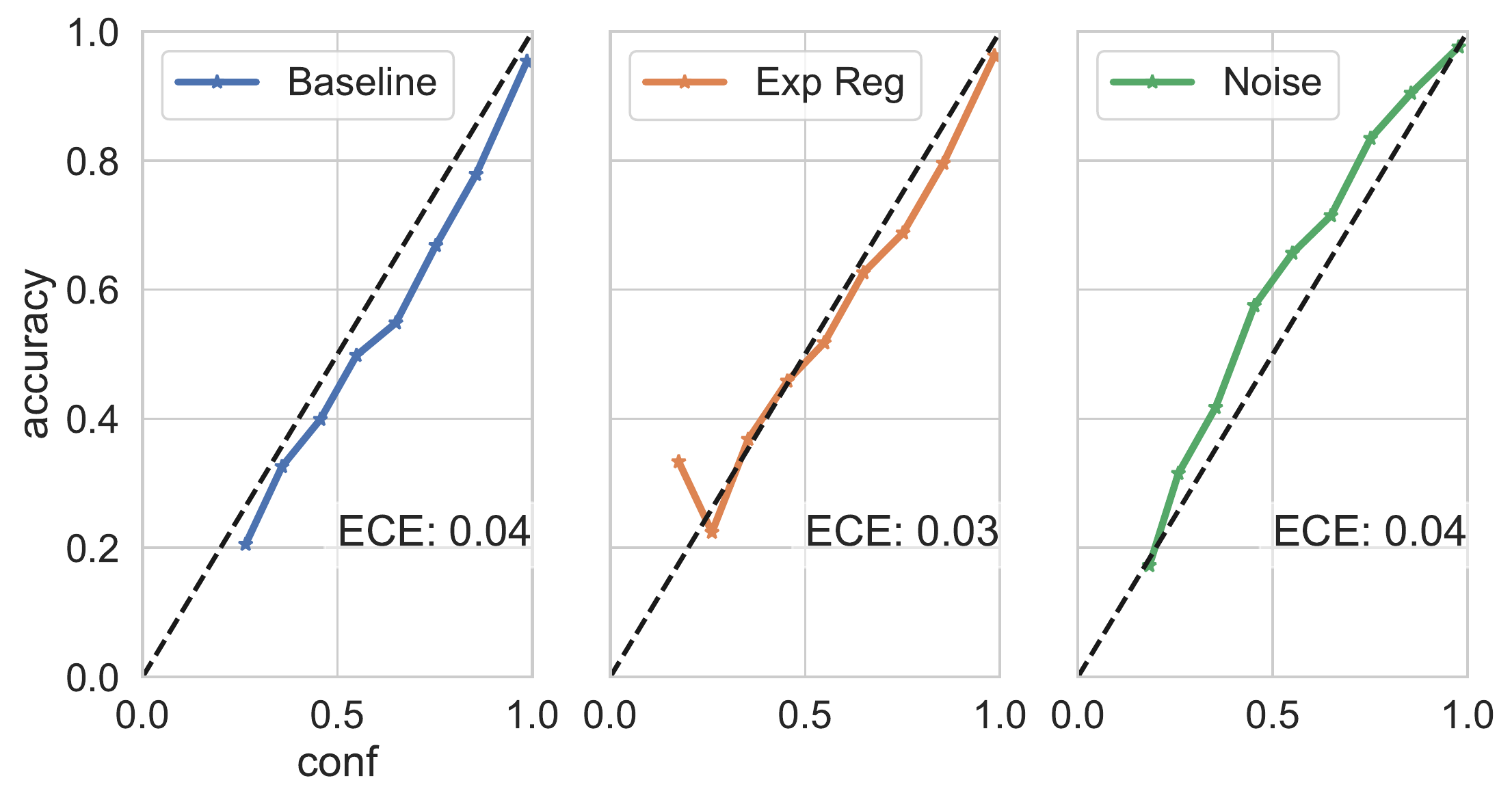}} 
    \subfigure[][SVHN MLP, $\sigma=0.1$]{\includegraphics[height=1.2in]{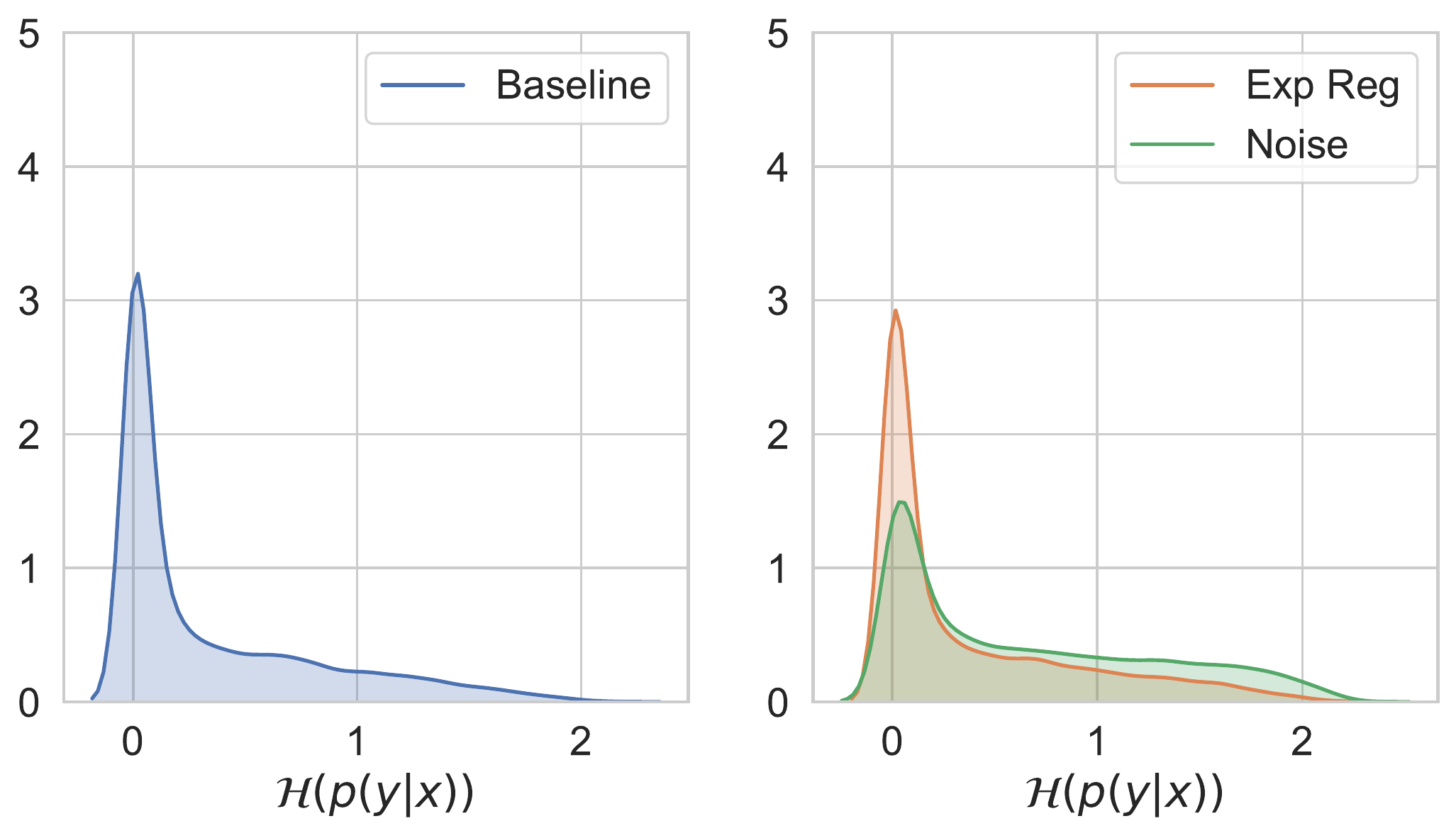}}
    \\
    \subfigure[][CIFAR10 CONV, $\sigma=0.1$]{\includegraphics[height=1.2in]{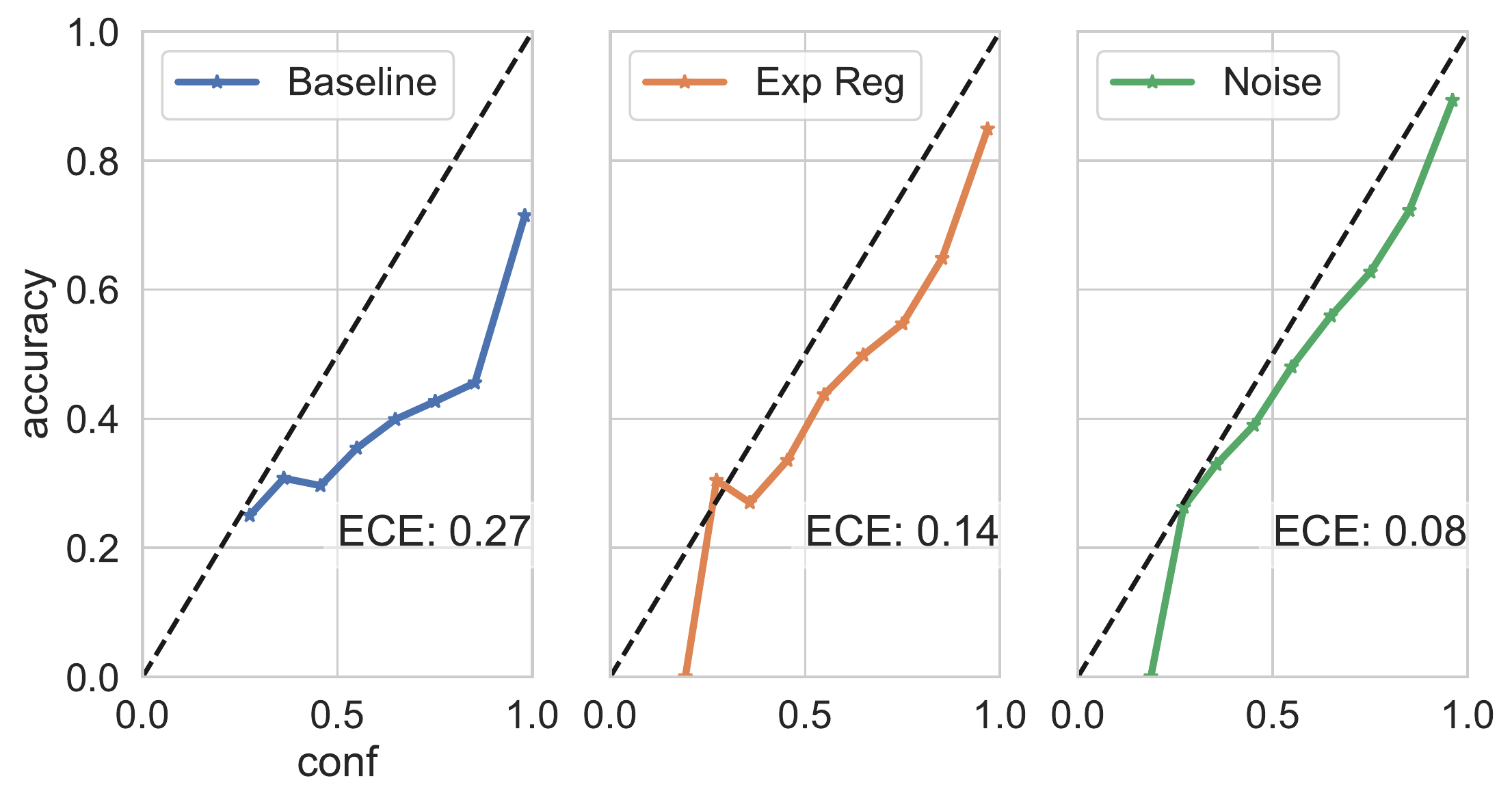}} 
     \subfigure[][CIFAR10 CONV, $\sigma=0.1$]{\includegraphics[height=1.2in]{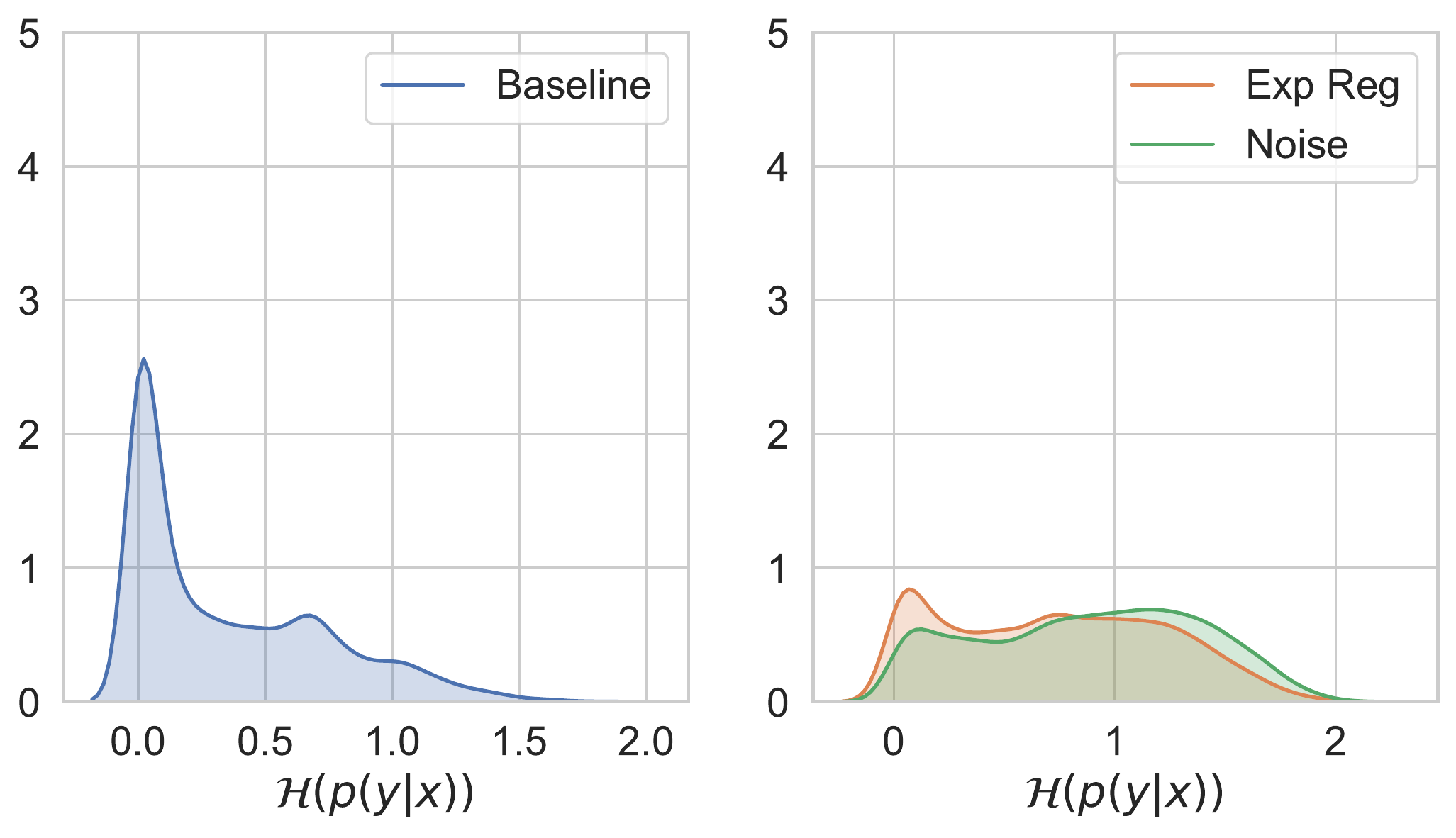}} \\
    \subfigure[][SVHN CONV, $\sigma=0.1$]{\includegraphics[height=1.2in]{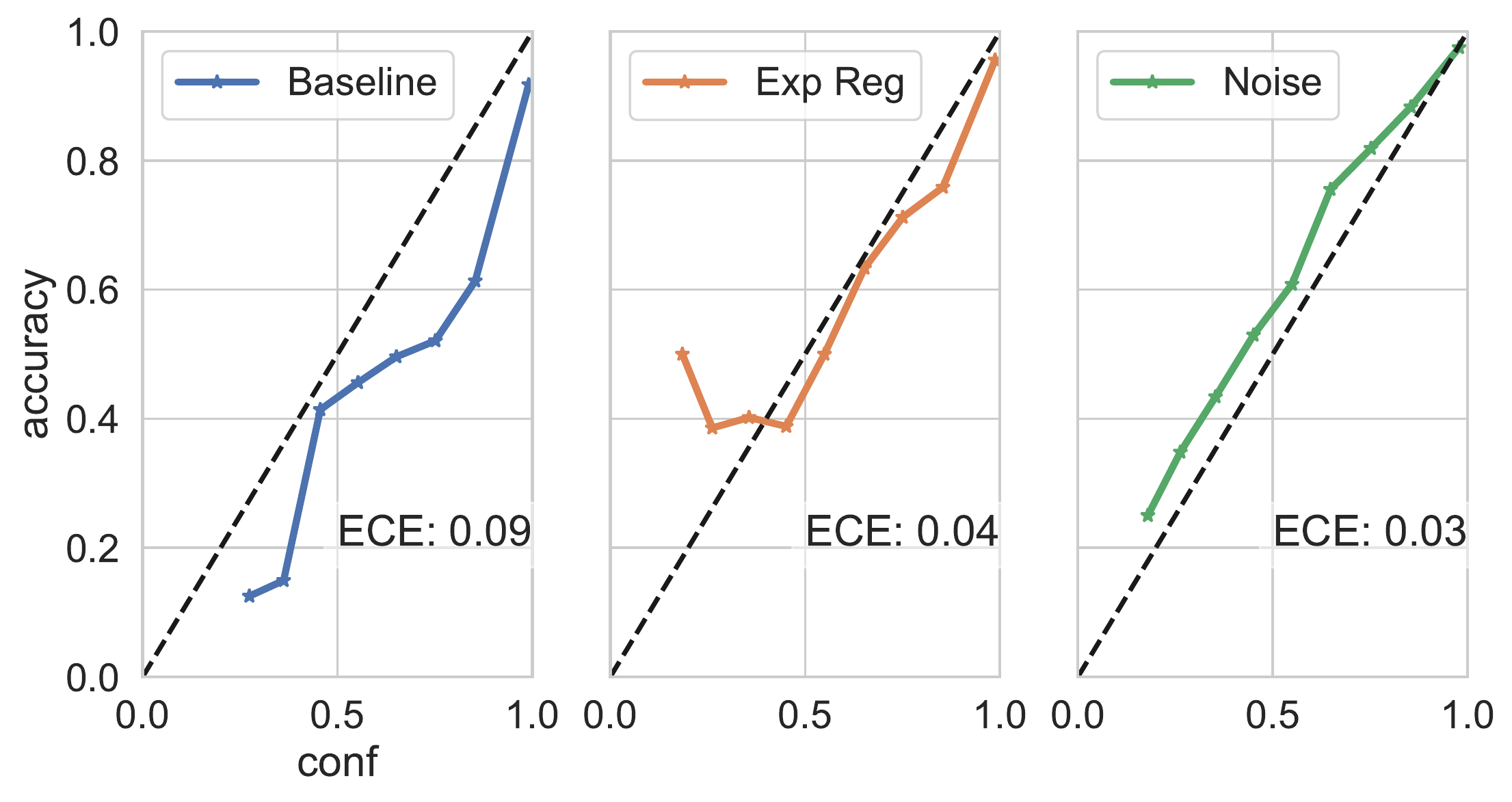}}
    \subfigure[][SVHN CONV, $\sigma=0.1$]{\includegraphics[height=1.2in]{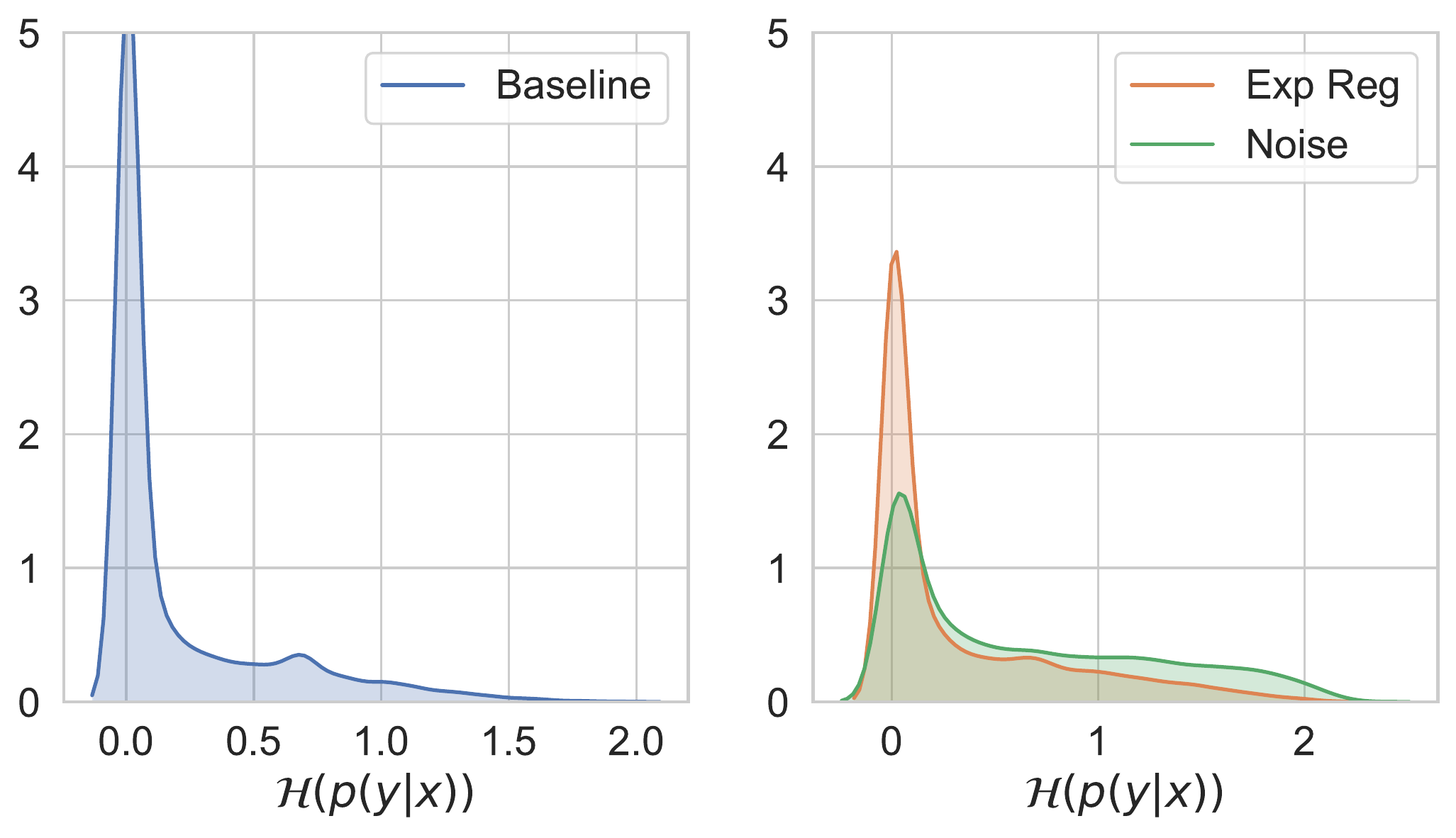}}
    \caption{Illustration of how Gaussian noise (Noise) \textit{additions} improve calibration relative to models trained without noise injections (Baselines) and how $R(\cdot)$ (Exp Reg) also captures some of this improvement in calibration.
    We include results for MLPs and convolutional networks (CONV) with ELU activations on SVHN and CIFAR10 image datasets.
    On the left hand side we plot reliability diagrams \citep{Guo2017, Niculescu-Mizil2005}, which show the accuracy of a model as a function of its confidence over $M$ bins $B_m$. 
    Models that are perfectly calibrated have their accuracy in a bin match their predicted confidence: this is the dotted line appearing in figures. 
    We also calculate the Expected Calibration Error (ECE) which measures a model's distance to this ideal (see Appendix C for a full description of ECE) \citep{Naeini2015}. 
    Clearly, Noise and Exp Reg models are better calibrated with a lower ECE relative to baselines. This can also be appraised visually in the reliability diagram. 
    The right hand side supports these results. We show density plots of the entropy of model predictions. One-hot, highly confident, predictions induce a peak around 0, which is very prominent in baselines. Both Noise and Exp Reg models smear out predictions, as seen by the greater entropy, meaning that they are more likely to output lower-probability predictions.}
    \label{fig:calibration_app}
\end{figure}

\clearpage
\newpage

\section{Network Hyperparameters}
\label{app:hyperparams}

All networks were trained using stochastic gradient descent with a learning rate of 0.001 and a batch size of 512. 

All MLP networks, unless specified otherwise, are 2 hidden layer networks with 512 units per layer. 

All convolutional (CONV) networks are 2 hidden layer networks. 
The first layer has 32 filters, a kernel size of 4, and a stride length of 2. 
The second layer has 128 filters, a kernel size of 4, and a stride length of 2. 
The final output layer is a dense layer.

\end{document}